\documentclass{article}

\usepackage{blindtext}

\usepackage[preprint]{neurips_2026}

\usepackage{subcaption}
\usepackage{amsmath,amssymb,graphicx,mathrsfs,bbm,url}
\usepackage{hyperref}
\hypersetup{colorlinks=true,citecolor=blue,linkcolor=blue,urlcolor=black}

\usepackage[table]{xcolor}

\usepackage{float}
\usepackage{graphicx}
\usepackage{subcaption}
\usepackage{caption}
\usepackage{lscape}
\usepackage{booktabs}
\usepackage{multirow}

\usepackage{xparse}

\usepackage{comment}

\usepackage{amsmath,amssymb}
\usepackage{mathtools}
\usepackage{latexsym}
\usepackage{dsfont}
\usepackage{graphicx}
\usepackage{float}
\usepackage{color}
\usepackage{mathrsfs}

\usepackage{verbatim}
\usepackage{epstopdf}

\usepackage{stmaryrd}

\usepackage{url}

\usepackage[utf8]{inputenc}

\usepackage{hyperref}

\usepackage{minted}

\DeclareMathOperator\diag{diag}

\newtheorem{assumption}{Assumption}
\newtheorem{theorem}{Theorem}
\newtheorem{proposition}{Proposition} 
\newtheorem{lemma}{Lemma} 
\newtheorem{corollary}{Corollary}

\newcommand{\eqdef}{\ensuremath{\stackrel{\mbox{\upshape\tiny def.}}{=}}}
\newenvironment{proof}{\par\noindent\emph{Proof.}\ }{\hfill$\blacksquare$\par}

\usepackage[colorinlistoftodos]{todonotes}
\usepackage{soul}

\def\1{\bm{1}}

\definecolor{darkcyan}{rgb}{0.0, 0.55, 0.55}
\definecolor{MidnightBlue}{RGB}{25,25,112}
\definecolor{MidnightBlueComplementingGreen}{RGB}{25,112,25}
\definecolor{MidnightBlueComplementingPurple}{RGB}{112,25,112}
\definecolor{MidnightBlueComplementingRed}{RGB}{112,25,69}
\definecolor{WowColor}{rgb}{.75,0,.75}
\definecolor{MildlyAlarming}{rgb}{0.85,0.25,0.1}
\definecolor{SubtleColor}{rgb}{0,0,.50}
\definecolor{antiquefuchsia}{rgb}{0.57, 0.36, 0.51}
\definecolor{fashionfuchsia}{rgb}{0.96, 0.0, 0.63}
\definecolor{jade}{rgb}{0.0, 0.66, 0.42}
\definecolor{caribbeangreen}{rgb}{0.0, 0.8, 0.6}
\definecolor{aquamarine}{rgb}{0.5, 0.8, 0.85}
\definecolor{lightseagreen}{rgb}{0.13, 0.7, 0.67}
\definecolor{darkgreen}{rgb}{0.0, 0.2, 0.13}
\definecolor{darkspringgreen}{rgb}{0.09, 0.45, 0.27}
\definecolor{attentioncolor}{RGB}{152,90,81}
\definecolor{burgred}{RGB}{40,3,22}
\definecolor{AnnieGreen}{RGB}{17,123,92}
\definecolor{Turquoise}{RGB}{64,224,208}

\definecolor{darkjade}{RGB}{0,122,84}
\definecolor{Window1}{RGB}{92,150,31}%
    \definecolor{Window1dark}{RGB}{41,67,13}%
\definecolor{Window2}{RGB}{255,168,28}
    \definecolor{Window2dark}{RGB}{114,75,12}
\definecolor{Window3}{RGB}{255,96,33}
    \definecolor{Window3dark}{RGB}{97,36,12}
\definecolor{InputColor}{RGB}{20,255,177}
    \definecolor{InputColorlight}{RGB}{222,237,229}

\definecolor{RedAlizarin}{rgb}{0.82, 0.1, 0.26}

\definecolor{darkcerulean}{rgb}{0.03, 0.27, 0.49}
\definecolor{smokyblack}{rgb}{0.06, 0.05, 0.03}
\definecolor{warmblack}{rgb}{0.0, 0.26, 0.26}
\definecolor{cobalt}{rgb}{0.0, 0.28, 0.67}
\definecolor{darkerred}{RGB}{139, 0, 0}
\definecolor{darkergreen}{RGB}{0, 100, 0}
\definecolor{warmred}{RGB}{205, 92, 92} 
\definecolor{warmgreen}{RGB}{34, 139, 34} 
\definecolor{warmblue}{RGB}{70, 130, 180} 
\definecolor{arrowredfig1}{HTML}{B85450} 
\definecolor{expertfig1}{HTML}{6A9153} % Use boarder colour for legibility
\definecolor{mixturefig1}{HTML}{6C8EBF} % Use boarder colour for legibility

\usepackage{algpseudocode}
\usepackage[linesnumbered,lined,boxed,commentsnumbered,ruled,longend]{algorithm2e}
    \RestyleAlgo{ruled}

\SetCommentSty{mycommfont}

\usepackage{cleveref}
\newcounter{termcounter}
\crefname{term}{term}{terms}
\creflabelformat{term}{#2\textup{(#1)}#3}

\makeatletter
\def\term{\@ifnextchar[\term@optarg\term@noarg}%]
\def\term@optarg[#1]#2{%
  \textup{#1}%
  \def\@currentlabel{#1}%
  \def\cref@currentlabel{[][2147483647][]#1}%
  \cref@label[term]{#2}}
\def\term@noarg#1{%
  \refstepcounter{termcounter}%
  \textup{(\thetermcounter)}%
  \cref@label[term]{#1}}
\makeatother

\usepackage{marginnote}
\definecolor{MidnightBlue}{RGB}{25,25,112}
\definecolor{MidnightBlueComplementingGreen}{RGB}{25,112,25}
\definecolor{MidnightBlueComplementingPurple}{RGB}{112,25,112}
\definecolor{MidnightBlueComplementingRed}{RGB}{112,25,69}
\definecolor{coolblack}{rgb}{0.0, 0.18, 0.39}

\definecolor{deepjunglegreen}{rgb}{0.0, 0.29, 0.29}
\definecolor{applegreen}{rgb}{0.55, 0.71, 0.0}
\definecolor{WowColor}{rgb}{.75,0,.75}
\definecolor{MildlyAlarming}{rgb}{0.85,0.25,0.1}
\definecolor{SubtleColor}{rgb}{0,0,.50}
\definecolor{SubtleColor2}{rgb}{0.6,0.21,.50}

\definecolor{lasallegreen}{rgb}{0.03, 0.47, 0.19}
\newcounter{margincounter}

\definecolor{britishracinggreen}{rgb}{0.0, 0.26, 0.15}

\NewDocumentCommand{\AK}{mo}{
    \IfValueF{#2}{
                        {{\scriptsize
                            \textcolor{deepjunglegreen}{
                            \hfill\\
                                \textbf{A:}
                                \textit{{#1}}
                            \hfill\\
                            }
                        }}
        }
    \IfValueT{#2}{
                        \marginnote{{\scriptsize
                            \textcolor{deepjunglegreen}{ 
                            \textbf{A:}
                            \textit{{#1}}
                            }
                        }}
        }
                    }

\NewDocumentCommand{\Xuwei}{mo}{
    \IfValueF{#2}{
                        {{\scriptsize
                            \textcolor{coolblack}{ 
                            \textbf{X:}
                            \textit{{#1}}
                            }
                        }}
        }
    \IfValueT{#2}{
                        \marginnote{{\scriptsize
                            \textcolor{coolblack}{ 
                            \textbf{X:}
                            \textit{{#1}}
                            }
                        }}
        }
                    }

\NewDocumentCommand{\IE}{mo}{
    \IfValueF{#2}{
                        {{\scriptsize
                            \textcolor{orange}{ 
                            \textbf{IE:}
                            \textit{{#1}}
                            }
                        }}
        }
    \IfValueT{#2}{
                        \marginnote{{\scriptsize
                            \textcolor{orange}{ 
                            \textbf{IE:}
                            \textit{{#1}}
                            }
                        }}
        }
                    }

\NewDocumentCommand{\GA}{mo}{
    \IfValueF{#2}{
                        {{\scriptsize
                            \textcolor{violet}{ 
                            \textbf{GA:}
                            \textit{{#1}}
                            }
                        }}
        }
    \IfValueT{#2}{
                        \marginnote{{\scriptsize
                            \textcolor{violet}{ 
                            \textbf{GA:}
                            \textit{{#1}}
                            }
                        }}
        }
                    }

\usepackage{amssymb}
\usepackage{amsmath,amsfonts} 
\usepackage{booktabs}
\usepackage{longtable}
\usepackage{multirow}
\usepackage{lastpage}

\usepackage{pdfpages}

\usepackage{algpseudocode}

\title{MEMOA: Massive  Mixtures of Online Agents via Mean-Field Decentralized Nash Equilibria}

\author{
Xuwei Yang \\
Department of Mathematics and Statistics \\
McMaster University \\
Ontario, Canada \\
\texttt{henryyangxuwei@gmail.com}
\And
David B. Emerson \\
Vector Institute \\
Ontario, Canada \\
\texttt{david.emerson@vectorinstitute.ai}
\And
Fatemeh Tavakoli \\
Vector Institute \\
Ontario, Canada \\
\texttt{fatemeh.tavakoli@vectorinstitute.ai}
\And
Anastasis Kratsios\thanks{Corresponding Author. Funded by NSERC DG No.\ RGPIN-2023-04482.} \\
Department of Mathematics and Statistics \\
McMaster University and the Vector Institute \\
Ontario, Canada \\
\texttt{kratsioa@mcmaster.ca}
}

\usepackage[most]{tcolorbox}

\newtcolorbox{questionbox}{
  enhanced,
  hbox,
  colback=gray!6,
  colframe=gray!50,
  boxrule=0.4pt,
  arc=1.5mm,
  left=3mm,
  right=3mm,
  top=1.2mm,
  bottom=1.2mm,
  before=\begin{center},
  after=\end{center}
}

\begin{document}

%\editor{}

\maketitle

%%%%%%%%%%%%%%%%%%%%%%%%%%%
\begin{abstract}
In the modern age of large-scale AI, federated learning has become an increasingly important tool for training large populations of AI agents; however, its computational and communication costs can rapidly fail to scale with the number of agents.  
This is precisely where decentralized agentic strategies shine: each agent acts autonomously, using only its own state together with a minimal summary of the ensemble, namely the mean-field.  
We derive the unique optimal decentralized policy in closed form.  Optimality is characterized through a worst-client/minimax criterion: minimizing the \textit{under-performer regret}, namely the maximal online cost incurred by the weakest agent in the ensemble.
We further prove that the resulting decentralized policy asymptotically converges, in the large-population limit, to the Nash-optimal centralized policy, whose direct computation is not scalable.  
We use an online weighting mechanism to optimize the server-computed mixture of client predictions, thereby improving the mean prediction in addition to the previously optimized weakest-client prediction.
Numerical experiments verify our theoretical guarantees and demonstrate that our decentralized policy typically outperforms natural greedy decentralized baselines.
\end{abstract}
%%%%%%%%%%%%%%%%%%%%%%%%%%%

\section{Introduction}
\label{sec:intro}

Federated learning (FL) has become a central strategy in modern machine learning, where users (the server) have access to large ensembles of models (clients) whose joint predictions are used for prediction.  
As modern AI clients continue to scale up~\cite{hoffmann2022training,kaplan2020scaling,liu2025budget}, the cost of prediction---computational, energy~\cite{faiz2024llmcarbon,luccioni2024power,luccioni2023estimating,samsi2023words}, and ultimately monetary~\cite{bai2026how,xiao2025reducing}---continues to increase.  Moreover, as the number of available cutting-edge clients grows, so does the size of ensembles, leading to a very high cost of having any, or several, under-performing AI clients.  
This then drives one of the main looming questions in cost-efficient FL:
\begin{questionbox}
How should federated agents behave to minimize worst-agent system cost?
\end{questionbox}
We study the online version of this question under the \textit{proprietary} constraint: clients are treated as AI agents which actively make decisions \textit{online}, while the central server has no access to, and cannot influence, the hidden features of any model.  Thus, only the final linear output layer can be controlled; building on the ideas of~\cite{YTEK25}.

To minimize communication cost, we search for the optimal autonomous policy of the clients/agents, whereby each agent tracks only its own state and the mean-field of the entire massive ensemble of agents~\cite{Liao23GlobalNashFL,Murhekar23IncentivesFL,YTEK25}.  Incorporating the mean field allows AI clients to implicitly coordinate with one another at minimal computational cost.
We show that minimizing the \textit{cost to the system} of the worst agent is equivalent to identifying a policy which is a \textbf{Nash equilibrium} amongst the massive ensemble of AI agents (Proposition~\ref{prop:whynash}).  Critically, we consider a \textbf{decentralized} Nash asymptotic equilibrium: an equilibrium policy in the \textbf{mean-field limit}, as ($N\uparrow \infty$), with the key benefit that no agent needs to observe the state of any other agent to behave optimally.  This is an asymptotic simplification widely used in stochastic game theory~\cite{HY21mfg,Liang2023DecentralizedOpenLoop,Xu2024DecentralizedEpsilonNash}.

\paragraph{Contributions}
We identify an asymptotically Nash-optimal decentralized policy (Algorithm~\ref{alg:sync_subroutine}) which scales linearly in the number of AI agents $(N)$; unlike centralized exact Nash equilibria such as~\cite{yang2023regret,YTEK25} whose computational cost scales super-quadratically in $N$.  We prove the asymptotic Nash optimality of our decentralized policy (Theorem~\ref{thm:Y(N)->barY}) in the mean-field limit by exploiting simple structures in the optimal Nash policy for the finite client $(N)$ paradigm (Theorem~\ref{thm:betan_centralized}) which holds under a mild agentic homogeneity assumption (Assumption~\ref{assm:Zt}) which is much weaker than the i.i.d.\ assumption standard in mean-field game theory.

\paragraph{Organization of the Paper.}
Section~\ref{s:Prelim} sets up the dynamic mean-field game formulation.  
%Section~\ref{sec:MainResults} 
Section~\ref{sec:MainResults} derives the decentralized asymptotically Nash-optimal policy, proves its large-population guarantees and outlines the full algorithm. Section~\ref{sec:numerics} presents numerical validation. Finally, Section \ref{conclusions} discussions conclusions, limitations, and future work. All proofs are relegated to the appendices.

\subsection{Related Work to Decentralized Policy}

\paragraph{Game-Theoretic Equilibria for Optimal Federated Learning.}
Game-theoretic formulations of federated learning have recently been used to model strategic participation, incentive alignment, robustness, and equilibrium behavior among clients.  Model-sharing games study voluntary participation and coalition stability~\cite{DonahueKleinberg2021ModelSharing,DonahueKleinberg2021OptimalityStability}, while incentive-based formulations analyze Nash equilibria, welfare, and best-response dynamics~\cite{MurhekarEtAl2023IncentivesFL,YoonChoudhuryLoizou2025MultiplayerFL}; robust federated learning has also been formulated through mixed Nash equilibria~\cite{Xie24mixed}.  Closest to our setting is~\cite{YTEK25}, which gives an exactly optimal online Nash policy for mixtures of proprietary agents with black-box encoders, but requires centralized server-side updates using all agents' states and incurs \textit{super-quadratic} cost in $N$.  Our $\mathcal{O}(1/N)$-optimal decentralized policy (Algorithm~\ref{alg:sync_subroutine}) instead trades a linearly vanishing sub-optimality gap for local agent-level updates.

\paragraph{Federated Learning and Mean-Field Games.}
Mean-field methods make large-client federated systems scalable by replacing full population interactions with low-dimensional population summaries.  This perspective has been used to interpret federated optimization as a mean-field game~\cite{Mehrjou2021FederatedMeanFieldGame}, and to design communication-efficient online control, incentive, privacy-aware, and market-based mechanisms~\cite{ShiriParkBennis2020UAVOnlineMFGFL,SunEtAl2026PrivacyCommodityMFGRegretNet,SunWuLi2024ReputationMFGFL,YuanWang2024PrivacyAwareSampling}.  Our approach follows this line, but focuses on an explicit dynamic linear-quadratic Nash system and proves that its decentralized mean-field policy asymptotically matches the centralized Nash policy.

\paragraph{Decentralized Strategies in Linear-Quadratic Mean-Field Games.}
Our technical perspective is connected to the classical theory of linear-quadratic mean-field games, including Nash certainty equivalence, decentralized $\varepsilon$- and the asymptotically (in $N$) Nash-optimal constructions, asymptotic solvability, Riccati characterizations, and finite-population-to-mean-field limits~\cite{HuangCainesMalhame2007LQG,HuangMalhameCaines2006NCE,HY21mfg,HZ20,Liang2023DecentralizedOpenLoop,Xu2024DecentralizedEpsilonNash}.  
We adapt this philosophy to federated online prediction, where the mean-field summarizes an minimal ensemble information shared amongst the agents/clients.

\subsection{Notation and Conventions} 

Let $I_m$ denote the $m\times m$ identity matrix; when dimensions are clear, we simply write $I$.  We use $0$ for either the scalar zero or a zero matrix of appropriate dimensions, and let $\mathbf{1}_{m\times n}$ denote the $m\times n$ matrix with all entries equal to $1$. We denote the set of positive integers by $\mathbb{N}_+$. For every $K\in \mathbb{N}_+$, write $[K]\eqdef \{1,\dots,K\}$.  For matrices $A$ and $B$, we write $A\otimes B$ for their Kronecker product. 
%$\operatorname{vec}(A)$ for the vectorization/``flattening'' of $A$, and $A^\dagger$ for the Moore--Penrose pseudo-inverse of $A$. 
We denote the Frobenius norm and inner product by $\|\cdot\|$ and $\langle\cdot,\cdot\rangle$, respectively. 
%If $A$ is positive semi-definite, then $\lambda_{\min}(A)$ denotes its smallest eigenvalue.  Finally, we set $\xi_I\eqdef \operatorname{vec}(I_{d_z})\in \mathbb{R}^{d_z^2}$.

\section{Preliminaries on Dynamic Mean-Field Gates}
\label{s:Prelim}
We henceforth work on a filtered probability space $(\Omega,\mathcal{F},\mathbb{F}\eqdef (\mathcal{F}_t)_{t=0}^T, \mathbb{P})$ for some $T\in \mathbb{N}_+$, and we let $\mathcal{A}$ denote the set of square-integrable predictable processes thereon; i.e.\ belonging to  the set of admissible strategies
\[
    \mathcal{A}
\eqdef
    \biggl\{
        \beta:\Omega\times [0,T]\to \mathbb{R}^{d_z\times 1}
        :
        \beta \text{ is } \mathbb{F}\text{-predictable and }
        \mathbb{E}\Biggl[
            % \int_0^T
            \sum_{t=0}^T
            \,
            \|\beta_t\|_2^2
            \,
            dt
        \Biggr]
        <\infty
    \biggr\}
.
\]
Let $J_1,\dots,J_N:\mathcal{A}^N\to [0,\infty)$ be non-linear functionals.
Given any $(\beta^{1},\dots,\beta^{N})\in \mathcal{A}^N$, and any $n\in [N]$, we write
$
(\tilde{\beta},\beta^{-n})
$
for the element of $\mathcal{A}^N$ whose $n^{\mathrm{th}}$ component is $\tilde{\beta}\in \mathcal{A}$ and whose $m^{\mathrm{th}}$ component is $\beta^{m}$ for every $m\neq n$. 
Then an $N$-tuple $(\beta^{n})_{n=1}^N\in \mathcal{A}^N$ is a \emph{Nash equilibrium} if for each $n\in[N]$ there \textit{does not} exist some $\tilde{\beta}^{n} \in \mathcal{A}$ satisfying
$J_n(\tilde{\beta}^{n},\beta^{-n}) < J_n(\beta^{n},\beta^{-n})$.

Under the usual interchangeable agent hypothesis in mean-field game-theory, a Nash equilibrium simply means that the \textit{weakest} agent/client is as \textit{\textbf{strong} as possible}.
\begin{proposition}[Nash Equilibria Characterized Minimal Under-Performer Regret]
\label{prop:whynash}
If $J_1,\dots,J_N:\mathcal{A}^N\to [0,\infty)$ are continuous, and interchangeable; i.e.\ $J_n=J_1$ for each $n\in[N]$, and there exists a unique Nash equilibrium $(\beta^{\star:n})_{n=1}^N$ for $J_1,\dots,J_N$ then the following are equivalent for any policy $(\beta^{n})_{n=1}^N\in\mathcal{A}^N$:
\begin{enumerate}
    \item[(i)] \textit{Nash:} $(\beta^{n})_{n=1}^N\in\mathcal{A}^N$ is a Nash equilibrium for $J_1,\dots,J_N$,
    \item[(ii)] \textit{Minimal Under-performer Regret:} $
    \mathcal{R}_{\downarrow}\big((\beta^{n})_{n=1}^N\big)
    =
    \inf_{
(\tilde{\beta}^{n})_{n=1}^N\in \mathcal{A}^N
}
\,\mathcal{R}\big((\tilde{\beta}^{n})_{n=1}^N\big)$
\end{enumerate}
where \textbf{under-performer regret} functional $\mathcal{R}_{\downarrow}$ maps any $(\beta^{n})_{n=1}^N \in \mathcal{A}^N$ to 
\[
\mathcal{R}_{\downarrow}\big(\beta^{n})_{n=1}^N\big)
\eqdef
\max_{n\in [N]}\, J_n(\beta^{n},\beta^{-n})
.
\]
\end{proposition}
\begin{proof}
    See Appendix~\ref{appednix:Secondary}.
\end{proof}
We now specialize our fully-general formulation of dynamic Nash gates to an appropriate level of generality compatible with our problem of optimally dynamically updating the linear layer of $N$ centralized agents/clients for recurrent deep learning models.

\subsection{The Finite Population Model: Implementable on a Finite-Capacity Computer}

We consider a federated learning model with a large but finite population of proprietary agents, i.e., the population size $N$ is large. 
Each agent aims to make sequential predictions of the target data $\{y_t\}_{t=0,1,2,\dots}$. 

For each agent $n \in [N]$, the deep learning model takes as input the shared historical data $x_{[0:t]}$, an idiosyncratic noise source $\epsilon^n_t$, and an individual hidden state $Z^n_{t-1}$ to produce a latent representation
\[
Z^n_t = \phi^n(x_{[0:t]}, Z^n_{t-1}, \epsilon^n_t),
\]
where $\phi^n(\cdot)$ is a (fixed) encoder; see~\cite[Sec. 2]{YTEK25}. 
The latent state is then used to update the prediction via a dynamically adjusted decoder parameterized by $\beta^n$. 

Specifically, the prediction of agent $n$ evolves according to
\begin{align} 
Y^n_{t+1} = \theta Y^n_t + \overline{\theta} Y^{(N)}_t + Z^n_t \beta^n_t, 
\quad 
t = 0,1,\dots,T-1, 
\quad 
Y^n_0 = y_0, 
\quad 
n \in [N],
\label{hatYit+1finite}
\end{align}
where $Y^n_t \in \mathbb{R}^{d_y}$, $\theta, \overline{\theta} \in \mathbb{R}^{d_y \times d_y}$, 
$Z^n_t \in \mathbb{R}^{d_y \times d_z}$, $\beta^n_t \in \mathbb{R}^{d_z}$, and
$Y^{(N)}_t \eqdef \frac{1}{N} \sum_{n=1}^N Y^n_t$
denotes the population average prediction.

Agents operate in periodic rounds of length $T$. At the beginning of each round, they initialize from $Y^n_0 = y_0$, and over the interval $\{0,1,\dots,T-1\}$ they may exchange information to improve predictive performance. The objective of each agent $n \in [N]$ over a typical period is given by
\begin{align} 
\label{eq:Objective_nth_expert}
J_n(\beta^n, \beta^{-n}) 
= 
\mathbb{E} \Bigg[ \sum_{t=0}^{T-1} 
e^{-\alpha (T-1-t)} 
\Big(
\kappa \| y_{t+1} - Y^n_{t+1} \|^2
+ \overline{\kappa} \| Y^n_{t+1} - Y^{(N)}_{t+1} \|^2 
+ \gamma \| \beta^n_t \|^2
\Big)
\Bigg],
\end{align}
where $\alpha, \kappa, \overline{\kappa}, \gamma > 0$. 

The first term in~\eqref{eq:Objective_nth_expert} captures the prediction error with respect to the target data, the second term enforces consensus by penalizing deviations from the population mean~\cite{Li21Ditto,Li20FedOptim,Liu22PrivacyFL}, and the third term regularizes the control effort.

Each agent $n \in [N]$ selects a policy $\{\beta^n_t\}_{t=0}^{T-1}$ to minimize~\eqref{eq:Objective_nth_expert} subject to the dynamics~\eqref{hatYit+1finite}, and subsequently updates its prediction at the end of the period, $Y^n_T$. Since the proprietary agents operate in a noncooperative manner~\cite{YTEK25}, we formulate the federated learning model~\eqref{hatYit+1finite}--\eqref{eq:Objective_nth_expert} as a Nash game among the agents.

Direct computation of the Nash equilibrium for~\eqref{hatYit+1finite}--\eqref{eq:Objective_nth_expert} requires inverting high-dimensional matrices whose size scales with the population size $N$, rendering the approach computationally prohibitive in large systems. In the next section, we introduce mild homogeneity assumptions across agents that enable a decomposition of these matrices into repeating submatrices, thereby significantly reducing computational complexity.

\section{Main Results} 
\label{sec:MainResults}

\subsection{Phase 1: The Decentralized Asymptotically Nash-Optimal Policy} 
\label{sec:decentralizedStrat} 

%In this section, we introduce a decentralized framework in which each agent’s policy depends only on its own prediction together with a low-dimensional auxiliary term that can be computed offline. This significantly reduces communication costs while preserving the tractability of the policy computation.

In this section, we present our main result, Algorithm~\ref{alg:sync_subroutine}, which implements the decentralized learning policy~\eqref{betan_decentralized}. The proposed policy avoids the curse of dimensionality: its computation involves only fixed-dimensional quantities and relies solely on local information together with a low-dimensional mean-field state. As a result, it significantly reduces both computational and communication costs in large-population federated learning models. 

The \textbf{decentralized policy} (asymptotically Nash-Optimal; cf.\ Theorem~\ref{thm:Y(N)->barY}) for each agent is: 
\begin{align} 
 \check\beta_t^n = G_1 (t, \Lambda_1 ) \check{Y}^n_t + G_2(t, \Lambda_1, \Lambda_2 ) \overline{Y}_t 
  + H(t, \Lambda_1, \chi_1 ) , \quad  n\in [N] ,  
 \label{betan_decentralized}
\end{align} 
where $\check{Y}^n_\cdot$ is the prediction update~\eqref{hatYit+1finite} under~\eqref{betan_decentralized} and the auxiliary mean field process $\overline{Y}_\cdot$ is given by 
\begin{align} 
\begin{aligned}
 & \overline{Y}_{t+1} =  \big[ \theta + \overline\theta 
  + M^{(1)}_Z(t) \big( G_1(t, \Lambda_1 ) + G_2(t,\Lambda_1, \Lambda_2 ) \big) \big] \overline{Y}_t 
 + M^{(1)}_Z(t) H(t, \Lambda_1, \chi_1 ) 
 , 
  \\ 
 % \mbox{ and }\,
& \overline{Y}_0 =  y_0  . 
\end{aligned} 
 \label{barY}
\end{align}

\begin{algorithm}[t]
\caption{Synchronization Subroutine (Decentralized Policy)}
\label{alg:sync_subroutine}
\SetAlgoLined
\DontPrintSemicolon

\SetKwInOut{Require}{Input}
\SetKwInOut{Ensure}{Output}

\Require{Observed targets $\{y_t\}_{t=0}^T$, 
hidden states $\{Z^1_t,\dots,Z^N_t\}_{t=0}^T$, 
parameters $\theta$, $\overline{\theta}$, $\kappa$, $\overline{\kappa}$, $\gamma$, $\alpha$}

\Ensure{Policies $\{\check\beta^n_{T-1}\}_{n=1}^N$ and terminal prediction $\hat{Y}_T$}

\tcp{Initialization}
$\big[ \Lambda_1, \Lambda_2, \Lambda_3, \Lambda_4, \chi_1, \chi_2 \big](0) \leftarrow 
 \big[ 0, 0, 0, 0, 0, 0 \big]$ \;

\tcp{Backward recursion}
\For{$t = T-1,\dots,0$}{
    $\big[ M^{(1)}_Z(t), \, M^{(2)}_Z(t), \, M^{(2)}_{Z, \Lambda_1}(t) \big] \leftarrow \big[ \mathbb{E}(Z^1_t), \, \mathbb{E}(Z^{1\top}_t Z^1_t), \, \mathbb{E}(Z^{1\top}_t \Lambda_1(t+1) Z^1_t) \big]$\;

    Update $F(t,\Lambda_1)$, $K(t,\Lambda_2)$, $M(t,\Lambda_1)$, $E(t,\Lambda_1,\Lambda_2)$ using~\eqref{FKME}\;
    
    Update $G_1(t,\Lambda_1)$, $G_2(t,\Lambda_1,\Lambda_2)$, $H(t,\Lambda_1, \chi_1 )$ using~\eqref{G12H}\;
    
    Update $Q_1(t,\Lambda_1)$, $Q_2(t,\Lambda_2)$, $Q_3(t,\Lambda_3)$, $Q_4(t,\Lambda_4)$ using~\eqref{Q1234}\;

    $\big[ \Lambda_1,\Lambda_2,\Lambda_3,\Lambda_4,\chi_1,\chi_2 \big](t)\leftarrow \big[ \varphi_1,\varphi_2,\varphi_3,\varphi_4,\psi_1,\psi_2\big] (t;\Lambda_1,\Lambda_2,\Lambda_3,\Lambda_4,\chi_1,\chi_2)$\; 
    
    \hspace{6cm}  \mbox{ \tcp{$\varphi_1, \varphi_2, \varphi_3, \varphi_4, \psi_1, \psi_2$ given by~\eqref{varphi1}-\eqref{psi2}} }
}

\tcp{Forward simulation}
$\check{Y}_0 \leftarrow (y^{\top}_0,\dots,y^{\top}_0)^\top$\;
$\overline{Y}_0 \leftarrow  y_0$\;

\For{$t=0,\dots,T-1$}{
    \For{$n=1,\dots,N$}{
        $\check\beta_t^n \leftarrow G_1(t,\Lambda_1)\check{Y}^n_t 
        + G_2(t,\Lambda_1,\Lambda_2)\overline{Y}_t 
        + H(t,\Lambda_1,\chi_1)$\;
    }

    $\check{Y}_{t+1} \leftarrow \big[\mathbf{\Theta} + \tfrac{1}{N}\overline{\mathbf{\Theta}}\big]\hat{Y}_t 
    + \sum_{n=1}^N \mathbf{e}^y_n Z^n_t \check\beta_t^n$\;

    $\overline{Y}_{t+1} \leftarrow \big[\theta + \overline{\theta} 
    + M^{(1)}_Z(t)\big(G_1(t,\Lambda_1)+G_2(t,\Lambda_1,\Lambda_2)\big)\big]\overline{Y}_t 
    + M^{(1)}_Z(t)H(t,\Lambda_1,\chi_1)$\;
}

\end{algorithm}

The matrix-valued coefficients $G_1$, $G_2$ and $H$ are functions of $\Lambda_1$, $\Lambda_2$, and $\chi_1$, which are determined by a low-dimensional system of ordinary differential equations (ODEs); cf.~\eqref{ODEs:Lambda1234Chi12} whose solution is encoded directly into the $\Lambda_1,\Lambda_2,, \Lambda_3, \Lambda_4,\chi_1,\chi_2$ updates in Algorithm~\ref{alg:sync_subroutine}. The decentralized policy~\eqref{betan_decentralized} is implemented via Algorithm~\ref{alg:sync_subroutine}, one of the main results of this paper.

\subsubsection{{Main Guarantees: Asymptotic Nash-Optimality of Decentralized Policy (Algorithm~\ref{alg:sync_subroutine})}}

The decentralized learning policy~\eqref{betan_decentralized} is best understood through its derivation.  We first solve the $N$-player game completely via the dynamic programming principle~\cite{bellman1954theory}, more specifically the dynamic programming approach of~\cite[Corollary 6.1]{Basar1999GameTheory}.  We then pass to the $N$-agent limit, $N\uparrow\infty$, where the strategy simplifies: under the Nash-optimal homogeneity assumption on the agents (Assumption~\ref{assm:Zt}), each agent can asymptotically ignore the individual behaviour of the other agents, modulo the mean-field information encoded by the ensemble mean.  Thus, the limiting $N\uparrow\infty$ strategy becomes decentralized.  Finally, convergence of the finite $N$-agent game to the limiting $\infty$-agent game implies the asymptotic optimality of the limiting \textit{decentralized policy}/strategy for the finite-population $N$-agent game, whose Nash-optimal strategy is centralized and therefore computationally costly; cf.~\cite{YTEK25}.
We begin by exhibiting the Nash equilibrium policy in Theorem~\ref{thm:NagentNash}.  We then present, in Theorem~\ref{thm:betan_centralized}, the simplified centralized policy under Assumption~\ref{assm:Zt}; its inter-agent terms \textit{asymptotically vanish} as $N\uparrow\infty$.

We now present the Nash-optimal \textit{centralized} strategy; which \textit{does not scale well} in $N$.
\begin{theorem}[Nash-Optimal Policy for $N$-Agent Mean-Field Game]
\label{thm:NagentNash}
Consider the $N$-agent model~\eqref{hatYit+1finite} with objective functions~\eqref{eq:Objective_nth_expert}. For each $t=0,1,\dots,T-1$, the game admits a Nash equilibrium of the form
\begin{align}
\big[ \widehat\beta_t^{1\top},\dots,\widehat\beta_t^{N\top}\big]^\top
& =
\mathbf{G}(t)
\big[\widehat Y_t^{1\top},\dots,\widehat Y_t^{N\top}\big]^\top
+
\mathbf{H}(t),
\label{NashEqm:bfbeta}
\end{align}
where the matrix-valued coefficients $\mathbf{G}(\cdot)\in \mathbb{R}^{N d_z \times N d_y}$ and $\mathbf{H}(\cdot)\in \mathbb{R}^{N d_z \times 1}$, defined by~\eqref{bfG}–\eqref{bfH}, depend on the solutions to the following system of backward difference equations: for each $n\in[N]$ and $t=0,1,\dots,T-1$,
\begin{align}
P_n(t)
&= \Phi_n\big(t; P_1,\dots,P_N\big),
\qquad
P_n(T)=0,
\label{Pn=Phin(t;P1:N)}
\\
S_n(t)
&= \Psi_n\big(t; S_1,\dots,S_N; P_1,\dots,P_N\big),
\qquad
S_n(T)=0.
\label{Sn=Psi(t;S1:N;P1:N)}
\end{align}
Here, $P_n(\cdot)\in\mathbb{R}^{N d_y \times N d_y}$ and $S_n(\cdot)\in\mathbb{R}^{N d_y}$. The explicit forms of $\{\Phi_n(\cdot)\}_{n=1}^N$ and $\{\Psi_n(\cdot)\}_{n=1}^N$ are defined by~\eqref{ODE:Pn} and~\eqref{ODE:Sn} in Appendix~\ref{appendix:proofNagentNash}.
\end{theorem} 
\begin{proof}
    See Appendix~\ref{appendix:proofNagentNash}.
\end{proof}

The system of autonomous system of nonlinear ODEs~\eqref{Pn=Phin(t;P1:N)}, which is independent of~\eqref{Sn=Psi(t;S1:N;P1:N)}, whose solutions are given by $\{P_n\}_{n=1}^N$ from~\eqref{Pn=Phin(t;P1:N)} given, \eqref{Sn=Psi(t;S1:N;P1:N)} forms a system of first-order linear equations, according to~\eqref{ODE:Pn} and~\eqref{ODE:Sn}. The \textit{scalability problem} is in computing the Nash equilibrium~\eqref{NashEqm:bfbeta} requires solving the nonlinear equation~\eqref{Pn=Phin(t;P1:N)} whose \textbf{scalability bottleneck} arises from the inversion of an $d_z N \times d_z N$ matrix in~\eqref{bfG}; whose computational complexity \textit{super-quadratic}, at a best-known cost of $\mathcal{O}((N d_z)^{2.373})$ according to~\cite{LeGall2014Matrix}. 
This is feasible for small systems, cf.\ the centralized Nash-strategy problem for federated learning without a mean-field term was solved in~\cite{YTEK25}, which requires one hour on A4$k$ GPUs, cf.~\cite{NVIDIA_RTX_A4000}, for system of $N=5$ agents. 
Unfortunately, its computational cost rapidly becomes infeasible as $N$ grows making it infeasible for large agent populations.

The key to \textbf{breaking the scalability bottleneck} is to solve the system in the mean-field limit, as $N\uparrow \infty$, which is made possible by homogeneity requirements among our agents/clients.  In mean-field game theory, this typically means imposing a stringent i.i.d.\ assumption on all agents; cf.~\cite{HuangCainesMalhame2007LQG,LasryLions2006MFGStationary,LasryLions2006MFGFiniteHorizon,LasryLions2007MeanFieldGames} and applications thereof, e.g.~\cite{HuangJaimungalNourian2019MeanField}.  By contrast, our theory goes through under the much weaker assumption that \textit{only the first two moments} of the agents must coincide; their distributions may otherwise be entirely different.  This is critical in practice, as it allows us to use different AI agent systems, rather than $N$ copies of the same system with randomized parameters, ensuring that the theory applies to a range of non-stylized practical use-cases.

\begin{assumption}[Agentic Homogeneity (\textit{not Interchangeability}): First Two Moments Only]
\label{assm:Zt} 
\hfill\\
For each $t = 0, 1, \dots, T$, 
$Z^1_t, \dots, Z^N_t$ satisfy   
\begin{align} 
 & \mathbb{E} Z^1_t = \mathbb{E} Z^2_t = \dots = \mathbb{E} Z^N_t \eqdef M_Z^{(1)}(t) ;  
  \label{EZit=MZ1t}
  \\
 & \mathbb{E} ( Z^{n\top}_t W Z^m_t ) = \mathbb{E} ( Z^{n\top}_t ) W  \mathbb{E} ( Z^m_t ) , \ 
 \forall n \neq m ; 
 \notag \\ 
 & \mathbb{E} ( Z^{1\top}_t W Z^1_t ) = \mathbb{E} ( Z^{2\top}_t W Z^2_t ) 
  = \dots = \mathbb{E} ( Z^{N\top}_t W Z^N_t ) 
  \eqdef M_{Z,\, W}^{(2)}(t), 
  \quad \forall \ W \in \mathbb{R}^{d_y \times d_y} . 
  \notag 
\end{align} 
In particular, $ 
 \mathbb{E} ( Z^{1\top}_t Z^1_t ) = \mathbb{E} ( Z^{2\top}_t Z^2_t ) 
  = \dots = \mathbb{E} ( Z^{N\top}_t Z^N_t ) \eqdef M_{Z}^{(2)}(t)$.
\end{assumption}  

Under Assumption~\ref{assm:Zt}, the agents exhibit a form of homogeneity: the systems~\eqref{Pn=Phin(t;P1:N)} and~\eqref{Sn=Psi(t;S1:N;P1:N)} are invariant under certain permutations of their sub-matrices. 
Following~\cite{HY21mfg,HZ20}, we systematically permute the sub-matrices of $P_n$ to show that these high-dimensional matrices decompose into repeating sub-matrices of fixed dimension. 
Consequently, computing the full matrices $P_n$ reduces to computing only a significantly smaller collection of such sub-matrices; the $N\to\infty $ limit of the Nash-optimal policy for homogeneous agents \textbf{is} precisely \textbf{Algorithm}~\ref{alg:sync_subroutine}.
\begin{theorem}[Unlocking Asymptotically-Tractable Nash-Optimal Policies via Agentic Homogeneity]
\label{thm:betan_centralized} 
Under Assumption~\ref{assm:Zt}, the Nash equilibrium policy~\eqref{NashEqm:bfbeta} of each individual agent takes the form 
\vspace{-0.2cm}
\begin{align} 
 \widehat\beta^n_t 
 = & 
 G_1^N(t) \widehat{Y}^n_t 
 + G_2^N(t) \sum_{j\neq n = 1}^N \widehat{Y}^j_t + H^N(t) , 
 \quad n \in [N], 
 \label{betan_centralized}
\end{align} 
where $G^N_1(\cdot)\in \mathbb{R}^{d_z \times d_y}$, $G^N_2(\cdot)\in \mathbb{R}^{d_z \times d_y}$, and $H^N(\cdot)\in \mathbb{R}^{d_y\times 1}$ are given by~\eqref{G1N}, \eqref{G2N}, and \eqref{HN}. 
\begin{align} 
& \Pi_i^N(t) = \varphi_i^N( t; \Pi^N_1, \Pi^N_2, \Pi^N_3 , \Pi^N_4  )  , 
\quad \Pi^N_i (T) = 0, 
\quad i = 1, 2, 3, 4; 
\label{ODE:PiNi} \\ 
%\end{align} 
%\begin{align} 
& \Xi_i^N(t) =  
\psi^N_i\big( t; \Pi^N_1, \Pi^N_2, \Pi^N_3, \Pi^N_4, \Xi_1^N, \Xi_2^N \big) , 
\quad 
\Xi_i^N(T) = 0 , 
\quad 
i = 1, 2 ; 
\label{ODE:XiNi}
\end{align} 
where the matrix-valued functions $\varphi_i^N$, $i=1,2,3, 4$ are defined by~\eqref{varphiN1/4}, \eqref{varphiN2/4}, \eqref{varphiN3/4}, and~\eqref{varphiN4/4}, 
and the matrix-valued functions $\psi_i^N$, $i=1,2$ are defined by~\eqref{XiN1} and~\eqref{XiN2}. 
\end{theorem}   
\begin{proof}
The proof of Theorem~\ref{thm:betan_centralized} is given in Appendix~\ref{appendix:centralizedStrat_homogeneity}. 
\end{proof}

Under Assumption~\ref{assm:Zt}, the formula~\eqref{betan_centralized} simplifies~\eqref{NashEqm:bfbeta} substantially: it suffices to compute the three low-dimensional matrix-valued functions $G_1^N(\cdot)$, $G_2^N(\cdot)$, and $H^N(\cdot)$, rather than the full high-dimensional matrices $\mathbf{G}(\cdot)$ and $\mathbf{H}(\cdot)$; cf.\ Algorithm~\ref{alg:sync_subroutine} for \textit{closed-form} updates of these three low-dimensional matrices whose dimension is \textbf{independent of $N$}.      
However, implementing~\eqref{betan_centralized} still requires each agent to collect all predictions $\{\widehat{Y}_t^n\}_{n=1}^N$; this vanished in the $N\uparrow \infty$ limit, as is reflected by these terms not appearing in the decentralized asymptotically Nash-optimal policy in Algorithm~\ref{alg:sync_subroutine}.  
The omittance of these terms is justified by the our final main theorem confirming that $\overline{Y}_\cdot$ is the limiting counterpart of the corresponding mean prediction process as $N\uparrow\infty$.
\begin{theorem}[Asymptotic Nash Equilibrium of the Decentralized Policy in Algorithm~\ref{alg:sync_subroutine}]
\label{thm:Y(N)->barY} 

For all $N\in \mathbb{N}_+$, suppose that$:$ 

(i). For each $t=0,1,\dots,T-1$, the states $\{Z^n_t\}_{n\in[N]}$ are independent and identically distributed, and satsify Assumption~\eqref{assm:Zt}. 

(ii). $\{Z^n_t\}_{n\in[N],\, t=0,\dots,T-1}$ are uniformly bounded, i.e., there exists a constant $K_Z<\infty$ such that $\|Z^n_t\| \le K_Z$ almost surely for all $n$ and $t$. 

%(iii). The dataset $\{y_t\}_{t=0,\dots,T}$ is uniformly bounded.

Let $\overline{Y}_\cdot$ be defined by~\eqref{barY}. Let $\widehat{Y}^{(N)}_\cdot$ and $\check{Y}^{(N)}_\cdot$ denote the mean predictions under the centralized policy~\eqref{betan_centralized} and the decentralized policy~\eqref{betan_decentralized}, respectively. Then, for each fixed $t=0,1,\dots,T$,
\[
\lim_{N\to\infty} 
\mathbb{E}\big\| \widehat{Y}^{(N)}_t - \overline{Y}_t \big\| = 0,
\qquad
\lim_{N\to\infty} 
\mathbb{E}\big\| \check{Y}^{(N)}_t - \overline{Y}_t \big\| = 0.
\]
Consequently,
\[
\lim_{N\to\infty} 
\mathbb{E}\big\| \widehat{Y}^{(N)}_t - \check{Y}^{(N)}_t \big\| = 0.
\]
\end{theorem} 
\begin{proof}
The proof of Theorem~\ref{thm:Y(N)->barY} is provided in Appendix~\ref{appendix:decentralizedStrat}. 
\end{proof}

\subsection{Phase 2: Agent Prediction Aggregation} \label{score_based_aggregation}

At each timestep, $t$, each agent produces a prediction, $Y^n_{t+1}$. These predictions are combined to produce a final prediction, $Y_{t+1}$, from the ensemble. A score-based approach is used as the primary means of mixing agent predictions. At each timestep, a score is computed for agent $n$ such that
\begin{align*}
\omega_{t+1}^n = \sum_{s = (t-T_a+1) \vee 0}^{t} e^{-\alpha_a (t-s)} \left \Vert y_{s} - Y^n_{s} \right \Vert^2
\end{align*}
where $\alpha_a$ is a discount factor and $T_a$ designates a prediction history length to incorporate. Aggregation weights, $w_{t+1}^n$, are computed via Softmax, as
\begin{align}
w_{t+1}^n = \frac{\exp(-\omega_{t+1}^n)}{\sum_{j=1}^N \exp(-\omega_{t+1}^j)}, \label{aggregation_equation}
\end{align}
and $Y_{t+1} = \sum_{n=1}^N w_{t+1}^n Y_{t+1}^n$. This approach is used to produce ensemble predictions at each timestep, unless otherwise stated.

\section{Numerical Validation} 
\label{sec:numerics}

In the experiments to follow, the decentralized strategy proposed in Section \ref{sec:MainResults} is compared to a greedy, local approach wherein each agent optimizes the local objective function 
\begin{align}
    J_n(\beta_t^n) = \sum_{s=(t-T)\vee 0}^{t-1} e^{-\alpha(t-1-s)} \left\Vert y_{s+1} - \left( \theta Y_s^n + \bar{\theta} Y^{(N)}_s + Z_s^n \beta_t^n \right) \right\Vert^2 + \gamma \Vert \beta_t^n \Vert^2, \label{ridge_objective}
\end{align}
where $T$ designates the length of past timeseries values considered. For each agent, $J_n(\beta_t^n)$ may be re-written as a standard ridge-regression problem and, therefore, admits the unique minimum derived in Appendix \ref{ridge_solution}. This is referred to as the greedy approach throughout the sections to follow. This is a natural baseline as it is also decentralized, fast to solve, and the objective incorporates similar components, including average prediction behavior in $Y_t^{(N)}$, but ignores the mean-field dynamics.

We consider two types of simple but important encoders, random feature networks (RFNs) \cite{huang2006universal, rahimi2008weighted} and echo-state networks (ESNs) \cite{grigoryeva2018echo}. For each agent $n$, define $A^{n} \in \mathbb{R}^{d_y \times d_x}$, $B^n \in \mathbb{R}^{d_y \times d_y}$, and $b^{n} \in \mathbb{R}^{d_y \times d_z}$ to be randomly sampled matrices, independent of $t$. Let $W_t^n \sim \mathcal{N}_{1 \times d_z}(0, 1)$ and $\sigma \in \mathbb{R}^{d_y \times 1}$ be a fixed scaling vector. For an input $x_t$, RFNs and ESNs produce encodings as
\begin{align*}
Z_{t, \text{RFN}}^n &= \operatorname{ReLU}\bullet\left(A^{n} [x_t, \dots, x_t ] + b^{n} + \sigma \, W_t^n \right), \\
Z_{t, \text{ESN}}^n &= \operatorname{Hard Sigmoid} \bullet \left (A^n [x_t, \dots , x_t] + B^n Z_{t-1, \text{ESN}}^n + b^n + \sigma \, W_t^n \right)
\end{align*}
respectively, where $Z_{t-1, \text{ESN}}^n$ is the latent encoding from the previous timestep. In all experiments, encoder components $A^n$, $B^n$, and $b^n$ are independently sampled from uniform Gaussian distributions.

Timeseries predictions are carried out for three synthetic series, Periodic Signal, Logistic Map \cite{echotorch}, and Concept Drift and two real-world series, Bank of Canada (BoC) Exchange Rate \cite{boc_exchange_rates} and Electricity Transformer Temperature (ETT) \cite{haoyietal-informer-2021}. The first three datasets consist of $200$ points, while the final two are longer, incorporating $2000$ timesteps. For all experiments, either the decentralized strategy or the greedy approach are applied at every timestep. Hyperparameter sweeps are conducted for each dataset independently to find the best parameters for both optimization strategies. To test the generalizability of these parameters and approaches, two validation series are drawn from the BoC and ETT datasets and performance is measured when applying the best parameters found during the sweeps for the analogous non-validation sets. More details on the datasets is found in Appendix \ref{datasets_metrics} and a comprehensive list of all hyperparameters and the best performing values appears in Appendix \ref{details_hyperparameters}. The appendix also contains a discussion of additional experimental details.

Regardless of agent-side optimization strategy, predictions are aggregated using the weighted-averaging in Equation \eqref{aggregation_equation}. Performance is quantified via root-mean squared error (RMSE). RMSE is computed for aggregated agent predictions, $Y_t$, along with the RMSE of the worst performing individual agent as well as bottom $20\%$ agents predictions over each timeseries. Reported RMSE is the average of three separate runs with different random seeds. Additional details are in Appendix \ref{datasets_metrics}.\footnote{All experimental code is found at: \url{https://anonymous.4open.science/r/MEMOE-723E}}

\begin{table}[H]%[ht!]
\caption{Average RMSE for aggregated RFN model predictions using the greedy and decentralized (Nash) strategies. Numbers in parentheses indicate the number of agents in the mixture.}
\small
\centering
\begin{tabular}{llllllll}
\toprule
Strategy & Periodic & Logistic & Concept & BoC & BoC Val. & ETT & ETT Val. \\
\midrule
Greedy (25)  & 9.9400e-1 & 2.5242e-1 & \textbf{2.4029e-1} & 1.0340e-2 & 6.5883e-3 & 2.6527e-2 & 5.2746e-2  \\
Nash (25)  & \textbf{9.1705e-1} & \textbf{1.9725e-1} & 2.5965e-1 & \textbf{6.9134e-3} & \textbf{5.0013e-3} & \textbf{2.6478e-2} & \textbf{5.1787e-2} \\
\midrule
Greedy (100) & 9.8567e-1  & 2.5047e-1  & \textbf{2.4333e-1}  & 1.0033e-2  & 7.0897e-3  & \textbf{2.6253e-2}  & 5.1658e-2 \\
Nash (100) & \textbf{9.1876e-1} & \textbf{1.9763e-1} & 2.5949e-1 & \textbf{6.8740e-3} & \textbf{5.0071e-3} & 2.6560e-2 & \textbf{5.1581e-2} \\
\midrule
Greedy (500) & 9.8158e-1 & 2.3814e-1 & \textbf{2.4695e-1} & 8.8927e-3 & 6.8933e-3 & \textbf{2.6293e-2} & 5.1817e-2 \\
Nash (500) & \textbf{9.1856e-1} & \textbf{1.9770e-1} & 2.6151e-1 & \textbf{6.9314e-3} & \textbf{5.0123e-3} & 2.6512e-2 & \textbf{5.1382e-2} \\
\bottomrule
\end{tabular}
\label{rfn_overall_results}
\end{table}

\subsection{Numerical Results} \label{main_results}

The results of timeseries prediction using RFN models for all datasets are shown in Table \ref{rfn_overall_results}. The decentralized (Nash) strategy produces better predictions for a majority of datasets considered, regardless of the size of the agent pool, sometimes by a large margin. For example, with $500$ agents, the strategy produces an RMSE that is $22$\% smaller than that of the greedy approach. 
% The Logistic Map RMSE is reduced by $17$\% when using the proposed strategy. 
For both datasets where the decentralized algorithm does not outperform the greedy approach, it is competitive. The algorithm also shows good generalization for the validation timeseries of the BoC Exchange and ETT datasets. Figure \ref{fig:esn_overlaid_prediction_comparison} displays predictions using the decentralized and greedy approaches for the Logistic Map and BoC Validation (restricted to the final 200 points for visualization) datasets.

\begin{figure}[ht!]
\centering
\includegraphics[width=0.48\linewidth]{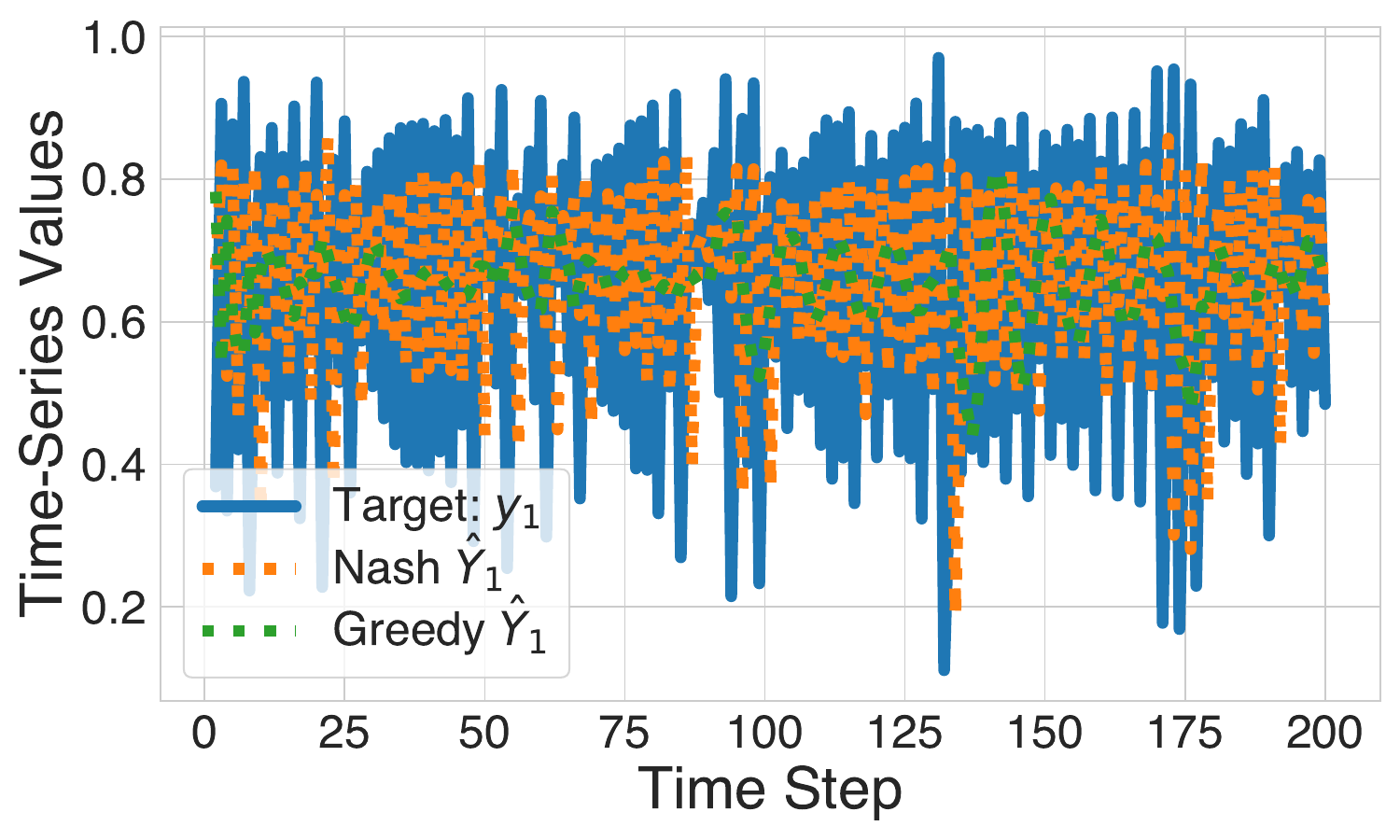}
\hfill
\includegraphics[width=0.48\linewidth]{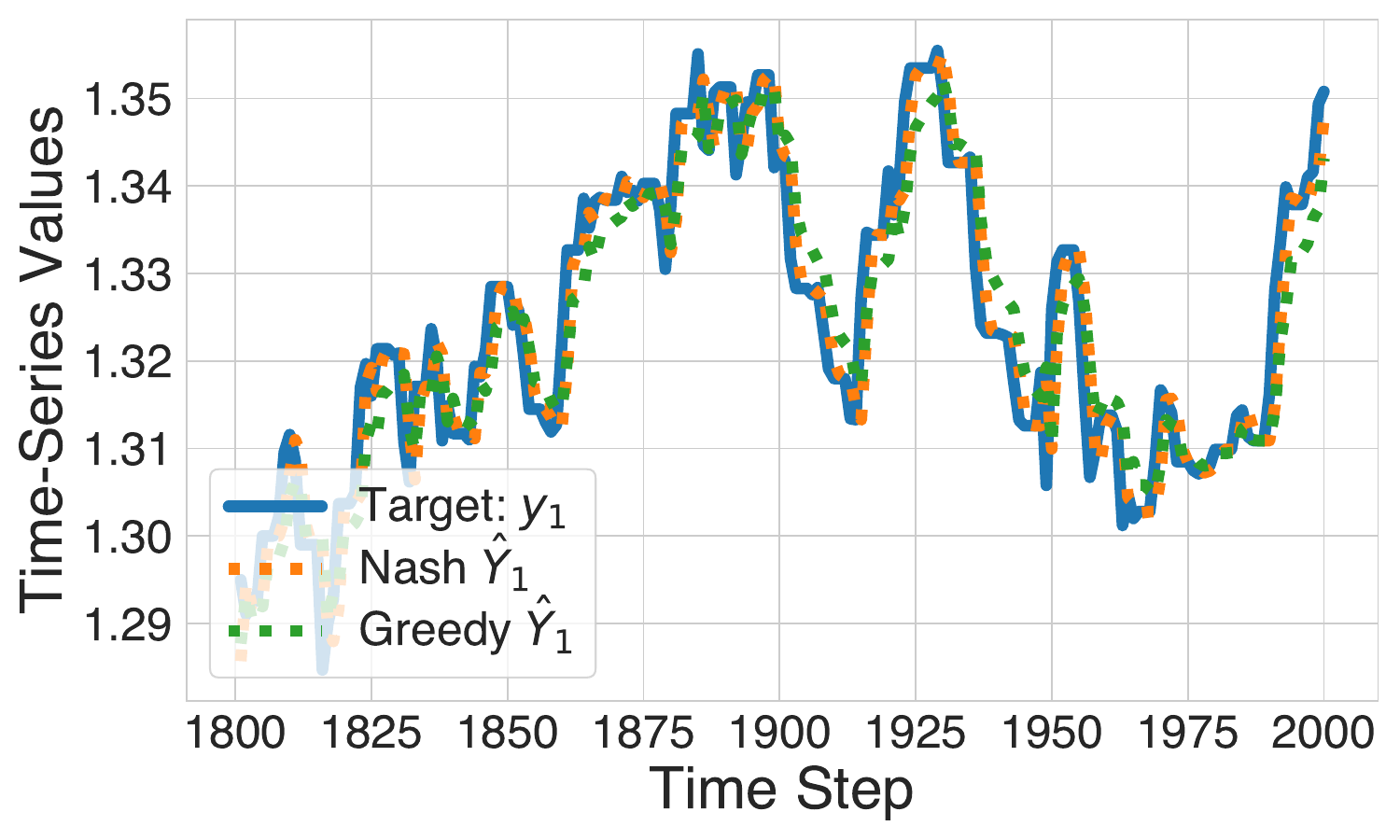}
\caption{Aggregated predictions compared with targets for RFN models with and without the decentralized strategy for the Logistic Map (left) and BoC Validation (right) timeseries.}
\label{fig:esn_overlaid_prediction_comparison}
\end{figure}

\begin{table}[ht!]
\caption{Average RMSE for aggregated ESN model predictions using the greedy and decentralized (Nash) strategies. Numbers in parentheses indicate the number of agents in the mixture.}
\small
\centering
\begin{tabular}{llllllll}
\toprule
Strategy & Periodic & Logistic & Concept & BoC & BoC Val. & ETT & ETT Val. \\
\midrule
Greedy (25)  & \textbf{8.5018e-1} & 2.5874e-1 & \textbf{2.0167e-1} & 7.1707e-3 & 5.1937e-3 & \textbf{2.6365e-2} & \textbf{5.1312e-2} \\
Nash (25)  & 1.0568e-0 & \textbf{1.9017e-1} & 2.4700e-1 & \textbf{6.7218e-3} & \textbf{4.8770e-3} & 2.6400e-2 & 5.1509e-2 \\
\midrule
Greedy (100) & \textbf{8.4452e-1} & 2.4312e-1 & \textbf{2.0690e-1} & 7.0577e-3 & 5.0340e-3 & \textbf{2.6357e-2} & \textbf{5.1246e-2} \\
Nash (100) & 1.0574e-0 & \textbf{1.9000e-1} & 2.5258e-1 & \textbf{6.7283e-3} & \textbf{4.8749e-3} & 2.6393e-2 & 5.1399e-2 \\
\midrule
Greedy (500) & \textbf{8.4862e-1} & 2.5557e-1 & \textbf{2.0770e-1} & 7.0293e-3 & 5.0490e-3 & 2.6358e-2 & \textbf{5.1216e-2} \\
Nash (500) & 1.0568e-0 & \textbf{1.8829e-1} & 2.5366e-1 & \textbf{6.7110e-3} & \textbf{4.8677e-3} & \textbf{2.6318e-2} & 5.1533e-2 \\
\bottomrule
\end{tabular}
\label{esn_overall_results}
\end{table}

Average RMSE results for experiments with ESN models are exhibited in Table \ref{esn_overall_results}. Overall performance between the two approaches is more mixed. The approach still outperforms the greedy approach on the Logistic Map, BoC Exchange, and BoC Exchange Validation datasets. It also remains competitive for the ETT series, eventually outperforming the greedy algorithm with $500$ agents. The dynamics of these results, and those for RFNs, are further investigated below.

Table \ref{worst_agents_metric} reports average RMSE of the single-worst performing agents for different settings. The results reinforce the theoretical expectation from Proposition \ref{prop:whynash} that the decentralized strategy yields pools of agents with better worst-case performance than that of the greedy strategy. That is, in nearly all settings, the worst-case agent performance when using the decentralized strategy is better, at times by a large margin, than that of the greedy approach. For RFNs, the only exception is the validation series of ETT with $500$ clients. For ESNs, there are a few additional instances where the greedy approach outperforms the decentralized strategy in this measure, but they are a minority. Appendix \ref{worst_performing_agents_study} further reports results for the bottom $20\%$ of agents, confirming the same trend.

\begin{table}[ht!]
\caption{Average RMSE for the worst agent using the greedy and decentralized (Nash) strategies for RFN (top) and ESN (bottom) models. Numbers in parentheses indicate the size of the agent pool.}
\small
\centering
\begin{tabular}{llllllll}
\toprule
Strategy & Periodic & Logistic & Concept & BoC & BoC Val. & ETT & ETT Val. \\
% \midrule
% Greedy (25) & 2.0991e-0  & 5.4231e-1  & 1.4334e-0  & 4.9302e-1  & 5.9375e-2  & 7.0084e-2  & \textbf{1.0578e-1} \\
% Nash (25) & \textbf{1.1321e-0}  & \textbf{3.2320e-1} & \textbf{3.5209e-1} & \textbf{2.1468e-2} & \textbf{1.5304e-2} & \textbf{5.4614e-2} & 1.1358e-1 \\
\midrule
Greedy (100) & 3.0682e-0  & 1.3444e-0  & 3.5085e+8  & 7.1508e-1  & 1.8677e-1  & 1.1308e-1  & 7.6078e-1 \\
Nash (100) & \textbf{1.3013e-0} & \textbf{3.4573e-1} & \textbf{3.7250e-1} & \textbf{2.4519e-2} & \textbf{1.7592e-2} & \textbf{5.7215e-2} & \textbf{1.2296e-1} \\
\midrule
Greedy (500) & 4.5723e-0  &  1.8720e-0  & 1.4425e-0  & 1.0749-0  & 3.4287e-1  & 1.7743e-1  & \textbf{1.2579e-1} \\
Nash (500) & \textbf{2.6089e-0} & \textbf{4.0196e-1} & \textbf{4.0190e-1} & \textbf{2.3189e-2} & \textbf{1.7076e-2} & \textbf{9.4500e-2} & 1.7696e-1 \\
\midrule
% \midrule
% Greedy (25) & \textbf{9.7359e-1}  & 3.3974e-1  & 1.0091e-0  & 1.4314e-2  & 2.9362e-2  & 3.5069e-2  & \textbf{5.5004e-2} \\
% Nash (25) & {1.1236e-0}  & \textbf{3.1879e-1} & \textbf{3.2196e-1} & \textbf{8.8190e-3} & \textbf{6.3557e-3} & \textbf{3.1017e-2} & 6.2458e-2 \\
\midrule
Greedy (100) & 2.7267e-0  & 7.0511e-1  & 1.1630e-0  & 7.4834e-2  & 4.9163e-2  & 4.1654e-2  & \textbf{6.0443e-2} \\
Nash (100) & \textbf{1.1581e-0}  & \textbf{4.4232e-1} & \textbf{3.2001e-1} & \textbf{8.8388e-3} & \textbf{6.3672e-3} & \textbf{3.2685e-2} & 6.4507e-2 \\
\midrule
Greedy (500) & \textbf{1.1035e-0}  & \textbf{4.5678e-1}  & 1.2358e-0  & 1.0128e-1  & 5.7545e-2  & \textbf{5.0294e-2}  & \textbf{8.1408e-2} \\
Nash (500) & 1.1789e-0  & 5.5784e-1 & \textbf{3.3075e-1} & \textbf{9.2607e-3} & \textbf{6.6757e-3} & 5.9513e-2 & 8.6272e-2 \\

\bottomrule
\end{tabular}
\label{worst_agents_metric}
\end{table}

The Concept Drift dataset is an interesting outlier in these results. The greedy approach is best for both model types. While the performance gap for ESN models is pronounced, the worst-case analysis also shows that the weakest agent in the greedy setting is poor. Noting this, we ablate the score-based aggregation and consider uniformly averaged predictions in Table \ref{uniform_averages}. In this setting, the decentralized strategy outperforms the greedy approach, implying that the greedy algorithm relies on performant aggregation for accurate predictions for this dataset. Based on the ablation in Appendix \ref{ablation_results}, Concept Drift is also the only dataset that uniformly benefits from setting $\bar{\theta} = 0.0$.

\begin{table}[ht!]
\caption{Average RMSE for uniformly averaged ESN predictions using the greedy and decentralized (Nash) strategies. Numbers in parentheses indicate the number of agents in the mixture.}
\small
\centering
\begin{tabular}{llllllll}
\toprule
Aggregation & Periodic & Logistic & Concept & BoC & BoC Val. & ETT & ETT Val. \\
% \midrule
% Greedy (25) & \textbf{8.4655e-1} & 2.4287e-1 & 2.7066e-1 & 7.1477e-3 & 5.0540e-3 & \textbf{2.6371e-2} & \textbf{5.1322e-2} \\
% Nash (25)  & {1.0565e-0} & \textbf{1.9781e-1} & \textbf{2.5308e-1} & \textbf{6.7218e-3} & \textbf{4.8771e-3} & {2.6396e-2} & {5.1499e-2} \\
% Score (25)  & 1.0568e-0 & \textbf{1.9017e-1} & \textbf{2.4700e-1} & \textbf{6.7218e-3} & \textbf{4.8770e-3} & 2.6400e-2 & 5.1509e-2 \\
\midrule
Greedy (100) & \textbf{8.3525e-1} & 2.3992e-1 & 2.9787e-1 & 6.9683e-3 & 5.0563e-3 & \textbf{2.6354e-2} & \textbf{5.1256e-2} \\
Nash (100) & {1.0572e-0} & \textbf{1.9770e-1} & \textbf{2.5835e-1} & \textbf{6.7283e-3} & \textbf{4.8749e-3} & {2.6389e-2} & {5.1389e-2} \\
% Score (100) & 1.0574e-0 & \textbf{1.9000e-1} & \textbf{2.5258e-1} & \textbf{6.7283e-3} & \textbf{4.8749e-3} & 2.6393e-2 & 5.1399e-2 \\
\midrule
Greedy (500) & \textbf{8.3125e-1} & 2.3943e-1 & 2.8935e-1 & 6.9753e-3 & 5.0397e-3 & {2.6359e-2} & \textbf{5.1218e-2} \\
Nash (500) & {1.0561e-0} & \textbf{1.9937e-1} & \textbf{2.5981e-1} & \textbf{6.7110e-3} & \textbf{4.8677e-3} & \textbf{2.6317e-2} & {5.1527e-2} \\
% Score (500) & 1.0568e-0 & \textbf{1.8829e-1} & \textbf{2.5366e-1} & \textbf{6.7110e-3} & \textbf{4.8677e-3} & 2.6318e-2 & 5.1533e-2 \\
\bottomrule
\end{tabular}
\label{uniform_averages}
\end{table}

\section{Conclusions, Limitations, and Future Work} \label{conclusions}

In this paper, we develop a decentralized learning policy for large-scale FL systems with proprietary agents. The proposed policy relies only on local information together with minimal shared statistics, thereby significantly reducing both computational and communication costs. We further establish that the decentralized policy constitutes an asymptotic Nash equilibrium as the number of agents grows to infinity. Numerical experiments on both synthetic and real-world datasets demonstrate that our approach outperforms natural greedy decentralized baselines under a worst-agent (minimax) performance criterion. A key limitation of the proposed framework is the homogeneity assumption imposed on the agents, which may be restrictive in practical FL environments with heterogeneous agentic models. An important direction for future work is to quantify the performance degradation under heterogeneity and to develop robust decentralized policies that explicitly account for agent-level variability. 

\bibliographystyle{plain}
\bibliography{Refs}

\begin{thebibliography}{10}

\bibitem{bai2026how}
L.~Bai et~al.
\newblock How do {AI} agents spend your money? {Analyzing} and predicting token consumption in agentic coding tasks.
\newblock {\em arXiv preprint arXiv:2604.22750}, 2026.

\bibitem{boc_exchange_rates}
{Bank of Canada}.
\newblock Historical noon and closing rates, Feb 2025.

\bibitem{Basar1999GameTheory}
Tamer Ba{\c{s}}ar and Geert~Jan Olsder.
\newblock {\em Dynamic Noncooperative Game Theory}.
\newblock SIAM: Society for Industrial and Applied Mathematics, Philadelphia, PA, 2nd edition, 1999.

\bibitem{bellman1954theory}
Richard Bellman.
\newblock The theory of dynamic programming.
\newblock {\em Bulletin of the American Mathematical Society}, 60(6):503--515, 1954.

\bibitem{DonahueKleinberg2021ModelSharing}
Kate Donahue and Jon Kleinberg.
\newblock Model-sharing games: Analyzing federated learning under voluntary participation.
\newblock In {\em Proceedings of the AAAI Conference on Artificial Intelligence}, 2021.

\bibitem{DonahueKleinberg2021OptimalityStability}
Kate Donahue and Jon Kleinberg.
\newblock Optimality and stability in federated learning: A game-theoretic approach.
\newblock In {\em Advances in Neural Information Processing Systems}, 2021.

\bibitem{faiz2024llmcarbon}
Ahmad Faiz, Sotaro Kaneda, Ruhan Wang, Rita Osi, Prateek Sharma, Fan Chen, and Lei Jiang.
\newblock Llmcarbon: Modeling the end-to-end carbon footprint of large language models.
\newblock In {\em International Conference on Learning Representations}, 2024.

\bibitem{grigoryeva2018echo}
Lyudmila Grigoryeva and Juan-Pablo Ortega.
\newblock Echo state networks are universal.
\newblock {\em Neural Networks}, 108:495--508, 2018.

\bibitem{GrimmettStirzaker01Probability}
Geoffrey~R Grimmett and David~R Stirzaker.
\newblock {\em Probability and Random Processes}.
\newblock Oxford University Press, May 2001.

\bibitem{HastieTibshiraniFriedman_2001ESLBook}
Trevor Hastie, Robert Tibshirani, and Jerome Friedman.
\newblock {\em The elements of statistical learning}.
\newblock Springer Series in Statistics. Springer-Verlag, New York, 2001.
\newblock Data mining, inference, and prediction.

\bibitem{hoffmann2022training}
Jordan Hoffmann, Sebastian Borgeaud, Arthur Mensch, Elena Buchatskaya, Trevor Cai, Eliza Rutherford, Diego de~Las~Casas, Lisa~Anne Hendricks, Johannes Welbl, Aidan Clark, et~al.
\newblock Training compute-optimal large language models.
\newblock {\em arXiv preprint arXiv:2203.15556}, 2022.

\bibitem{huang2006universal}
Guang-Bin Huang, Lei Chen, and Chee-Kheong Siew.
\newblock Universal approximation using incremental constructive feedforward networks with random hidden nodes.
\newblock {\em IEEE transactions on neural networks}, 17(4):879--892, 2006.

\bibitem{HuangCainesMalhame2007LQG}
Minyi Huang, Peter~E. Caines, and Roland~P. Malham{\'e}.
\newblock Large-population cost-coupled {LQG} problems with nonuniform agents: Individual-mass behavior and decentralized {$\varepsilon$}-{Nash} equilibria.
\newblock {\em IEEE Transactions on Automatic Control}, 52(9):1560--1571, 2007.

\bibitem{HuangMalhameCaines2006NCE}
Minyi Huang, Roland~P. Malham{\'e}, and Peter~E. Caines.
\newblock Large population stochastic dynamic games: Closed-loop {McKean--Vlasov} systems and the {Nash} certainty equivalence principle.
\newblock {\em Communications in Information and Systems}, 6(3):221--252, 2006.

\bibitem{HY21mfg}
Minyi Huang and Xuwei Yang.
\newblock Linear quadratic mean field game: Decentralized ${O(1/N)}$ {Nash} equilibria.
\newblock {\em Journal of Systems Science and Complexity}, 34(5):2003--2035, 2021.

\bibitem{HZ20}
Minyi Huang and Mengjie Zhou.
\newblock Linear quadratic mean field games: Asymptotic solvability and relation to the fixed point approach.
\newblock {\em IEEE Transactions on Automatic Control}, 65(4):1397--1412, 2020.

\bibitem{HuangJaimungalNourian2019MeanField}
Xuancheng Huang, Sebastian Jaimungal, and Mojtaba Nourian.
\newblock Mean-field game strategies for optimal execution.
\newblock {\em Applied Mathematical Finance}, 26(2):153--185, 2019.

\bibitem{kaplan2020scaling}
Jared Kaplan, Sam McCandlish, Tom Henighan, Tom~B. Brown, Benjamin Chess, Rewon Child, Scott Gray, Alec Radford, Jeffrey Wu, and Dario Amodei.
\newblock Scaling laws for neural language models.
\newblock {\em arXiv preprint arXiv:2001.08361}, 2020.

\bibitem{LasryLions2006MFGStationary}
Jean-Michel Lasry and Pierre-Louis Lions.
\newblock Jeux {\`a} champ moyen. i -- le cas stationnaire.
\newblock {\em Comptes Rendus Math{\'e}matique}, 343(9):619--625, 2006.

\bibitem{LasryLions2006MFGFiniteHorizon}
Jean-Michel Lasry and Pierre-Louis Lions.
\newblock Jeux {\`a} champ moyen. ii -- horizon fini et contr{\^o}le optimal.
\newblock {\em Comptes Rendus Math{\'e}matique}, 343(10):679--684, 2006.

\bibitem{LasryLions2007MeanFieldGames}
Jean-Michel Lasry and Pierre-Louis Lions.
\newblock Mean field games.
\newblock {\em Japanese Journal of Mathematics}, 2(1):229--260, 2007.

\bibitem{LeGall2014Matrix}
Fran\c{c}ois Le~Gall.
\newblock Powers of tensors and fast matrix multiplication.
\newblock In {\em Proceedings of the 39th International Symposium on Symbolic and Algebraic Computation}, ISSAC '14, page 296–303, New York, NY, USA, 2014. Association for Computing Machinery.

\bibitem{Li21Ditto}
Tian Li, Shengyuan Hu, Ahmad Beirami, and Virginia Smith.
\newblock Ditto: Fair and robust federated learning through personalization.
\newblock In {\em Proceedings of the 38th International Conference on Machine Learning}, volume 139 of {\em Proceedings of Machine Learning Research}, pages 6357--6368. PMLR, 2021.

\bibitem{Li20FedOptim}
Tian Li, Anit~Kumar Sahu, Manzil Zaheer, Maziar Sanjabi, Ameet Talwalkar, and Virginia Smith.
\newblock Federated optimization in heterogeneous networks.
\newblock In {\em Proceedings of Machine Learning and Systems}, volume~2, pages 429--450, 2020.

\bibitem{Liang2023DecentralizedOpenLoop}
X.~Liang and P.~E. Caines.
\newblock Decentralized open-loop strategies of linear quadratic mean field games.
\newblock In {\em IFAC-PapersOnLine}, volume~56, pages 11464--11469. Elsevier, 2023.

\bibitem{Liao23GlobalNashFL}
Xinting Liao, Chaochao Chen, Weiming Liu, Pengyang Zhou, Huabin Zhu, Shuheng Shen, Weiqiang Wang, Mengling Hu, Yanchao Tan, and Xiaolin Zheng.
\newblock Joint local relational augmentation and global {Nash} equilibrium for federated learning with non-iid data.
\newblock In {\em Proceedings of the 31st ACM International Conference on Multimedia}, MM '23, page 1536–1545, New York, NY, USA, 2023. Association for Computing Machinery.

\bibitem{liu2025budget}
T.~Liu et~al.
\newblock Budget-aware tool-use enables effective agent scaling.
\newblock {\em arXiv preprint arXiv:2511.17006}, 2025.

\bibitem{Liu22PrivacyFL}
Ziyu Liu, Shengyuan Hu, Zhiwei~Steven Wu, and Virginia Smith.
\newblock On privacy and personalization in cross-silo federated learning.
\newblock In {\em Proceedings of the 36th International Conference on Neural Information Processing Systems}, NIPS '22, Red Hook, NY, USA, 2022. Curran Associates Inc.

\bibitem{luccioni2024power}
Alexandra~Sasha Luccioni, Yacine Jernite, and Emma Strubell.
\newblock Power hungry processing: Watts driving the cost of ai deployment?
\newblock In {\em Proceedings of the 2024 ACM Conference on Fairness, Accountability, and Transparency (FAccT '24)}, pages 85--99, Rio de Janeiro, Brazil, 2024. ACM.

\bibitem{luccioni2023estimating}
Alexandra~Sasha Luccioni, Sylvain Viguier, and Anne-Laure Ligozat.
\newblock Estimating the carbon footprint of bloom, a 176b parameter language model.
\newblock {\em Journal of Machine Learning Research}, 24(253):1--15, 2023.

\bibitem{Mehrjou2021FederatedMeanFieldGame}
Arash Mehrjou.
\newblock Federated learning as a mean-field game.
\newblock {\em arXiv preprint arXiv:2107.03770}, 2021.

\bibitem{MurhekarEtAl2023IncentivesFL}
Aniket Murhekar, Milind Tambe, Kai Wang, and Yevgeniy Vorobeychik.
\newblock Incentives in federated learning: Equilibria, dynamics, and mechanisms for welfare maximization.
\newblock In {\em Advances in Neural Information Processing Systems}, 2023.

\bibitem{Murhekar23IncentivesFL}
Aniket Murhekar, Zhuowen Yuan, Bhaskar Ray~Chaudhury, Bo~Li, and Ruta Mehta.
\newblock Incentives in federated learning: Equilibria, dynamics, and mechanisms for welfare maximization.
\newblock In A.~Oh, T.~Naumann, A.~Globerson, K.~Saenko, M.~Hardt, and S.~Levine, editors, {\em Advances in Neural Information Processing Systems}, volume~36, pages 17811--17831. Curran Associates, Inc., 2023.

\bibitem{NVIDIA_RTX_A4000}
{NVIDIA Corporation}.
\newblock {NVIDIA RTX A4000 Graphics Card}.
\newblock \url{https://www.nvidia.com/en-us/products/workstations/rtx-a4000/}, 2026.
\newblock Accessed: 2026-05-02.

\bibitem{rahimi2008weighted}
Ali Rahimi and Benjamin Recht.
\newblock Weighted sums of random kitchen sinks: Replacing minimization with randomization in learning.
\newblock {\em Advances in neural information processing systems}, 21, 2008.

\bibitem{samsi2023words}
Siddharth Samsi, Dan Zhao, Joseph McDonald, Baolin Li, Adam Michaleas, Michael Jones, William Bergeron, Jeremy Kepner, Devesh Tiwari, and Vijay Gadepally.
\newblock From words to watts: Benchmarking the energy costs of large language model inference.
\newblock In {\em 2023 IEEE High Performance Extreme Computing Conference}, pages 1--9, 2023.

\bibitem{echotorch}
Nils Schaetti.
\newblock Echotorch: Reservoir computing with pytorch.
\newblock \url{https://github.com/nschaetti/EchoTorch}, 2018.

\bibitem{ShiriParkBennis2020UAVOnlineMFGFL}
Hamed Shiri, Jihong Park, and Mehdi Bennis.
\newblock Communication-efficient massive {UAV} online path control: Federated learning meets mean-field game theory.
\newblock {\em IEEE Transactions on Communications}, 2020.

\bibitem{SunEtAl2026PrivacyCommodityMFGRegretNet}
Sun et~al.
\newblock Privacy as commodity: {MFG-RegretNet} for large-scale privacy trading in federated learning.
\newblock {\em arXiv preprint arXiv:2603.28329}, 2026.

\bibitem{SunWuLi2024ReputationMFGFL}
Sun, Wu, and Li.
\newblock Reputation-aware incentive mechanism of federated learning: A mean field game approach.
\newblock {\em arXiv preprint}, 2024.

\bibitem{xiao2025reducing}
Yuan-An Xiao, Pengfei Gao, Chao Peng, and Yingfei Xiong.
\newblock Reducing cost of llm agents with trajectory reduction.
\newblock {\em arXiv preprint arXiv:2509.23586}, 2025.

\bibitem{Xie24mixed}
Wanyun Xie, Thomas Pethick, Ali Ramezani-Kebrya, and Volkan Cevher.
\newblock Mixed {Nash} for robust federated learning.
\newblock {\em Transactions on Machine Learning Research}, 2024.

\bibitem{Xu2024DecentralizedEpsilonNash}
Zhenhui Xu and Tielong Shen.
\newblock Decentralized $\epsilon$-{Nash} strategy for linear quadratic mean field games using a successive approximation approach.
\newblock {\em Asian Journal of Control}, 26(2):565--574, 2024.

\bibitem{yang2023regret}
Xuwei Yang, Anastasis Kratsios, Florian Krach, Matheus Grasselli, and Aurelien Lucchi.
\newblock {Synchronizing pretrained kernel regressors with applications to American option pricing}.
\newblock {\em Frontiers of Mathematical Finance}, 8:23--77, March 2026.

\bibitem{YTEK25}
Xuwei Yang, Fatemeh Tavakoli, David~B. Emerson, and Anastasis Kratsios.
\newblock Online federation for mixtures of proprietary agents with black-box encoders.
\newblock {\em arXiv preprint arXiv:2505.00216}, 2025.

\bibitem{YoonChoudhuryLoizou2025MultiplayerFL}
Taeho Yoon, Sayak~Ray Chowdhury, and Nicolas Loizou.
\newblock Multiplayer federated learning: Reaching equilibrium with communication-efficient algorithms.
\newblock {\em arXiv preprint arXiv:2501.08263}, 2025.

\bibitem{YuanWang2024PrivacyAwareSampling}
Yuan and Wang.
\newblock A game-theoretic framework for privacy-aware client sampling in federated learning.
\newblock {\em arXiv preprint arXiv:2412.05636}, 2024.

\bibitem{haoyietal-informer-2021}
Haoyi Zhou, Shanghang Zhang, Jieqi Peng, Shuai Zhang, Jianxin Li, Hui Xiong, and Wancai Zhang.
\newblock Informer: Beyond efficient transformer for long sequence time-series forecasting.
\newblock In {\em The Thirty-Fifth {AAAI} Conference on Artificial Intelligence, {AAAI} 2021, Virtual Conference}, volume~35, pages 11106--11115. {AAAI} Press, 2021.

\end{thebibliography}

\appendix 

\section{Proof of Secondary Results in Preliminaries Section}
\label{appednix:Secondary}
% \section{Strongest Weakest-Agent Characterization of Nash Equilibria}
% \label{s:nash_cash}

%\begin{proof}
\emph{{Proof of Proposition~\ref{prop:whynash}}.}
We assume that the agents are interchangeable ($J_1=\dots=J_N$).  
For each $n\in [N]$, upon defining
the $n^{th}$ agent's optimal value $
J_n^{\star}(\beta^{-n})\eqdef \inf_{\beta\in \mathcal{A}}\, J_n(\beta,\beta^{-n})$, 
we have
\begin{equation}
\label{eq:equilibrium2}
J_n(\beta^{n},\beta^{-n}) = J_n^{\star}(\beta^{-n})
\quad\mbox{and}\quad
J_n(\tilde{\beta}^{n},\beta^{-n}) \ge J_n^{\star}(\beta^{-n})
\end{equation}
for any $\tilde{\beta}^{n}\in \mathcal{A}$. If $\tilde{\beta}^{n}$ is not an optimal response for the $n^{th}$ agent relative to $\beta^{-n}$, then the latter inequality is strict.
Equivalently, define the ``systemic regret-type'' functional over $\mathcal{A}^N$ as sending any $(\beta^{n})_{n=1}^N\in \mathcal{A}^N$ to the non-negative value
\[
\mathcal{R}\big((\beta^{n})_{n=1}^N\big)
\eqdef
\max_{n\in [N]}\,
\Big(
    J_n(\beta^{n},\beta^{-n}) - J_n^{\star}(\beta^{-n})
\Big)
\]
then $(\beta^{n})_{n=1}^N$ is a Nash equilibrium if and only if
$
\mathcal{R}\big((\beta^{n})_{n=1}^N\big) =0
$. 
Note that, by construction, $\mathcal{R}$ takes non-negative values.  Since, in general, the benchmark values $\{J_n^{\star}(\beta^{-n})\}_{n=1}^N$ may be non-zero and need not be readily accessible; thus, we will instead consider the ``worst-agent's cost'' functional defined by
\begin{equation}
\label{eq:unormalized_regret}
\begin{aligned}
\mathcal{R}_{\downarrow}:\mathcal{A}^N &\rightarrow [0,\infty)
\\
(\beta^{n})_{n=1}^N
& \mapsto
\max_{n\in [N]}\, J_n(\beta^{n},\beta^{-n})
.
\end{aligned}
\end{equation}
If the Nash equilibrium $(\beta^{\star:n})_{n=1}^N$ is unique, then $J_n^{\star}\eqdef J_n^{\star}(\beta^{-\star:n})$ does not depend on the choice of the Nash strategies of all other agents as there is exactly one, namely $(\beta^{\star:n})_{n=1}^N$.  If, in addition,
$
J^{\star}\eqdef J_1^{\star}=J_n^{\star}
$
for all $n\in [N]$ (which holds under the agent's interchangeability), then
\[
\mathcal{R}\big((\beta^{n})_{n=1}^N\big)
=
\max_{n\in [N]}\,
\Big(
    J_n(\beta^{n},\beta^{-n}) - J^{\star}
\Big)
=
\Big(
    \max_{n\in [N]}
    \,
    J_n(\beta^{n},\beta^{-n})
\Big)
- J^{\star}
\]
and so, the minimizers of~\eqref{eq:unormalized_regret} coincide with those of $\mathcal{R}$.  Thus, assuming that a Nash equilibrium exists, any $(\beta^{n})_{n=1}^N\in\mathcal{A}^N$ is a Nash equilibrium if and only if it minimizes $\mathcal{R}$; i.e.\ if and only if
\[
\mathcal{R}\big((\beta^{n})_{n=1}^N\big)
=
\inf_{
(\tilde{\beta}^{n})_{n=1}^N\in \mathcal{A}^N
}
\,\mathcal{R}\big((\tilde{\beta}^{n})_{n=1}^N\big)
\]
showing the claim.
\hfill$\blacksquare$ 
%\end{proof}

\section{Proof of the Nash Equilibrium}
\label{appendix:proofNagentNash}

This section is devoted to proving Theorem~\ref{thm:NagentNash}. To facilitate the proof, we denote 
\begin{align} 
& Y_t = [ Y^{1\top}_t, \dots, Y^{N\top}_t ]^\top \in \mathbb{R}^{N d_y \times 1} , 
\notag \\ 
& \mathbf{1}_z  = \mathbf{1}_{1\times N} \otimes I_{d_z} 
 = \big[ I_{d_z}, \dots, I_{d_z} \big]
\in \mathbb{R}^{d_z \times d_z N} , 
 \quad 
  \mathbf{1}_{y} = \mathbf{1}_{1\times N} \otimes I_{d_y}
  = \big[ I_{d_y}, \dots, I_{d_y} \big]
  \in \mathbb{R}^{d_y \times d_y N} ,  
 \notag \\  
 &  \{ e_n \}_{n=1}^N: \, \mbox{the standard basis (column vectors) of $\mathbb{R}^N$},  
 \notag \\ 
 &  \mathbf{e}^y_n = e_n \otimes I_{d_y} = [0, \dots, I_{d_y} , \dots, 0]^\top \in \mathbb{R}^{d_y N\times d_y}, 
 \quad 
  \mathbf{e}^z_n = e_n \otimes I_{d_z} = [0, \dots, I_{d_z} , \dots, 0]^\top
  \in \mathbb{R}^{d_z N\times d_z} , 
 \notag \\ 
& \Theta_n =  e_n^\top  \otimes \theta 
 = \big[ 0, \dots, \theta, \dots, 0 \big]
\in \mathbb{R}^{d_y \times d_y N},  
 \quad 
 \Theta  = \mathbf{1}_{1\times N} \otimes 
  \theta 
  =  \big[\theta, \dots, \theta  \big] \in \mathbb{R}^{d_y \times N d_y}  , 
 \quad 
 \notag \\ 
& \overline\Theta = \mathbf{1}_{1\times N}\otimes \bar\theta
  =  \big[ \, \overline\theta, \dots, \overline\theta \,  \big]
 \in \mathbb{R}^{d_y \times d_y N} ,  
 \notag \\ 
 & \mathbf{\Theta} = [ \Theta_1^\top, \dots, \Theta_N^\top ]^\top
  = I_{N} \otimes \theta = \diag[\theta, \dots, \theta ] \in \mathbb{R}^{d_y N \times d_y N}  ,  
 \quad 
 \notag \\ 
& \overline{\mathbf{\Theta}} =
 \mathbf{1}_{N\times 1} \otimes \overline\Theta 
 = \mathbf{1}_{N\times N} \otimes \bar\theta 
 = 
 \begin{bmatrix} 
 \bar\theta , &  \dots, & \bar\theta \\
 \vdots & \ddots & \vdots \\ 
 \bar\theta , &  \dots, & \bar\theta 
 \end{bmatrix} 
 \in \mathbb{R}^{d_y N \times d_y N} , 
 \notag \\ 
 & 
 \mathbf{Z}_t = \diag[Z^1_t, \dots, Z^n_t] 
 = \sum_{n=1}^N \mathbf{e}^y_n Z^n_t 
  \mathbf{e}^{z\top}_n \in \mathbb{R}^{d_y N \times d_z N } . 
 \notag \\ 
 & 
 \mathbf{\beta}=\big[\beta^{1\top}, \dots, \beta^{N\top} \big]^\top . 
 \notag 
\end{align} 

Then~\eqref{hatYit+1finite} can be written as 
\begin{align} 
 & Y^n_{t+1} = [ \Theta_n  + (1/N) \overline\Theta ] Y_t + Z^n_t \beta^n_t , 
 \quad 
 n \in [N] ,  
 \notag \\ 
 & Y^{(N)}_{t+1} 
 = (\theta + \bar\theta ) Y^{(N)}_t + \frac{1}{N}  \sum_{n=1}^N Z^n_t \beta^n_t 
 %\notag \\ 
 %& \qquad 
 = (1/N)(\Theta +  \overline\Theta ) Y_t 
 + (1/N) \mathbf{1}_y \mathbf{Z}_t  \boldsymbol\beta_t , 
 \notag \\ 
 & Y_{t+1} =  [ \mathbf\Theta + (1/N) \overline{\mathbf\Theta} ] Y_t 
 + \sum_{n=1}^N \mathbf{e}^y_n Z^n_t \beta_t^n 
 %\notag \\ 
 %& \qquad 
 = [ \mathbf\Theta + (1/N) \overline{\mathbf\Theta} ] Y_t 
 +  \mathbf{Z}_t \boldsymbol{\beta}_t . 
 \notag 
\end{align}

The matrix-valued coefficients $\mathbf{G}(\cdot)\in \mathbb{R}^{N d_z \times N d_y}$ 
and $\mathbf{H}(\cdot)\in \mathbb{R}^{d_z \times 1}$ in the Nash equilibrium policy~\eqref{NashEqm:bfbeta} are defined by 
\begin{align} 
\mathbf{G}(t) 
= & 
- \Big\{ e^{-\alpha(T-1-t)}   
 \big[  \big( \kappa + \overline\kappa (1-1/N) \big) \widehat{\mathbf{A}}_1(t) 
  + \big( \overline\kappa (1-1/N)(-1/N) \big) \widehat{\mathbf{A}}_2(t) 
 +  \gamma I \big]  
 + \mathbf{A}(t)  
\Big\}^{-1} \cdot \notag \\ 
& \Big\{ e^{-\alpha(T-1-t)} \big[ (\kappa + \overline{\kappa}(1-1/N) ) \mathbf{D}(t) 
  - \overline{\kappa}(1-1/N)(1/N) \widehat{\mathbf{D}}(t) 
  + (\kappa/N) \overline{\mathbf{D}} 
 \, \big] + \mathbf{B}(t) \Big\} ,  
\label{bfG} \\ 
\mathbf{H}(t ) 
= & - \Big\{ e^{-\alpha(T-1-t)}   
 \big[  \big( \kappa + \overline\kappa (1-1/N) \big) \widehat{\mathbf{A}}_1(t) 
  + \big( \overline\kappa (1-1/N)(-1/N) \big) \widehat{\mathbf{A}}_2(t) 
 +  \gamma I \big]  
 + \mathbf{A}(t)  
\Big\}^{-1} \cdot  \notag \\ 
 &\big[  e^{-\alpha(T-1-t)} (-\kappa) \mathbf{F}(t) 
 + \mathbf{C}(t)  \big] , 
\label{bfH} 
\end{align} 
where 
\allowdisplaybreaks
\begin{align} 
\begin{cases} 
  \mathbf{A}(t) = \mathbb{E} \big\{ \big[ 
  P_1(t+1) \mathbf{e}^y_1 Z^1_t, 
 \dots, 
  P_N(t+1) \mathbf{e}^y_N Z^N_t
 \big]^\top  
  \mathbf{Z}_t \big\} ,   
  \\ 
  \widehat{\mathbf{A}}_1(t) =  
   \mathbb{E} \big\{ \diag[ Z^{1\top}_t Z^1_t, \dots, Z^{N\top}_t Z^N_t ] 
   \big\} 
   = \mathbb{E} \mathbf{Z}_t^\top \mathbf{Z}_t, 
   \\ 
  \widehat{\mathbf{A}}_2(t) 
 = 
 \mathbb{E} \big\{ \big[ Z^{1}_t, \dots, Z^N_t \big]^\top 
  \mathbf{1}_y \mathbf{Z}_t
 \big\} , 
  \\  
  \mathbf{B}(t) = \mathbb{E} \big[ P_1(t+1) \mathbf{e}^y_1 Z^1_t, \dots, P_N(t+1) \mathbf{e}^y_N Z^N_t \big]^\top 
  (\boldsymbol{\Theta} + \overline{\boldsymbol{\Theta}}/N ) ,  \\ 
   \mathbf{D}(t) =  \mathbb{E} \big[  \Theta_1^\top Z^1_t , \dots,  \Theta_N^\top Z^N_t \big]^\top   
  = 
  \mathbb{E} \mathbf{Z}_t^\top \boldsymbol{\Theta} ,
   \\ 
   \widehat{\mathbf{D}}(t) = 
  \mathbb{E} \big[ Z^1_t, \dots, Z^N_t \big]^\top \Theta  ,
   \\ 
    \overline{\mathbf{D}}(t) =  \mathbb{E} \big[ Z^1_t, \dots, Z^N_t \big]^\top \overline{\Theta} ,    
   \\ 
   \mathbf{C}(t) = \mathbb{E} \big[  S_1(t+1)^\top \mathbf{e}^y_1 Z^1_t , \dots, S_N(t+1)^\top \mathbf{e}^y_N Z^N_t \big]^\top , 
   \\ 
   \mathbf{F}(t) = 
   \mathbb{E} \big[ Z^1_t, \dots, Z^N_t \big]^\top   y_{t+1}  ,  
  \end{cases}
  \label{bfABCDF}
\end{align} 
and $\{P_n(\cdot)\}_{n=1}^N$ and $\{S_n(\cdot)\}_{n=1}^N$ are determined by the system~\eqref{Pn=Phin(t;P1:N)} and~\eqref{Sn=Psi(t;S1:N;P1:N)} 
\begin{align} 
\begin{aligned} 
 & \Phi_n(t; P_1, \dots, P_n ) \\ 
 \eqdef & \mathbf{G}^\top(t) 
\mathbf{Q}_n(t) \mathbf{G}(t) 
+ \mathbf{G}^\top(t) \mathbf{L}_n(t) 
+ \mathbf{L}_n^\top(t) 
\mathbf{G}(t) 
\\ 
& + e^{-\alpha(T-1-t)} \big[ \kappa \big( \Theta_n + \overline\Theta/N \big)^\top  
\big( \Theta_n + \overline\Theta/N \big)
+ \overline\kappa \big( \Theta_n - \Theta/N \big)^\top \big( \Theta_n - \Theta/N \big) \big] 
 \\
& + ( \boldsymbol{\Theta} + \overline{\boldsymbol{\Theta}}/N )^\top P_n(t+1) 
( \boldsymbol{\Theta} + \overline{\boldsymbol{\Theta}}/N ) , 
\quad t = 0, 1, \dots, T-1 , 
\end{aligned} 
\label{ODE:Pn}
\end{align} 
and 
\begin{align}
\begin{aligned} 
 & \Psi_n(t; S_1,\dots,S_N; P_1, \dots, P_n )  \\
  \eqdef & \mathbf{G}^\top(t) \mathbf{Q}_n(t) \mathbf{H}(t) 
 + \mathbf{G}^\top(t) \big[ - e^{-\alpha(T-1-t)} \kappa \mathbf{e}^z_n \mathbb{E}Z_t^{n\top} y_{t+1}  
  + \mathbb{E} \mathbf{Z}_t^\top S_n(t+1)
 \big] 
  \\ 
 & +   \big\{ e^{-\alpha(T-1-t)} \big[ \kappa ( \Theta_n + \overline\Theta/N )^\top \mathbb{E} Z^n_t \mathbf{e}^{z\top}_n 
 + \overline\kappa (\Theta_n - \Theta/N )^\top 
  \mathbb{E}  \big( Z^n_t \mathbf{e}_n^{z\top} - \mathbf{1}_y \mathbf{Z}_t /N \big)  
 \big] 
  \\ 
 & \qquad  + ( \mathbf{\Theta} + \overline{\mathbf{\Theta}}/N )^\top P_n(t+1) \mathbb{E}\mathbf{Z}_t 
 \big\} \mathbf{H}(t) 
  \\ 
 &  - e^{-\alpha(T-1-t)} \kappa (\Theta_n + \overline\Theta/N )^\top y_{t+1} 
  +   (\mathbf{\Theta} + \overline{\Theta}/N )^\top S_n(t+1)  , 
  \quad t = 0, 1, \dots, T-1, 
\end{aligned} 
 \label{ODE:Sn}
\end{align} 
where $\mathbf{Q}_n(t) \in \mathbb{R}^{Nd_z\times Nd_z}$ and 
$\mathbf{L}_n(t)\in\mathbb{R}^{Nd_z\times Nd_y}$ are given by 
\begin{align} 
\mathbf{Q}_n(t) = & 
 e^{-\alpha(T-1-t)} \Big\{ \kappa \mathbf{e}_n^z \mathbb{E}(Z_t^{n \top} Z^n_t) \mathbf{e}_n^{z \top} 
 + \overline\kappa \mathbb{E}\big[ \big( Z^n_t \mathbf{e}_n^{z\top} - \mathbf{1}_y \mathbf{Z}_t/N \big)^\top  \big( Z^n_t \mathbf{e}_n^{z\top} - \mathbf{1}_y \mathbf{Z}_t/N \big) \big] 
 + \gamma \mathbf{e}^z_n \mathbf{e}^{z\top}_n
 \Big\}
\notag \\ 
& +     
   \mathbb{E}( \mathbf{Z}_t^\top P_n(t+1)  \mathbf{Z}_t )  
 , 
  \label{bfQn} \\
%\end{align} 
%\begin{align} 
\mathbf{L}_n(t) 
= &
 e^{-\alpha(T-1-t)}  \big[ \kappa \mathbf{e}^z_n \mathbb{E} Z_t^{n \top} (\Theta_n - \overline\Theta /N) 
  + \overline\kappa \mathbb{E}\big( Z^n_t \mathbf{e}_n^{z\top} - \mathbf{1}_y \mathbf{Z}_t / N \big)^\top 
  (\Theta_n - \Theta/N)
 \big]  
 \notag \\ 
 & + \mathbb{E}\mathbf{Z}_t^\top P_n(t+1) ( \mathbf{\Theta} + \overline{\mathbf{\Theta}}/N ) . 
\label{bfLn}
\end{align}

%\begin{proof} 
\emph{Proof of Theorem~\ref{thm:NagentNash}.}
The proof is carried out in the similar manner as~\cite[Theorem 3]{YTEK25}. 
Define $\{V_n(t, \mathbf{y})\}_{n=1}^N$ as the value functions of the agents at time $t$, given $Y_t=\mathbf{y}$. The value functions at time $t=T$ with $Y_T=\mathbf{y}$ satify $V_n(T,\mathbf{y})=0$ and take the form 
\begin{align} 
V_n(T, \mathbf{y} ) =  \mathbf{y}^\top P_n(T) \mathbf{y} 
 + 2 S_n^\top(T) \mathbf{y} + r_n(T) , 
 \quad n \in [N] . 
\notag 
\end{align} 
with $P_n(T)=0$ and $S_n(T)=0$, for $n\in[N]$. 

Assume by induction that for some $t\in \{ 1, \dots, T-1 \}$ and for all $s\in\{t+1, \dots, T\}$, the value functions 
$\{V_n(s, \mathbf{y})\}_{n=1}^N$ at time $s$ with $\hat{Y}_s=\mathbf{y}$ take the affine-quadratic form 
\begin{align} 
V_n(s, \mathbf{y} ) =  \mathbf{y}^\top P_n(s) \mathbf{y} 
 + 2 S_n^\top(s) \mathbf{y} + r_n(s) , 
 \quad n = 1, \dots, N , 
\notag 
\end{align} 
with $P_n(s)$ and $S_n(s)$ determined by~\eqref{ODE:Pn} and~\eqref{ODE:Sn} on the time span $s\in\{t+1, \dots, T\}$. 
We show that at time $t$ with $Y_t=\mathbf{y}$, the equilibrium policy is given by 
$\mathbf{\beta}_t = \big[ \beta_t^{1\top}, \dots, \beta_t^{N\top} \big]^\top = \mathbf{G}(t) \mathbf{y}
 + \mathbf{H}(t)$, and the value functions $\{V_n(t, \mathbf{y})\}_{n=1}^N$ take the affine-quadratic form 
 \begin{align} 
V_n(t, \mathbf{y}) = \mathbf{y}^\top P_n(t) \mathbf{y} + 2 S_n(t)^\top \mathbf{y} + r_n(t), 
\quad n \in [N], 
\label{Vn(t,y):QuadraticAffine}
 \end{align} 
with $P_n(t)$ and $S_n(t)$ determined by~\eqref{ODE:Pn} and~\eqref{ODE:Sn}, respectively.

By the dynamic programming principle and the induction hypothesis, 
$\{V_n(t, \mathbf{y} )\}_{n=1}^N$ satisfy the system of dynamic programming equations:  
\allowdisplaybreaks
\begin{align} 
V_n(t, \mathbf{y} ) 
= & \min_{\beta^n} \mathbb{E} \Big\{  e^{-\alpha (T-1-t)} 
 \Big[ \kappa \big\| y_{t+1} -  ( \Theta_n + \overline\Theta / N  ) \mathbf{y} -  Z^n_t \beta^n   \big\|^2 
 \notag \\ 
& \hspace{3.5cm}  + \overline{\kappa} \big\| (\Theta_n - \Theta/N ) \mathbf{y} 
    +  \big(  Z^n_t \beta^n -  (1/N) \sum_{m=1}^N Z^m_t \beta^m \big)  \big\|^2
  + \gamma \| \beta^n \|^2   \Big]   \notag \\ 
  & \hspace{5.5cm} 
  +   V_n \Big( t+1, \big( \mathbf\Theta +  \overline{\mathbf\Theta}/N \big) \mathbf{y} 
 + \sum_{m=1}^N \mathbf{e}^y_m Z^m_t \beta^m \Big) 
   \Big| Y_t = \mathbf{y}  \Big\} \notag \\ 
   %%%%% 
  = & \min_{\beta^n} \mathbb{E} 
  \Big\{   e^{-\alpha (T-1-t)} 
 \Big[ \kappa \big\| y_{t+1} -   ( \Theta_n + \overline\Theta / N  ) \mathbf{y} -  Z^n_t \beta^n   \big\|^2 
 \notag \\ 
& \hspace{3.5cm} 
+ \overline{\kappa} \big\| (\Theta_n - \Theta/N ) \mathbf{y} 
    +  \big( Z^n_t \beta^n -  (1/N) \sum_{m=1}^N Z^m_t \beta^m \big)  \big\|^2
  + \gamma \| \beta^n \|^2   \Big]   \notag \\ 
 & \hspace{1cm}  +   \Big[ \big( \mathbf\Theta + \overline{\mathbf\Theta}/N \big) \mathbf{y} 
 + \sum_{m=1}^N \mathbf{e}^y_m Z^m_t \beta^m    \Big]^\top P_n(t+1) 
  \Big[ \big( \mathbf\Theta + \overline{\mathbf\Theta}/N \big) \mathbf{y} 
 + \sum_{m=1}^N \mathbf{e}^y_m Z^m_t \beta^m    \Big] \notag \\ 
 &   \hspace{3.5cm}  + 2    S_n(t+1)^\top \Big[ \big( \mathbf\Theta + \overline{\mathbf\Theta}/N \big) \mathbf{y} 
 + \sum_{m=1}^N \mathbf{e}^y_m Z^m_t \beta^m    \Big] 
  +   r_n(t+1)      \Big\} , 
  \label{DPeqnVi(t)}  
\end{align} 
where the condition $Y_{t}=\mathbf{y}$ is removed from the expectation because $Z_t$ is independent of $Y_{t}$. 
Since the objective~\eqref{DPeqnVi(t)} to be minimized is strictly convex in 
$\mathbf{\beta}=\big[\beta^{1\top}, \dots, \beta^{N\top} \big]^\top$, the first-order necessary conditions for minimization are also sufficient and therefore we have the unique set of equations 
\begin{align} 
0 = & e^{-\alpha(T-1-t)} \mathbb{E} \Big\{ - \kappa Z_t^{n \top} \big[ y_{t+1} -   ( \Theta_n + \overline\Theta / N  ) \mathbf{y} -  Z^n_t \beta^n \big] 
 \notag \\ 
 & \hspace{2.5cm}
 +  \overline{\kappa} (1-1/N)Z_t^{n \top}  \big[ (\Theta_n - \Theta/N ) \mathbf{y} 
    +  \big( Z^n_t \beta^n -  (1/N) \sum_{m=1}^N Z^m_t \beta^m \big)  \big] 
    + \gamma \beta^n 
    \Big\}
 \notag \\ 
 & + \mathbb{E} \Big\{  Z_t^{n\top} \mathbf{e}_n^{y\top}  P_n(t+1) 
  \Big[ \big( \mathbf\Theta + \overline{\mathbf\Theta}/N \big) \mathbf{y} 
 + \sum_{m=1}^N \mathbf{e}^y_m Z^m_t \beta^m   \Big] 
  + Z_t^{n \top} \mathbf{e}_n^{y \top}  S_n(t+1) 
   \Big\} , 
   \quad n \in [N] . 
 \notag 
\end{align} 
It then follows that  
\begin{align} 
& - e^{-\alpha(T-1-t)}  \Big\{   
\big[ \kappa + \bar{\kappa} (1-1/N) \big] 
\mathbb{E} (Z_t^{n \top} Z^n_t)
 + \gamma I  \Big\}  \beta^n 
\notag \\ 
%%%%%%%%%% 
= & e^{-\alpha(T-1-t)}    
    \overline{\kappa} (1-1/N) (-1/N) 
    \mathbb{E}  \big( Z_t^{n \top}   
       \mathbf{1}_y \mathbf{Z}_t \big) \boldsymbol{\beta}     
       + \mathbb{E} \big[  Z_t^{n \top} \mathbf{e}_n^{y \top}  P_n(t+1) 
    \mathbf{Z}_t \big] \mathbf{\beta}     
 \notag \\ 
 & + e^{-\alpha(T-1-t)} 
 \Big\{ \kappa \mathbb{E}Z^{n\top}_t ( \Theta_n + \overline\Theta/N ) 
  + \overline\kappa (1-1/N) \mathbb{E}Z^{n\top}_t 
  (\Theta_n - \Theta/N )
 \Big\} \mathbf{y}
 \notag \\ 
 & + \mathbb{E}  Z_t^{n \top} \mathbf{e}_n^{y \top}  P_n(t+1) 
   \big( \mathbf\Theta + \overline{\mathbf\Theta}/N \big) \mathbf{y} 
   \notag \\ 
& + e^{-\alpha(T-1-t)} (-\kappa) \mathbb{E} Z^{n\top}_t y_{t+1}
  + \mathbb{E} Z_t^{n \top} \mathbf{e}_n^{y \top}  S_n(t+1) 
 , 
  \quad n \in [N] . 
\label{eqn:betan=beta}
\end{align}

From~\eqref{eqn:betan=beta}, the equilibrium policy vector $\boldsymbol{\beta}=[\beta^{1\top}, \dots, \beta^{N\top}]^\top$ for all agents satisfies
\begin{align} 
 & - \Big\{ e^{-\alpha(T-1-t)}   
 \Big[  \big( \kappa + \overline\kappa (1-1/N) \big) \widehat{\mathbf{A}}_1(t) 
  +  \overline\kappa (1-1/N)(-1/N) \widehat{\mathbf{A}}_2(t) 
 + \gamma I \Big]  
 + \mathbf{A}(t)  
\Big\} 
 \boldsymbol{\beta}  
 \notag \\ 
 = & \Big\{ e^{-\alpha(T-1-t)} \big[ (\kappa + \overline{\kappa}(1-1/N) ) \mathbf{D}(t) 
  - \overline{\kappa}(1-1/N)(1/N) \widehat{\mathbf{D}}(t) 
  + (\kappa/N) \overline{\mathbf{D}} 
 \big] + \mathbf{B}(t) \Big\} \hat{\mathbf{y}}
 \notag \\ 
 & + e^{-\alpha(T-1-t)} (-\kappa) \mathbf{F}(t) 
 + \mathbf{C}(t) ,
 \notag 
\end{align} 
and can be written in the form
\begin{align} 
\boldsymbol{\beta}_t 
 = &  \mathbf{G} (t) \mathbf{y} +  \mathbf{H}(t) , 
\label{beta=Gy+H}  
\end{align} 
where $\mathbf{G}(t)$ and $\mathbf{H}(t)$ are defined by~\eqref{bfG} and~\eqref{bfH}. 

The $n$-th agent's equilibrium policy can be written as 
\begin{align} 
\beta^n_t = \mathbf{e}^{z\top}_n \mathbf{G}(t) \mathbf{y} + \mathbf{e}^{z\top}_n \mathbf{H}(t) , 
\quad n \in [N] .
\label{betan=Gny+Hn}
\end{align} 
We substitute~\eqref{betan=Gny+Hn} and~\eqref{beta=Gy+H} into~\eqref{DPeqnVi(t)} to obtain 
\begin{align} 
 & \mathbf{y}^\top P_n(t) \mathbf{y} + 2 S_n(t)^\top \mathbf{y} + r_n(t) \notag \\ 
= & \big[ \mathbf{G} (t) \mathbf{y} +  \mathbf{H}(t) \big]^\top \Big\{ e^{-\alpha(T-t-1)} \Big[ \kappa \mathbf{e}_n^z \mathbb{E}( Z_t^{n \top} Z^n_t )  \mathbf{e}_n^{z\top} 
+ \overline\kappa \mathbb{E} \big[ \big( Z^{n}_t \mathbf{e}_n^{z\top} 
- \mathbf{1}_y \mathbf{Z}_t /N \big)^\top \big( Z^{n}_t \mathbf{e}_n^{z\top} 
- \mathbf{1}_y \mathbf{Z}_t /N \big) 
\big] 
+ \gamma \mathbf{e}^z_n \mathbf{e}^{z\top}_n
\Big]  
\notag \\ 
 & \qquad  + \mathbb{E}\big[ \mathbf{Z}_t^\top P_n(t+1) \mathbf{Z}_t \big] \Big\} 
 \big[ \mathbf{G} (t) \mathbf{y} +  \mathbf{H}(t) \big] 
 \notag \\ 
 & + 2 \Big\{ e^{-r(T-1-t)} \big[ - \kappa \big( y_{t+1} - (\Theta_n + \overline\Theta/N ) \mathbf{y} \big)^\top \mathbb{E}Z^n_t \mathbf{e}_n^{z\top} 
  +  \overline\kappa \big(  (\Theta_n - \Theta/N ) \mathbf{y} \big)^\top \mathbb{E} 
   \big( Z^n_t e^{z\top}_n - \mathbf{1}_y \mathbf{Z}_t /N  \big) \big] 
   \notag \\ 
& \qquad 
+ \big( ( \boldsymbol{\Theta} +\overline{\boldsymbol{\Theta}} ) \mathbf{y} \big)^\top 
 P_n(t+1) \mathbb{E} \mathbf{Z}_t + S_n(t+1)^\top \mathbb{E} \mathbf{Z}_t 
 \Big\} \big[ \mathbf{G} (t) \mathbf{y} +  \mathbf{H}(t) \big] 
 \notag \\ 
 & + e^{-\alpha(T-1-t)} \big[ \kappa \big\| y_{t+1} - (\Theta_n + \overline\theta/N) \mathbf{y} \big\|^2 
  + \overline\kappa \big\| (\Theta_n - \Theta/N ) \mathbf{y} \big\|^2
 \big] 
 \notag \\ 
 & + \big[ ( \boldsymbol{\Theta} + \overline{\boldsymbol{\Theta}}/N ) \mathbf{y} \big]^\top P_n(t+1) 
 ( \boldsymbol{\Theta} + \overline{\boldsymbol{\Theta}}/N ) \hat{\mathbf{y}} 
 \notag \\ 
 & + 2 S_n(t+1)^\top ( \boldsymbol{\Theta} + \overline{\boldsymbol{\Theta}}/N ) \mathbf{y} 
  + r_n(t+1) . 
 \notag 
\end{align} 
Since the above equations hold for all possible $\mathbf{y}$, by matching the coefficients of the quadratic and linear terms of $\mathbf{y}$, we obtain $P_n(t)$ and $S_n(t)$ as determined by~\eqref{ODE:Pn} and~\eqref{ODE:Sn}. By induction, we have shown that the equilibrium policy takes the form~\eqref{NashEqm:bfbeta} with the matrix-valued coefficients determined by~\eqref{ODE:Pn} and~\eqref{ODE:Sn}. 
\hfill$\blacksquare$ 
%\end{proof} 

\section{Proof of the Centralized Policy under the Homogeneity Assumption}
\label{appendix:centralizedStrat_homogeneity}
%\label{appendix:proofPropPnSubmat:PiN1234} 

This section is devoted to proving Theorem~\ref{thm:betan_centralized}, which establishes the simplified form of the centralized policy under Assumption~\ref{assm:Zt}. 
In particular, deriving the simplified form~\eqref{betan_centralized} from the general representation~\eqref{NashEqm:bfbeta} hinges on showing that $\mathbf{G}(\cdot)$ and $\mathbf{H}(\cdot)$ admit the forms~\eqref{bfG:submat} and~\eqref{bfH:submat}, respectively. 
To this end, we establish the key matrix decompositions for $\{P_n\}_{n=1}^N$ and $\{S_n\}_{n=1}^N$, as stated in Propositions~\ref{prop:PnSubmat:PiN1234} and~\ref{prop:SnSubmat}. Then the matrix decompositions of $\mathbf{G}(\cdot)$ and $\mathbf{H}(\cdot)$ can be established subsequently. 

\begin{proposition}   
\label{prop:PnSubmat:PiN1234}
Suppose the system~\eqref{ODE:Pn} admits a solution $(P_1, \dots, P_N)$ for $t=0, \dots, T$. Then 
\begin{align} 
& P_1(t) = \begin{bmatrix} 
\Pi_1^N(t) & \Pi_2^N(t) & \Pi_2^N(t) & \cdots &  \Pi_2^N(t) \\ 
\Pi_2^{N\top}(t) & \Pi_3^N(t) & \Pi_4^N(t) & \cdots & \Pi_4^N(t) \\  
\Pi_2^{N\top}(t) & \Pi_4^N(t) & \Pi_3^N(t) & \cdots & \Pi_4^N(t) \\  
\vdots & \vdots & \vdots & \ddots & \vdots \\ 
\Pi_2^{N\top}(t)
& \Pi_4^N(t) & \Pi_4^N(t) & \cdots & \Pi_3^N(t) 
\end{bmatrix}
\in \mathbb{R}^{Nd_y \times Nd_y}, 
\quad \Pi^N_i(t) \in \mathbb{R}^{d_y\times d_y}, 
\quad i=1,2,3,4, 
\label{P1:PiN1234} 
\end{align} 
and, for $2\leq n \leq N$, 
\begin{align} 
& P_n(t) = J_{1n}^{y\top} P_1(t) J^y_{1n} , \quad 2\leq n \leq N . 
\label{Pn=J1nP1J1n}
\end{align} 
Here $J^y_{ij}$ denote the matrix obtained by exchanging the $i$-th and $j$-th rows of sub-matrices in 
$I_{N d_y}= ( \delta_{i,j}I_{d_y} )_{i,j=1}^N$. 
\end{proposition}

\begin{corollary}  
\label{cor:inv(hatA1+hatA2+hatA)} 
The inverse matrix in~\eqref{bfG} and ~\eqref{bfH} admits the block form 
\begin{align} 
& \Big\{ e^{-\alpha(T-1-t)}   
 \big[  \big( \kappa + \overline\kappa (1-1/N) \big) \widehat{\mathbf{A}}_1(t) 
  + \big( \overline\kappa (1-1/N)(-1/N) \big) \widehat{\mathbf{A}}_2(t) 
 + \gamma I \big]  
 + \mathbf{A}(t)  
\Big\}^{-1}
\notag \\ 
= & 
\begin{bmatrix} 
M^N(t) & E^N(t) & \dots & E^N(t) \\ 
E^N(t) & M^N(t) & \dots & E^N(t) \\ 
\vdots & \vdots & \ddots & \vdots \\ 
E^N(t) & E^N(t) & \dots & M^N(t)  
\end{bmatrix} 
\in \mathbb{R}^{N d_z  \times N d_z  } 
, 
\label{inv(bfA+bfhatA1+bfhatA2+gammaI):submat}
\end{align} 
where the $d_z \times d_z$-dimensional sub-matrices are given by
\begin{align} 
& M^N(t) = (F^N(t))^{-1} \big[ I - (N-1) K^N(t) E^N(t) \big] , 
\notag \\ 
& E^N(t) = - \big[ F^N(t) - (N-1) K^N(t) (F^N(t))^{-1} K^N(t)  + (N-2)K^N(t) \big]^{-1} K^N(t) (F^{N}(t))^{-1} , 
\notag \\ 
%\end{align} 
%\begin{align} 
& F^N(t) =   e^{-\alpha(T-1-t)}   
 \big[  \big( \kappa + \overline\kappa (1-1/N)^2 \big) 
 M_{Z}^{(2)}(t)
 + \gamma I \big]  
+ M_{Z, \, \Pi_1^N(t+1)}^{(2)}(t) ,   
\notag \\ 
& K^N(t) = 
e^{-\alpha(T-1-t)}   
 \overline\kappa (1-1/N)(-1/N)  
 ( M_Z^{(1)}(t) )^\top ( M_Z^{(1)}(t) )   
 + ( M_Z^{(1)}(t) )^\top \Pi_2^N(t+1) ( M_Z^{(1)}(t) )^\top . 
\notag 
\end{align} 
\end{corollary}

\begin{corollary} 
\label{cor:bfGsubmat}
The matrix $\mathbf{G}(t)$ defined by~\eqref{bfG} takes the form 
\begin{align} 
\mathbf{G}(t) 
= & 
\begin{bmatrix} 
G_1^N(t) & G_2^N(t) & \cdots & G_2^N(t) \\ 
G_2^N(t) & G_1^N(t) & \cdots & G_2^N(t) \\ 
\vdots & \vdots & \ddots & \vdots \\ 
G_2^N(t) & G_2^N(t) & \cdots & G_1^N(t) \\
\end{bmatrix} \in \mathbb{R}^{N d_z \times N d_y} , 
\quad 
G_1^N(t)\in \mathbb{R}^{d_z \times d_y}, \quad 
G_2^N(t)\in \mathbb{R}^{d_z \times d_y} , 
\label{bfG:submat} 
\end{align} 
where the sub-matrices are given by 
\allowdisplaybreaks
\begin{align} 
G_1^N(t) = &
- e^{-\alpha(T-1-t)} (\kappa + \overline{\kappa}(1-1/N) ) M^N(t)  
 \big( M_Z^{(1)}(t) \big)^\top  \theta  
\label{G1N} \\ 
& + e^{-\alpha(T-1-t)} \big[ M^N(t) + (N-1) E^N(t) \big]   
 \big( M_Z^{(1)}(t) \big)^\top
\big[ \overline{\kappa}(1-1/N)(1/N) \theta  
- (\kappa/N) \overline{\theta} \,  \big] 
\notag \\ 
& - \big[ M^N(t) 
 \big( M_Z^{(1)}(t) \big)^\top  \Pi_1^N(t+1) \theta   
 + (N-1) E^N(t)  \big( M_Z^{(1)}(t) \big)^\top  \Pi_2^N(t+1) \theta   \big] 
\notag \\ 
& -  \big[ M^N(t) + (N-1) E^N(t) \big] (1/N) 
\big( M_Z^{(1)}(t) \big)^\top 
 \big[ \Pi_1^N(t+1) + (N-1) \Pi_2^N(t+1) \big] \overline{\theta}   , 
\notag \\ 
%\end{align} 
%\begin{align} 
 G_2^N(t) = &  
 - e^{-\alpha(T-1-t)} (\kappa + \overline{\kappa}(1-1/N) ) E^N(t) 
 \big( M_Z^{(1)}(t) \big)^\top \theta  
\label{G2N} \\ 
& + e^{-\alpha(T-1-t)} \big[ M^N(t) + (N-1) E^N(t) \big]   
\big( M_Z^{(1)}(t) \big)^\top 
\big[ \overline{\kappa}(1-1/N)(1/N) \theta  
- (\kappa/N) \overline{\theta} \,  \big] 
\notag \\ 
& - \big[ E^N(t) \big( M_Z^{(1)}(t) \big)^\top  
\Pi_1^N(t+1) \theta 
+ \big( M^N(t) + (N-2)E^N(t) \big) 
 \big( M_Z^{(1)}(t) \big)^\top  
\Pi_2^N(t+1)\theta  \big] 
 \notag \\ 
 & -  \big[ M^N(t) + (N-1) E^N(t) \big] (1/N) 
 \big( M_Z^{(1)}(t) \big)^\top 
   \big[ \Pi_1^N(t+1) + (N-1) \Pi_2^N(t+1) \big] \overline{\theta}   .  
\notag  
\end{align} 
\end{corollary}  

\begin{corollary}
\label{cor:bfQnsubmat}
The matrices $\{ \mathbf{Q}_n(t)\}_{n=1}^N$ defined by~\eqref{bfQn} take the form 
\begin{align} 
\mathbf{Q}_1(t) = \begin{bmatrix} 
Q_1^N(t) & Q_2^N(t) & Q_2^N(t) & \cdots &  Q_2^N(t) \\ 
Q_2^{N\top}(t) & Q_3^N(t) & Q_4^N(t) & \cdots & Q_4^N(t) \\  
Q_2^{N\top}(t) & Q_4^N(t) & Q_3^N(t) & \cdots & Q_4^N(t) \\  
\vdots & \vdots & \vdots & \ddots & \vdots \\ 
Q_2^{N\top}(t)
& Q_4^N(t) & Q_4^N(t) & \cdots & Q_3^N(t) 
\end{bmatrix} , 
\quad 
 \mathbf{Q}_n(t) = J_{1n}^{y\top} \mathbf{Q}_1(t) J^y_{1n} , \quad 2\leq n \leq N , 
\label{bfQn:submat}
\end{align} 
where 
\begin{align} 
\begin{cases} 
 Q_1^N(t) = e^{-\alpha(T-1-t)} \Big[ \kappa M^{(2)}_Z(t) + \overline\kappa(1-1/N)^2 M^{(2)}_Z(t) + \gamma I \Big] 
 + M^{(2)}_{Z, \, \Pi_1^N(t+1)}(t) , 
 \\ 
%%%%%%%%%%%% 
 Q_2^N(t) =   e^{-\alpha(T-1-t)} \overline\kappa ( M^{(1)}_Z(t) )^\top  M^{(1)}_Z(t) (1-1/N)(-1/N) 
 \\ 
 \hspace{7cm} + (M^{(1)}_Z(t))^\top \Pi_2^N(t+1) M^{(1)}_Z(t) 
, 
 \\ 
%%%%%%%%%% 
 Q_3^N(t)  =  e^{-\alpha(T-1-t)} \overline\kappa M^{(2)}_Z(t) (1/N^2) 
 + M^{(2)}_{Z, \, \Pi_3^N(t+1)}(t) , 
  \\ 
 %%% 
  Q_4^N(t) = e^{-\alpha(T-1-t)} \overline\kappa ( M^{(1)}_Z(t) )^\top M^{(1)}_Z(t) (1/N^2) 
  + ( M^{(1)}_Z(t) )^\top \Pi_4^N(t+1) M^{(1)}_Z(t) .  
\end{cases} 
\notag 
%\label{QN1234}
\end{align} 
\end{corollary}

Substituting~\eqref{P1:PiN1234} into~\eqref{ODE:Pn} yields the system of equations satisfied by $\{\Pi_i^N\}_{i=1}^4$, as summarized in Corollary~\ref{cor:ODEs:PiN1234}.  
\begin{corollary}  
\label{cor:ODEs:PiN1234}
The sub-matrices $\Pi^N_i$, $i=1,2,3, 4$ in Proposition~\ref{prop:PnSubmat:PiN1234} satisfy the backward equations~\eqref{ODE:PiNi} for $t=0, 1, \dots, T$. 
\begin{comment} 
\begin{align} 
& \Pi_1^N(t) = \varphi_1^N( t; \Pi^N_1, \Pi^N_2, \Pi^N_3 , \Pi^N_4  ) ,  \notag \\ 
& \Pi_2^N(t) = \varphi_2^N( t; \Pi^N_1, \Pi^N_2, \Pi^N_3 , \Pi^N_4   ) , \notag \\ 
& \Pi_3^N(t) = \varphi_3^N( t; \Pi^N_1, \Pi^N_2, \Pi^N_3 , \Pi^N_4  ),   \notag \\ 
& \Pi_4^N(t) = \varphi_4^N( t; \Pi^N_1, \Pi^N_2, \Pi^N_3 , \Pi^N_4  ),   \notag \\ 
& \Pi_1^N(T) = \Pi_2^N(T) = \Pi_3^N(T) = \Pi_4^N(T) = 0 , 
\notag 
\end{align} 
\end{comment} 
\end{corollary}

By Corollary~\ref{cor:ODEs:PiN1234}, computing the high-dimensional $\{ P_n \}_{n=1}^N$ from~\eqref{ODE:Pn} reduces to computing the low-dimensional $\{ \Pi_i^N\}_{i=1}^4$ from~\eqref{ODE:PiNi}, which involves inverting only $d_z\times d_z$ matrices with time complexity of $\mathcal{O}((d_z)^{2.373})$ according to~\cite{LeGall2014Matrix}.

Analogous to~\eqref{prop:PnSubmat:PiN1234}, the sequence $\{S_n\}_{n=1}^N$ can be decomposed into repeating sub-matrices, as formalized in Proposition~\ref{prop:SnSubmat}. 
\begin{proposition} 
\label{prop:SnSubmat}
Suppose the system~\eqref{ODE:Pn} admits a solution $(P_1, \dots, P_N)$ for $t=0, 1, \dots, T$. 
Then the system~\eqref{ODE:Sn} has a solution $(S_1, \dots, S_N)$ satisfying 
\begin{align} 
& S_1(t) = 
\big[ \Xi_1^{N\top}(t) , \,  \Xi_2^{N\top}(t) , \, \dots, \,  \Xi_2^{N\top}(t) \big]^\top 
\in \mathbb{R}^{Nd_y \times 1}  , 
\quad \Xi^N_1(t) \in \mathbb{R}^{d_y\times 1}, 
\quad \Xi^N_2(t) \in \mathbb{R}^{d_y\times 1} , 
\label{S1:submat} 
\end{align} 
and, for $2\leq n \leq N$, 
\begin{align}
& S_n(t) = J_{1n}^{y\top} S_1(t)  . 
\label{Sn=J1nS1}
\end{align} 
Here $\{J^y_{ij}\}_{1\leq i,j \leq N}$ are as defined in Proposition~\ref{prop:PnSubmat:PiN1234}.  
%denote the matrix obtained by exchanging the $i$-th and $j$-th rows of sub-matrices in $I_{N d_y}= ( \delta_{i,j}I_{d_y} )_{i,j=1}^N$. 
\end{proposition} 

\begin{corollary}  
\label{cor:bfHsubmat}
The coefficients $\mathbf{H}(t)$ defined by~\eqref{bfH} take the form 
\begin{align} 
 \mathbf{H}(t) 
= & \begin{bmatrix} H^N(t) \\ \vdots \\ H^N(t) \end{bmatrix} 
\in \mathbb{R}^{Nd_z \times 1} , 
\quad 
H^N(t)\in \mathbb{R}^{d_z \times 1} , 
\label{bfH:submat}
\end{align} 
where the sub-matrices are given by 
\allowdisplaybreaks
\begin{align} 
 H^N(t) = & - [ M^N(t) + (N-1) E^N(t) ] 
 \big( M_Z^{(1)}(t) \big)^\top 
 \big[ e^{-\alpha(T-1-t)} (-\kappa) y_{t+1}
  +  \Xi_1^N(t+1) \big] . 
 \label{HN}  
\end{align} 
\end{corollary}  

\begin{corollary}    
\label{cor:ODEs:XiN12}
The sub-matrices $\Xi_i^N$, $i=1, 2$ in Proposition~\ref{prop:SnSubmat} satisfy the backward equations~\eqref{ODE:XiNi} for $t=0, 1, \dots, T$.  
\end{corollary}

With the above results, particularly~\eqref{bfG:submat} and~\eqref{bfH:submat}, it is straight forward to prove Theorem~\ref{thm:betan_centralized}.  

\emph{Proof of Theorem~\ref{thm:betan_centralized}.}
%\begin{proof} 
Substituting~\eqref{bfG:submat} and~\eqref{bfH:submat} into~\eqref{NashEqm:bfbeta}, we have  
\begin{align} 
 \widehat\beta^n_t = & \mathbf{e}^{z\top}_n \mathbf{G}(t) \big[ \widehat{Y}^{1\top}_t, \dots, \widehat{Y}^{N\top}_t \big]^\top  
  + \mathbf{e}^{z\top}_n \mathbf{H}(t) , 
 \notag \\ 
 = & 
 G_1^N(t) \widehat{Y}^n_t 
 + G_2^N(t) \sum_{j\neq n = 1}^N \widehat{Y}^j_t + H^N(t) , 
 \quad n \in [N], 
 \notag 
\end{align} 
which is~\eqref{betan_centralized}. 
%\end{proof}
\hfill$\blacksquare$

Now we prove Proposition~\ref{prop:PnSubmat:PiN1234}. We establish a series of preliminary results through the following lemmas. 
\begin{lemma} 
For any $i,j, n \in \{ 1, \dots, N\}$ and any $t=0,1,\dots, T-1$, we have 
\begin{align} 
& J^{z\top}_{ij} \mathbb{E}\big[ \big( Z^n_t \mathbf{e}_n^{z\top} - \mathbf{1}_y \mathbf{Z}_t/N \big)^\top  \big( Z^n_t \mathbf{e}_n^{z\top} - \mathbf{1}_y \mathbf{Z}_t/N \big) \big] J^z_{ij} 
\notag \\ 
 = & 
 \begin{cases} 
 \mathbb{E}\big[ \big( Z^n_t \mathbf{e}_n^{z\top} - \mathbf{1}_y \mathbf{Z}_t/N \big)^\top  \big( Z^n_t \mathbf{e}_n^{z\top} - \mathbf{1}_y \mathbf{Z}_t/N \big) \big] , 
 \quad  \text{ if } n \neq i, j  ; \\ 
 %%% 
  \mathbb{E}\big[ \big( Z^j_t \mathbf{e}_j^{z\top} - \mathbf{1}_y \mathbf{Z}_t/N \big)^\top  \big( Z^j_t \mathbf{e}_j^{z\top} - \mathbf{1}_y \mathbf{Z}_t/N \big) \big] , 
 \quad  \text{ if } n = i  ; \\ 
%%% 
 \mathbb{E}\big[ \big( Z^i_t \mathbf{e}_i^{z\top} - \mathbf{1}_y \mathbf{Z}_t/N \big)^\top  \big( Z^i_t \mathbf{e}_i^{z\top} - \mathbf{1}_y \mathbf{Z}_t/N \big) \big] , 
 \quad  \text{ if } n = j  . 
 \end{cases} 
\label{JE(Zn-Z/N)(Zn-Z/N)J}
\end{align} 
\end{lemma} 

\begin{proof} 
Under Assumption~\ref{assm:Zt}, we have 
\allowdisplaybreaks
\begin{align} 
 & \mathbb{E}\big[ \big( Z^n_t \mathbf{e}_n^{z\top} - \mathbf{1}_y \mathbf{Z}_t/N \big)^\top  \big( Z^n_t \mathbf{e}_n^{z\top} - \mathbf{1}_y \mathbf{Z}_t/N \big) \big] 
 \notag \\ 
 = & 
 \mathbb{E} \Big[ \big( Z^n_t \mathbf{e}_n^{z\top} -  \sum_{m=1}^N Z^m_t \mathbf{e}^{z\top}_m /N \big)^\top  
  \big( Z^n_t \mathbf{e}_n^{z\top} - \sum_{m=1}^N Z^m_t \mathbf{e}^{z\top}_m /N \big) \Big] 
 \notag \\ 
&  + \mathbb{E} \Big[ \sum_{m=1}^N  \mathbf{e}^{z}_m Z^{m\top}_t   \sum_{k=1}^N Z^k_t \mathbf{e}^{z\top}_k \Big] / N^2  
  \notag \\ 
  = & 
  (1 - 2/N ) \mathbb{E} \big[ \mathbf{e}_n^z Z^{n\top}_t Z^n_t \mathbf{e}_n^{z\top} \big] 
  - \mathbb{E} \big[ \mathbf{e}_n^z Z^{n\top}_t \sum_{m \neq n=1}^N Z^m_t \mathbf{e}^{z\top}_m /N \big] 
  - \mathbb{E} \big[ \sum_{m \neq n = 1}^N  \mathbf{e}^{z}_m Z^{m\top}_t Z^n_t \mathbf{e}_n^{z\top}/N \big] 
\notag \\ 
& + \mathbb{E} \Big[ \sum_{m=1}^N \mathbf{e}_m^z Z^{m\top}_t Z^m_t \mathbf{e}_m^{z\top} \Big] / N^2  
+ \mathbb{E} \Big[ \sum_{m\neq k=1}^N  \mathbf{e}^{z}_m Z^{m\top}_t  Z^k_t \mathbf{e}^{z\top}_k \Big] / N^2 
\notag \\ 
 = & 
 (1 - 1/N )^2 \mathbf{e}^z_n M^{(2)}_{Z} \mathbf{e}^{z\top}_n 
  -  \frac{1}{N} \sum_{m = 1 \neq n}^N \mathbf{e}^z_n \big( M^{(1)}_Z(t) \big)^\top M^{(1)}_Z(t) \mathbf{e}^{z\top}_m 
   -  \frac{1}{N} \sum_{m = 1 \neq n }^N \mathbf{e}^z_m \big( M^{(1)}_Z(t) \big)^\top M^{(1)}_Z(t) \mathbf{e}^{z\top}_n 
   \notag \\ 
& + \frac{1}{N^2}\sum_{m = 1\neq n}^N \mathbf{e}_m^z M^{(2)}_Z(t)  \mathbf{e}_m^{z\top} 
 + \frac{1}{N^2} \sum_{m\neq k = 1}^N \mathbf{e}^z_m \big( M^{(1)}_Z(t) \big)^\top M^{(1)}_Z(t) \mathbf{e}^{z\top}_k
\label{E(Znen - 1bfZ/N)T(Znen - 1bfZ/N)}
\end{align} 
and 
\begin{align} 
 & J^{z\top}_{ij} \mathbb{E}\big[ \big( Z^n_t \mathbf{e}_n^{z\top} - \mathbf{1}_y \mathbf{Z}_t/N \big)^\top  \big( Z^n_t \mathbf{e}_n^{z\top} - \mathbf{1}_y \mathbf{Z}_t/N \big) \big] J^z_{ij} 
 \notag \\ 
 = & 
 (1-1/N )^2 J^{z\top}_{ij} \mathbf{e}^z_n M^{(2)}_{Z} \mathbf{e}^{z\top}_n J^{z}_{ij} 
  -  \frac{1}{N} \sum_{m = 1 \neq n}^N J^{z\top}_{ij} \mathbf{e}^z_n \big( M^{(1)}_Z(t) \big)^\top M^{(1)}_Z(t) \mathbf{e}^{z\top}_m J^{z}_{ij} 
  \notag \\ 
&   -  \frac{1}{N} \sum_{m = 1 \neq n }^N J^{z\top}_{ij} \mathbf{e}^z_m \big( M^{(1)}_Z(t) \big)^\top M^{(1)}_Z(t) \mathbf{e}^{z\top}_n J^{z}_{ij}
   \notag \\ 
& + \frac{1}{N^2}\sum_{m = 1\neq n}^N J^{z\top}_{ij} \mathbf{e}_m^z M^{(2)}_Z(t)  \mathbf{e}_m^{z\top} J^{z}_{ij} 
 + \frac{1}{N^2} \sum_{m\neq k = 1}^N J^{z\top}_{ij} \mathbf{e}^z_m \big( M^{(1)}_Z(t) \big)^\top M^{(1)}_Z(t) \mathbf{e}^{z\top}_k
J^{z}_{ij} . 
\label{JE(Znen - 1bfZ/N)T(Znen - 1bfZ/N)J}
\end{align} 
Since 
\begin{align} 
J^{z\top}_{ij} \mathbf{e}^z_{m} 
= 
\begin{cases} 
\mathbf{e}^z_{m} \quad \text{ if } m \neq i, j; \\
\mathbf{e}^z_{j} \quad \text{ if } m = i; \\
\mathbf{e}^z_{i} \quad \text{ if } m = j ;  
\end{cases} 
\label{Je:3cases}
\end{align} 
By~\eqref{Je:3cases}, we compare the terms of~\eqref{JE(Znen - 1bfZ/N)T(Znen - 1bfZ/N)J} with the terms of~\eqref{E(Znen - 1bfZ/N)T(Znen - 1bfZ/N)}, and obtain that 

If $n\neq i, j$, then 
\allowdisplaybreaks 
\begin{align} 
 & J^{z\top}_{ij} \mathbb{E}\big[ \big( Z^n_t \mathbf{e}_n^{z\top} - \mathbf{1}_y \mathbf{Z}_t/N \big)^\top  \big( Z^n_t \mathbf{e}_n^{z\top} - \mathbf{1}_y \mathbf{Z}_t/N \big) \big] J^z_{ij} 
 \notag \\ 
%%%%% 
= & 
 (1-1/N )^2 \mathbf{e}^z_n M^{(2)}_{Z} \mathbf{e}^{z\top}_n  
  -  \frac{1}{N} \sum_{m = 1 \neq n}^N J^{z\top}_{ij} \mathbf{e}^z_n \big( M^{(1)}_Z(t) \big)^\top M^{(1)}_Z(t) \mathbf{e}^{z\top}_m J^{z}_{ij} 
  \notag \\ 
&   -  \frac{1}{N} \sum_{m = 1 \neq n }^N  \mathbf{e}^z_m \big( M^{(1)}_Z(t) \big)^\top M^{(1)}_Z(t) \mathbf{e}^{z\top}_n
   \notag \\ 
& + \frac{1}{N^2}\sum_{m = 1\neq n}^N J^{z\top}_{ij} \mathbf{e}_m^z M^{(2)}_Z(t)  \mathbf{e}_m^{z\top} J^{z}_{ij} 
 + \frac{1}{N^2} \sum_{m\neq k = 1}^N J^{z\top}_{ij} \mathbf{e}^z_m \big( M^{(1)}_Z(t) \big)^\top M^{(1)}_Z(t) \mathbf{e}^{z\top}_k
J^{z}_{ij} 
\notag \\ 
= & 
\mathbb{E}\big[ \big( Z^n_t \mathbf{e}_n^{z\top} - \mathbf{1}_y \mathbf{Z}_t/N \big)^\top  \big( Z^n_t \mathbf{e}_n^{z\top} - \mathbf{1}_y \mathbf{Z}_t/N \big) \big] ; 
\notag 
\end{align} 
if $n = i$, then 
\begin{align} 
 & J^{z\top}_{ij} \mathbb{E}\big[ \big( Z^i_t \mathbf{e}_i^{z\top} - \mathbf{1}_y \mathbf{Z}_t/N \big)^\top  
 \big( Z^i_t \mathbf{e}_i^{z\top} - \mathbf{1}_y \mathbf{Z}_t/N \big) \big] J^z_{ij} 
 \notag \\ 
 = & 
 (1-1/N )^2 J^{z\top}_{ij} \mathbf{e}^z_i M^{(2)}_{Z} \mathbf{e}^{z\top}_i J^{z}_{ij} 
  -  \frac{1}{N} \sum_{m = 1 \neq i}^N J^{z\top}_{ij} \mathbf{e}^z_i \big( M^{(1)}_Z(t) \big)^\top M^{(1)}_Z(t) \mathbf{e}^{z\top}_m J^{z}_{ij} 
  \notag \\ 
&   -  \frac{1}{N} \sum_{m = 1 \neq i }^N J^{z\top}_{ij} \mathbf{e}^z_m \big( M^{(1)}_Z(t) \big)^\top M^{(1)}_Z(t) \mathbf{e}^{z\top}_i J^{z}_{ij}
   \notag \\ 
& + \frac{1}{N^2}\sum_{m = 1\neq i}^N J^{z\top}_{ij} \mathbf{e}_m^z M^{(2)}_Z(t)  \mathbf{e}_m^{z\top} J^{z}_{ij} 
 + \frac{1}{N^2} \sum_{m\neq k = 1}^N J^{z\top}_{ij} \mathbf{e}^z_m \big( M^{(1)}_Z(t) \big)^\top M^{(1)}_Z(t) \mathbf{e}^{z\top}_k
J^{z}_{ij} 
\notag \\ 
%%%%% 
= & 
 (1-1/N )^2  \mathbf{e}^z_j M^{(2)}_{Z} \mathbf{e}^{z\top}_j 
  -  \frac{1}{N} \sum_{m = 1 \neq j}^N  \mathbf{e}^z_j \big( M^{(1)}_Z(t) \big)^\top M^{(1)}_Z(t) \mathbf{e}^{z\top}_m 
  \notag \\ 
&   -  \frac{1}{N} \sum_{m = 1 \neq j }^N  \mathbf{e}^z_m \big( M^{(1)}_Z(t) \big)^\top M^{(1)}_Z(t) \mathbf{e}^{z\top}_j
   \notag \\ 
& + \frac{1}{N^2}\sum_{m = 1\neq j}^N  \mathbf{e}_m^z M^{(2)}_Z(t)  \mathbf{e}_m^{z\top}  
 + \frac{1}{N^2} \sum_{m\neq k = 1}^N  \mathbf{e}^z_m \big( M^{(1)}_Z(t) \big)^\top M^{(1)}_Z(t) \mathbf{e}^{z\top}_k 
\notag \\ 
%%%%% 
 = & 
  \mathbb{E}\big[ \big( Z^j_t \mathbf{e}_j^{z\top} - \mathbf{1}_y \mathbf{Z}_t/N \big)^\top  \big( Z^j_t \mathbf{e}_j^{z\top} - \mathbf{1}_y \mathbf{Z}_t/N \big) \big] ; 
\notag 
\end{align} 
Similarly, if $n=j$, then 
\begin{align} 
 & J^{z\top}_{ij} \mathbb{E}\big[ \big( Z^j_t \mathbf{e}_j^{z\top} - \mathbf{1}_y \mathbf{Z}_t/N \big)^\top  \big( Z^j_t \mathbf{e}_j^{z\top} - \mathbf{1}_y \mathbf{Z}_t/N \big) \big] J^z_{ij} 
 = 
 \mathbb{E}\big[ \big( Z^i_t \mathbf{e}_i^{z\top} - \mathbf{1}_y \mathbf{Z}_t/N \big)^\top  
 \big( Z^i_t \mathbf{e}_i^{z\top} - \mathbf{1}_y \mathbf{Z}_t/N \big) \big] .
 \notag 
 \end{align} 
\end{proof} 

\begin{lemma} 
For any $i,j, n \in \{ 1, \dots, N\}$ and any $t=0,1,\dots, T-1$, we have 
\allowdisplaybreaks 
\begin{align} 
  J^{z\top}_{ij} \mathbb{E}( \mathbf{Z}_t^\top P_n(t+1)  \mathbf{Z}_t )  J^{z}_{ij}  
 =  \mathbb{E}( \mathbf{Z}_t^\top P^\dagger_n(t+1)  \mathbf{Z}_t )  . 
 \label{E(ZPZ)=JE(ZPdaggerZ)J}
\end{align} 
\end{lemma}

\begin{proof} 
 \allowdisplaybreaks 
\begin{align} 
 & J^{z\top}_{ij} \mathbb{E}( \mathbf{Z}_t^\top P_n(t+1)  \mathbf{Z}_t )  J^{z}_{ij}  
 \notag \\ 
 = & J^{z\top}_{ij} \mathbb{E} \Big[ \big( \sum_{m=1}^N \mathbf{e}_m^y Z^m_t \mathbf{e}_m^{z\top} \big)^\top 
  P_n(t+1) \big( \sum_{k=1}^N \mathbf{e}_k^y Z^k_t \mathbf{e}_k^{z\top} \big) \Big] 
  J^{z}_{ij} 
  \notag \\ 
  = & \sum_{m=1}^N \sum_{k=1}^N J^{z\top}_{ij} \mathbf{e}_m^z \mathbb{E} \big[ Z_t^{m\top}  \mathbf{e}^{y\top}_m P_n(t+1) \mathbf{e}^y_k Z^k_t \big]  \mathbf{e}_k^{z\top} J^z_{ij}
\notag \\ 
  = & \sum_{m=1}^N \sum_{k=1}^N J^{z\top}_{ij} \mathbf{e}_m^z 
  \mathbb{E} \big[ Z_t^{m\top}  \mathbf{e}^{y\top}_m J^y_{ij}
   \big( J^{y\top}_{ij} P_n(t+1) J^{y}_{ij} \big) J^{y\top}_{ij} \mathbf{e}^y_k Z^k_t \big]  \mathbf{e}_k^{z\top} J^{z}_{ij}
\notag \\ 
= & \sum_{m=1}^N \sum_{k=1}^N J^{z\top}_{ij} \mathbf{e}_m^z 
  \mathbb{E} \big[ Z_t^{m\top}  ( \mathbf{e}^{y\top}_m J^y_{ij} )
    P^\dagger_n(t+1) ( J^{y\top}_{ij}  \mathbf{e}^y_k ) Z^k_t  \big]  \mathbf{e}_k^{z\top} J^{z}_{ij}
\notag \\ 
%%%%%%% 
= & 
  \sum_{m=1}^N  J^{z\top}_{ij} \mathbf{e}_m^z  M^{(2)}_{ Z , \, \mathbf{e}^{y\top}_m J^y_{ij} P^\dagger_n(t+1) 
  J^y_{ij} \mathbf{e}^y_m}(t)   \mathbf{e}_m^{z\top} J^z_{ij}
  \notag \\ 
& + \sum_{m\neq k = 1}^N  ( J^{z\top}_{ij} \mathbf{e}_m^z )   \big( M_Z^{(1)}(t) \big)^\top  
 ( \mathbf{e}^{y\top}_m J^y_{ij} ) P_n^\dagger(t+1) 
 ( J^{y\top}_{ij} \mathbf{e}^y_k ) 
 M_Z^{(1)}(t)   ( \mathbf{e}_k^{z\top} J^z_{ij} )
\notag \\ 
%%% 
 = & 
  \sum_{m=1}^N  \mathbf{e}_m^z  M^{(2)}_{ Z , \, \mathbf{e}^{y\top}_m P^\dagger_n(t+1) \mathbf{e}^y_m}(t)   \mathbf{e}_m^{z\top} 
%  \notag \\ 
%& 
+ \sum_{m\neq k = 1}^N \mathbf{e}_m^z   \big( M_Z^{(1)}(t) \big)^\top  \mathbf{e}^{y\top}_m P^\dagger_n(t+1) \mathbf{e}^y_k 
 M_Z^{(1)}(t)   \mathbf{e}_k^{z\top} 
 \notag \\ 
 = & \mathbb{E}( \mathbf{Z}_t^\top P^\dagger_n(t+1)  \mathbf{Z}_t )  . 
 \notag 
\end{align} 
\end{proof} 

\begin{lemma} 
For any $i,j, n \in \{ 1, \dots, N\}$ and any $t=0,1,\dots, T-1$, we have 
\begin{align} 
J^{z\top}_{ij} \mathbf{A}(t; P_1, \dots, P_N ) 
 J^z_{ij} 
= & \mathbf{A}( t; P^\dagger_1, \dots, P^\dagger_{i-1}, P^\dagger_j, P^\dagger_{i+1}, \dots, 
P^\dagger_{j-1}, P^\dagger_i, P^\dagger_{j+1}, \dots, P^\dagger_N)  . 
\label{JTbfA(P)J=bfA(Pdagger)}
\end{align} 
\end{lemma} 

\begin{proof} 
\begin{align} 
 & J^{z\top}_{ij} \mathbf{A}(t) 
 J^z_{ij} 
 \notag \\ 
 = & J^{z\top}_{ij} \mathbf{A}(t; P_1, \dots, P_N ) 
 J^z_{ij} 
 \notag \\ 
 %%%%%%%%%% 
  = & J^{z\top}_{ij} \mathbb{E} \big\{ \big[ 
  P_1(t+1) \mathbf{e}^y_1 Z^1_t, 
 \dots, 
  P_N(t+1) \mathbf{e}^y_N Z^N_t
 \big]^\top  
 \big[ \mathbf{e}^y_1 Z^1_t, \dots, \mathbf{e}^y_N Z^N_t \big] \big\} 
 J^{z}_{ij}
 \notag \\ 
 %%%%% 
 = & J^{z\top}_{ij} \mathbb{E} \big\{ \big[ 
  P_1(t+1) \mathbf{e}^y_1 Z^1_t, 
 \dots, 
  P_N(t+1) \mathbf{e}^y_N Z^N_t
 \big]^\top  
 J^{y}_{ij}
  J^{y\top}_{ij}
 \big[ \mathbf{e}^y_1 Z^1_t, \dots, \mathbf{e}^y_N Z^N_t \big] J^{z}_{ij}  \big\}  
 \notag \\ 
  %%%%% 
 = & J^{z\top}_{ij} \mathbb{E} \big\{ \big[ 
  J^{y\top}_{ij} P_1(t+1) \mathbf{e}^y_1 Z^1_t, 
 \dots, 
  J^{y\top}_{ij} P_N(t+1) \mathbf{e}^y_N Z^N_t
 \big]^\top  
  J^{y\top}_{ij}
 \big[ \mathbf{e}^y_1 Z^1_t, \dots, \mathbf{e}^y_N Z^N_t \big] J^{z}_{ij} \big\} 
 \notag \\ 
 %%%%%%%%%%%%%%%%%% 
 = & 
 J^{z\top}_{ij} \mathbb{E} \big\{ \big[ 
   (J^{y\top}_{ij}P_1(t+1) J^y_{ij}) J^{y\top}_{ij}  \mathbf{e}^y_1 Z^1_t, 
 \dots, 
   (J^{y\top}_{ij} P_N(t+1) J^y_{ij})
  J^{y\top}_{ij} \mathbf{e}^y_N Z^N_t
 \big]^\top 
  J^{y\top}_{ij}
 \big[ \mathbf{e}^y_1 Z^1_t, \dots, \mathbf{e}^y_N Z^N_t \big] J^{z}_{ij}   
  \big\} 
 \notag \\ 
 %%%%%%
 = & 
 J^{z\top}_{ij} \mathbb{E} \big\{ \big[ 
   P^\dagger_1 J^{y\top}_{ij}  \mathbf{e}^y_1 Z^1_t, 
 \dots, 
    P^\dagger_N
  J^{y\top}_{ij} \mathbf{e}^y_N Z^N_t
 \big]^\top J^{y\top}_{ij}
 \big[ \mathbf{e}^y_1 Z^1_t, \dots, \mathbf{e}^y_N Z^N_t \big] J^{z}_{ij}  
  \big\} 
 \notag \\ 
 %%%%%%%%%% 
  = & 
 J^{z\top}_{ij} \mathbb{E} \big\{ \big[ 
   P^\dagger_1  \mathbf{e}^y_1 Z^1_t, 
 \dots, 
 P^\dagger_{i-1}  \mathbf{e}^y_{i-1} Z^{i-1}_t, \, 
 P^\dagger_i   \mathbf{e}^y_j Z^i_t, \, 
 P^\dagger_{i+1}  \mathbf{e}^y_{i+1} Z^{i+1}_t, 
 \dots 
 \notag \\ 
& 
\hspace{7cm}
\dots,  P^\dagger_{j-1}  \mathbf{e}^y_{j-1} Z^{j-1}_t, \, 
 P^\dagger_j   \mathbf{e}^y_i Z^{j}_t, \, 
 P^\dagger_{j+1}  \mathbf{e}^y_{j+1} Z^{j+1}_t, \,
 \dots, 
    P^\dagger_N
   \mathbf{e}_N Z_t
 \big]^\top 
 \cdot 
 \notag \\ 
&  J^{y\top}_{ij} \big[ \mathbf{e}^y_1 Z^1_t, \dots, \mathbf{e}^y_N Z^N_t \big] J^z_{ij}  
  \big\} 
 \notag \\ 
 %%%%%%%%%% 
  = & 
  \mathbb{E} \big\{ \big[ 
   P^\dagger_1  \mathbf{e}^y_1 Z^1_t, 
 \dots, 
 P^\dagger_{i-1}  \mathbf{e}^y_{i-1} Z^{i-1}_t, \, 
 P^\dagger_j   \mathbf{e}^y_i Z^j_t, \, 
 P^\dagger_{i+1}  \mathbf{e}^y_{i+1} Z^{i+1}_t, 
 \dots, 
 \notag \\ 
 &
 \hspace{7cm}
 \dots, P^\dagger_{j-1}  \mathbf{e}^y_{j-1} Z^{j-1}_t, \, 
 P^\dagger_i   \mathbf{e}^y_j Z^i_t, \, 
 P^\dagger_{j+1}  \mathbf{e}^y_{j+1} Z^{j+1}_t, \,
 \dots, 
    P^\dagger_N
   \mathbf{e}^y_N Z^N_t
 \big]^\top \cdot 
 \notag \\ 
&  \big[ \mathbf{e}^y_1 Z^1_t, \dots, \mathbf{e}^y_i Z^j_t, \dots, 
\mathbf{e}^y_j Z^i_t , \dots ,
\mathbf{e}^y_N Z^N_t \big]   
  \big\} . 
 \notag  
\end{align} 
Note that 
\begin{align} 
 & \mathbf{A}( t; P^\dagger_1, \dots, P^\dagger_{i-1}, P^\dagger_j, P^\dagger_{i+1}, \dots, 
P^\dagger_{j-1}, P^\dagger_i, P^\dagger_{j+1}, \dots, P^\dagger_N)  
\notag \\ 
  = & 
  \mathbb{E} \big\{ \big[ 
   P^\dagger_1  \mathbf{e}^y_1 Z^1_t, 
 \dots, 
 P^\dagger_{i-1}  \mathbf{e}^y_{i-1} Z^{i-1}_t, \, 
 P^\dagger_j   \mathbf{e}^y_i Z^i_t, \, 
 P^\dagger_{i+1}  \mathbf{e}^y_{i+1} Z^{i+1}_t, 
 \dots, 
 \notag \\ 
 & 
 \hspace{5cm}
 \dots 
  P^\dagger_{j-1}  \mathbf{e}^y_{j-1} Z^{j-1}_t, \, 
 P^\dagger_i   \mathbf{e}^y_j Z^j_t, \, 
 P^\dagger_{j+1}  \mathbf{e}^y_{j+1} Z^{j+1}_t, \,
 \dots, 
    P^\dagger_N
   \mathbf{e}^y_N Z^N_t
 \big]^\top \cdot 
 \notag \\ 
&  
 \hspace{8cm}
\big[ \mathbf{e}^y_1 Z^1_t, \dots, \mathbf{e}^y_i Z^i_t, \dots, 
\mathbf{e}^y_j Z^j_t , \dots ,
\mathbf{e}^y_N Z^N_t \big]   
  \big\} 
 \notag
\notag 
\end{align}

If  $m \neq i, j$, and $n\neq i, j$, we have  
\begin{align} 
 \big( J^{z\top}_{ij} \mathbf{A}(t; P_1, \dots, P_N ) 
 J^z_{ij} \big)_{m,n} 
 = &
 \mathbb{E}\big[ Z^{m\top}_t \mathbf{e}_m^{y \top} P_m(t+1) \mathbf{e}^y_n Z^n_t \big] 
 \notag \\ 
 = &
 \mathbb{E}\big[ Z^{m\top}_t (\mathbf{e}_m^{y \top} J^y_{ij}) ( J^{y\top}_{ij} P_m(t+1) J^y_{ij} ) 
  ( J^{y\top}_{ij} \mathbf{e}^y_n ) Z^n_t \big] 
 \notag \\
 = & \mathbb{E}\big[ Z^{m\top}_t \mathbf{e}_m^{y \top} 
  P^\dagger_m(t+1)  \mathbf{e}^y_n Z^n_t \big] 
  \notag \\ 
  = & \big( \mathbf{A}( t; P^\dagger_1, \dots, P^\dagger_{i-1}, P^\dagger_j, P^\dagger_{i+1}, \dots, 
P^\dagger_{j-1}, P^\dagger_i, P^\dagger_{j+1}, \dots, P^\dagger_N)  \big)_{m , n}
 \notag 
\end{align} 
If $m=i$, and $n=i$, then 
\begin{align} 
 \big( J^{z\top}_{ij} \mathbf{A}(t; P_1, \dots, P_N ) 
 J^z_{ij} \big)_{m,n} 
 = &
 \mathbb{E}\big[ Z^{j\top}_t \mathbf{e}_i^{y \top} P^\dagger_j(t+1) \mathbf{e}^y_i Z^j_t \big] 
 \notag \\ 
  = & \mathbb{E}\big[ Z^{i\top}_t \mathbf{e}_i^{y \top} P^\dagger_j(t+1) \mathbf{e}^y_i Z^i_t \big] 
  \notag \\ 
  = & 
  \big( \mathbf{A}( t; P^\dagger_1, \dots, P^\dagger_{i-1}, P^\dagger_j, P^\dagger_{i+1}, \dots, 
P^\dagger_{j-1}, P^\dagger_i, P^\dagger_{j+1}, \dots, P^\dagger_N)  \big)_{m , n}
\notag 
\end{align} 
If $m=i$ and $n=j$, then 
\begin{align} 
 \big( J^{z\top}_{ij} \mathbf{A}(t; P_1, \dots, P_N ) 
 J^z_{ij} \big)_{m,n} 
 = & \mathbb{E}\big[ Z^{j\top}_t \mathbf{e}_i^{y \top} P^\dagger_j(t+1) \mathbf{e}^y_j Z^i_t \big] 
 \notag \\ 
 = & \mathbb{E}\big[ Z^{j\top}_t \big] \mathbf{e}_i^{y \top} P^\dagger_j(t+1) \mathbf{e}^y_j 
  \mathbb{E} \big[ Z^i_t \big] 
 \notag \\ 
 = & \mathbb{E}\big[ Z^{i\top}_t \big] \mathbf{e}_i^{y \top} P^\dagger_j(t+1) \mathbf{e}^y_j 
  \mathbb{E} \big[ Z^j_t \big] 
 \notag \\ 
 = & \mathbb{E}\big[ Z^{i\top}_t \mathbf{e}_i^{y \top} P^\dagger_j(t+1) \mathbf{e}^y_j Z^j_t \big]  
 \notag \\ 
  = & 
  \big( \mathbf{A}( t; P^\dagger_1, \dots, P^\dagger_{i-1}, P^\dagger_j, P^\dagger_{i+1}, \dots, 
P^\dagger_{j-1}, P^\dagger_i, P^\dagger_{j+1}, \dots, P^\dagger_N)  \big)_{m , n} 
\notag 
 \end{align} 
If $m=i$ and $n\neq i, j$, then 
\begin{align} 
 \big( J^{z\top}_{ij} \mathbf{A}(t; P_1, \dots, P_N ) 
 J^z_{ij} \big)_{m,n} 
 = & \mathbb{E} \big[ Z^{j\top}_t \mathbf{e}^{y\top}_i P^\dagger_j(t+1) \mathbf{e}^y_n Z^n_t \big] 
 \notag \\ 
 = & \mathbb{E} \big[ Z^{j\top}_t \big] \mathbf{e}^{y\top}_i P^\dagger_j(t+1) \mathbf{e}^y_n 
  \mathbb{E} \big[ Z^n_t \big]
 \notag \\ 
  = & \mathbb{E} \big[ Z^{i\top}_t \big] \mathbf{e}^{y\top}_i P^\dagger_j(t+1) \mathbf{e}^y_n 
  \mathbb{E} \big[ Z^n_t \big]
 \notag \\ 
 = & \mathbb{E} \big[ Z^{i\top}_t  \mathbf{e}^{y\top}_i P^\dagger_j(t+1) \mathbf{e}^y_n 
   Z^n_t \big]
 \notag \\ 
 = & 
  \big( \mathbf{A}( t; P^\dagger_1, \dots, P^\dagger_{i-1}, P^\dagger_j, P^\dagger_{i+1}, \dots, 
P^\dagger_{j-1}, P^\dagger_i, P^\dagger_{j+1}, \dots, P^\dagger_N)  \big)_{m , n} 
\notag 
 \end{align} 
 Similarly, for $m=j$ and for all $n=1, \dots, N$, we have 
 \begin{align} 
 \big( J^{z\top}_{ij} \mathbf{A}(t; P_1, \dots, P_N ) 
 J^z_{ij} \big)_{m,n} 
 = & 
  \big( \mathbf{A}( t; P^\dagger_1, \dots, P^\dagger_{i-1}, P^\dagger_j, P^\dagger_{i+1}, \dots, 
P^\dagger_{j-1}, P^\dagger_i, P^\dagger_{j+1}, \dots, P^\dagger_N)  \big)_{m , n} 
\notag 
 \end{align} 
Summing up the above cases, we have~\eqref{JTbfA(P)J=bfA(Pdagger)}. 
\end{proof}

\begin{lemma} 
For any $i,j, n \in \{ 1, \dots, N\}$ and any $t=0,1,\dots, T-1$, we have 
\begin{align} 
  J^{z\top}_{ij} \mathbf{B}(t; P_1, \dots, P_N )  J^y_{ij}
 =  \mathbf{B}( t; P^\dagger_1, \dots, P^\dagger_{i-1}, P^\dagger_j, P^\dagger_{i+1}, \dots, 
P^\dagger_{j-1}, P^\dagger_i, P^\dagger_{j+1}, \dots, P^\dagger_N) . 
\label{JTbfB(P)J=bfB(Pdagger)}
\end{align} 
\end{lemma} 

\begin{proof} 
 \begin{align} 
& J^{z\top}_{ij} \mathbf{B}(t)  J^y_{ij} \notag \\  
= & J^{z\top}_{ij} \mathbf{B}(t; P_1, \dots, P_N )  J^y_{ij}
\notag \\
= & J^{z\top}_{ij}  \mathbb{E} \big[ P_1(t+1) \mathbf{e}^y_1 Z^1_t, \dots, P_N(t+1) \mathbf{e}^y_N Z^N_t \big]^\top 
  (\boldsymbol{\Theta} + \overline{\boldsymbol{\Theta}}/N )  J^y_{ij} 
\notag \\ 
= & \big\{ J^{z\top}_{ij}  \mathbb{E} \big[ P_1(t+1) \mathbf{e}^y_1 Z^1_t, \dots, P_N(t+1) \mathbf{e}^y_N Z^N_t \big]^\top 
 J^y_{ij} \big\} 
 \big\{ J^{y\top}_{ij} (\boldsymbol{\Theta} + \overline{\boldsymbol{\Theta}}/N )  J^y_{ij} \big\} 
\notag \\ 
= & J^{z\top}_{ij} \mathbb{E} \big[  J^{y\top}_{ij} P_1(t+1) \mathbf{e}^y_1 Z^1_t , \dots, J^{y\top}_{ij} P_N(t+1) \mathbf{e}^y_N Z^N_t   \big]^\top 
  (\boldsymbol{\Theta} + \overline{\boldsymbol{\Theta}}/N )  
\notag \\ 
= & J^{z\top}_{ij} \mathbb{E} \big[  ( J^{y\top}_{ij} P_1(t+1) J^y_{ij} ) J^{y\top}_{ij} \mathbf{e}^y_1 Z^1_t  , \dots, 
 ( J^{y\top}_{ij} P_N(t+1) J^y_{ij} ) J^{y\top}_{ij} \mathbf{e}^y_N Z^N_t   \big]^\top 
  (\boldsymbol{\Theta} + \overline{\boldsymbol{\Theta}}/N )  
\notag \\ 
%%%%%%% 
= & J^{z\top}_{ij} \mathbb{E} \big[  P_1^\dagger J^{y\top}_{ij} \mathbf{e}^y_1 Z_t  , \dots, 
  P_N^\dagger J^{y\top}_{ij} \mathbf{e}^y_N Z_t   \big]^\top 
  (\boldsymbol{\Theta} + \overline{\boldsymbol{\Theta}}/N )  
\notag \\ 
 %%%%%%%%%% 
  = & 
 J^{z\top}_{ij} \mathbb{E} \big\{ \big[ 
   P^\dagger_1  \mathbf{e}^y_1 Z^1_t, 
 \dots, 
 P^\dagger_{i-1}  \mathbf{e}^y_{i-1} Z^{i-1}_t, \, 
 P^\dagger_i   \mathbf{e}^y_j Z^i_t, \, 
 P^\dagger_{i+1}  \mathbf{e}^y_{i+1} Z^{i+1}_t, 
 \dots, 
 \notag \\ 
 & \hspace{3cm}
  \dots, P^\dagger_{j-1}  \mathbf{e}^y_{j-1} Z^{j-1}_t, \, 
 P^\dagger_j   \mathbf{e}^y_i Z^j_t, \, 
 P^\dagger_{j+1}  \mathbf{e}^y_{j+1} Z^{j+1}_t, \,
 \dots, 
    P^\dagger_N
   \mathbf{e}^y_N Z^N_t
 \big]^\top  
 (\boldsymbol{\Theta} + \overline{\boldsymbol{\Theta}}/N )    
  \big\} 
 \notag \\ 
 %%%%%%%%%% 
  = & 
  \mathbb{E} \big\{ \big[ 
   P^\dagger_1  \mathbf{e}^y_1 Z^1_t, 
 \dots, 
 P^\dagger_{i-1}  \mathbf{e}^y_{i-1} Z^{i-1}_t, \, 
 P^\dagger_j   \mathbf{e}^y_i Z^j_t, \, 
 P^\dagger_{i+1}  \mathbf{e}^y_{i+1} Z^{i+1}_t, 
 \dots, 
 \notag \\ 
 & \hspace{3cm} 
 \dots,  P^\dagger_{j-1}  \mathbf{e}^y_{j-1} Z^{j-1}_t, \, 
 P^\dagger_i   \mathbf{e}^y_j Z^i_t, \, 
 P^\dagger_{j+1}  \mathbf{e}^y_{j+1} Z^{i+1}_t, \,
 \dots, 
    P^\dagger_N
   \mathbf{e}^y_N Z^N_t
 \big]^\top 
 (\boldsymbol{\Theta} + \overline{\boldsymbol{\Theta}}/N )   
  \big\} 
 \notag \\ 
 %%%%%%%%%%% 
 = & 
    \big[ 
   P^\dagger_1  \mathbf{e}^y_1 \mathbb{E} Z^1_t, 
 \dots, 
 P^\dagger_{i-1}  \mathbf{e}^y_{i-1} \mathbb{E} Z^{i-1}_t, \, 
 P^\dagger_j   \mathbf{e}^y_i \mathbb{E} Z^j_t, \, 
 P^\dagger_{i+1}  \mathbf{e}^y_{i+1} \mathbb{E} Z^{i+1}_t, 
 \dots, 
 \notag \\ 
 & \hspace{3cm} 
 \dots,  P^\dagger_{j-1}  \mathbf{e}^y_{j-1} \mathbb{E} Z^{j-1}_t, \, 
 P^\dagger_i   \mathbf{e}^y_j \mathbb{E} Z^i_t, \, 
 P^\dagger_{j+1}  \mathbf{e}^y_{j+1} \mathbb{E} Z^{i+1}_t, \,
 \dots, 
    P^\dagger_N
   \mathbf{e}^y_N \mathbb{E} Z^N_t
 \big]^\top 
 (\boldsymbol{\Theta} + \overline{\boldsymbol{\Theta}}/N )   
 \notag \\ 
 %%%%%%%%%%% 
 = & 
    \big[ 
   P^\dagger_1  \mathbf{e}^y_1 \mathbb{E} Z^1_t, 
 \dots, 
 P^\dagger_{i-1}  \mathbf{e}^y_{i-1} \mathbb{E} Z^{i-1}_t, \, 
 P^\dagger_j   \mathbf{e}^y_i \mathbb{E} Z^i_t, \, 
 P^\dagger_{i+1}  \mathbf{e}^y_{i+1} \mathbb{E} Z^{i+1}_t, 
 \dots, 
 \notag \\ 
 & \hspace{3cm} 
 \dots,  P^\dagger_{j-1}  \mathbf{e}^y_{j-1} \mathbb{E} Z^{j-1}_t, \, 
 P^\dagger_i   \mathbf{e}^y_j \mathbb{E} Z^j_t, \, 
 P^\dagger_{j+1}  \mathbf{e}^y_{j+1} \mathbb{E} Z^{i+1}_t, \,
 \dots, 
    P^\dagger_N
   \mathbf{e}^y_N \mathbb{E} Z^N_t
 \big]^\top 
 (\boldsymbol{\Theta} + \overline{\boldsymbol{\Theta}}/N )   
 \notag \\ 
 %%%%%%%%%%% 
  = & 
   \mathbb{E} \big\{  \big[ 
   P^\dagger_1  \mathbf{e}^y_1 Z^1_t, 
 \dots, 
 P^\dagger_{i-1}  \mathbf{e}^y_{i-1} Z^{i-1}_t, \, 
 P^\dagger_j   \mathbf{e}^y_i  Z^i_t, \, 
 P^\dagger_{i+1}  \mathbf{e}^y_{i+1} Z^{i+1}_t, 
 \dots, 
 \notag \\ 
 & \hspace{3cm} 
 \dots,  P^\dagger_{j-1}  \mathbf{e}^y_{j-1}  Z^{j-1}_t, \, 
 P^\dagger_i   \mathbf{e}^y_j Z^j_t, \, 
 P^\dagger_{j+1}  \mathbf{e}^y_{j+1} Z^{i+1}_t, \,
 \dots, 
    P^\dagger_N
   \mathbf{e}^y_N  Z^N_t
 \big]^\top 
 (\boldsymbol{\Theta} + \overline{\boldsymbol{\Theta}}/N )  
 \big\} 
 \notag \\ 
 %%%%%%%%%%% 
 = & \mathbf{B}( t; P^\dagger_1, \dots, P^\dagger_{i-1}, P^\dagger_j, P^\dagger_{i+1}, \dots, 
P^\dagger_{j-1}, P^\dagger_i, P^\dagger_{j+1}, \dots, P^\dagger_N) . 
\notag 
\end{align} 
\end{proof} 

With the above lemmas, we are ready to prove Proposition~\ref{prop:PnSubmat:PiN1234}. 

%\begin{proof}
\emph{Proof of Proposition~\ref{prop:PnSubmat:PiN1234}.}
Let $J^z_{ij}$ denote the matrix obtained by exchanging the $i$-th and $j$-th rows of sub-matrices in 
$I_{N d_z}= ( \delta_{i,j}I_{d_z} )_{i,j=1}^N$. 
It is easy to check that 
$J^y_{ij}=J^y_{ji}$ and $J^y_{ij}=J^{y\top}_{ij} = (J^y_{ij})^{-1}$ for all $i$, $j$, 
and the same properties hold for $J^z_{ij}$. 

\begin{comment} 
Multiplying both sides of \eqref{ODE:Pn} from the left by $J^{y\top}_{ij}$ and from the right by $J^y_{ij}$, we obtain 
\begin{align} 
J_{ij}^\top P_n(t) J_{ij}
=& J_{ij}^\top \mathbf{G}^\top(t) 
\mathbf{Q}_n(t) \mathbf{G}(t) J_{ij}
+ J_{ij}^\top \mathbf{G}^\top(t) \mathbf{L}_n(t) J_{ij} 
+ J_{ij}^\top \mathbf{L}_n^\top(t) J_{ij}
\mathbf{G}(t) 
\notag 
\\ 
&  + e^{-\alpha(T-1-t)} \big[ \kappa J_{ij}^\top 
 \big( \Theta_n + \overline\Theta/N \big)^\top  \big( \Theta_n + \overline\Theta/N \big) J_{ij}
+ \overline\kappa J_{ij}^\top 
 \big( \Theta_n - \Theta/N \big)^\top \big( \Theta_n - \Theta/N \big)  J_{ij} \big] 
\notag \\ 
& + J_{ij}^\top ( \boldsymbol{\Theta} + \overline{\boldsymbol{\Theta}}/N )^\top P_n(t+1) 
( \boldsymbol{\Theta} + \overline{\boldsymbol{\Theta}}/N ) J_{ij} , 
\quad 0\leq t < T ,  
\notag 
 \\ 
J_{ij}^\top P_n(T)J_{ij} = & 0 , \quad n = 1, \dots, N, 
\notag 
\end{align} 
\end{comment}

Note that $J_{ij}^\top   P_n(t) J_{ij}$ : 
\begin{itemize} 
\item exchange of $i$-th and $j$-th diagonal sub-matrices of $P_n(t)$; 

\item exchange the submatrix in the $i$-th row and $j$-th column and the submatrix in the $j$-th row and $i$-th column. 

\item exchange of $i$-th and $j$-th rows of sub-matrices, and exchange the $i$-th and $j$-th columns of sub-matrices
\end{itemize} 

Denote 
\begin{align} 
& \mathbf{A}(t) = \mathbf{A}(t; P_1, \dots, P_N) , 
\quad \mathbf{B}(t) = \mathbf{B}(t; P_1, \dots, P_N), 
 \quad  
 \mathbf{G}(t) 
=  \mathbf{G}(t; P_1, \dots, P_N) , 
\notag \\ 
& \mathbf{Q}_n(t) 
= \mathbf{Q}_n(t; P_n), \quad 1\leq n \leq N .  \notag 
\end{align} 
%\begin{comment}
Then~\eqref{ODE:Pn} can then be written as 
\begin{align} 
\begin{aligned} 
\Phi_n(t; P_1, \dots, P_N) 
\eqdef & \mathbf{G}^\top(t; P_1, \dots, P_N) 
\mathbf{Q}_n(t; P_n ) 
\mathbf{G}(t; P_1, \dots, P_N )  
 \\ 
& + \mathbf{G}^\top( t; P_1, \dots, P_N ) 
\mathbf{L}_n(t) 
+ \mathbf{L}_n^\top(t) 
\mathbf{G}(t; P_1, \dots, P_N ) 
\\ 
& 
 + e^{-\alpha(T-1-t)} \big[ \kappa  
 \big( \Theta_n + \overline\Theta/N \big)^\top  \big( \Theta_n + \overline\Theta/N \big) 
+ \overline\kappa  
 \big( \Theta_n - \Theta/N \big)^\top \big( \Theta_n - \Theta/N \big)  \big] 
 \\ 
& + ( \boldsymbol{\Theta} + \overline{\boldsymbol{\Theta}}/N )^\top P_n(t+1) 
( \boldsymbol{\Theta} + \overline{\boldsymbol{\Theta}}/N ) . 
\end{aligned} 
\label{Phin(t;P1:N)}
\end{align} 
%\end{comment} 

Under Assumption~\ref{assm:Zt}, the matrix-valued coefficients~\eqref{bfABCDF}, \eqref{bfQn} and~\eqref{bfLn} can be written as 
\allowdisplaybreaks
\begin{align} 
\begin{cases} 
  \mathbf{A}(t) 
  = \big( \mathbf{A}^{m, n}(t) \big)_{1\leq m, n\leq N} : \quad 
  \\ 
 \qquad   \mathbf{A}^{m,m}(t) = \big( M^{(1)}_Z(t) \big)^\top \mathbf{e}_m^{y \top} P_m(t+1) \mathbf{e}^y_m M^{(1)}_Z(t) , 
 \quad 1 \leq m \leq N  ; 
 \\
 \qquad \mathbf{A}^{m,n}(t) = M_{Z, \, \mathbf{e}_m^{y \top} P_m(t+1) \mathbf{e}^y_n }^{(2)}(t) , 
 \quad 1 \leq m \neq n \leq N ; 
  %= \mathbb{E} \big\{ \big[ P_1(t+1) \mathbf{e}_1 Z^1_t,  \dots, 
  % P_N(t+1) \mathbf{e}_N Z^N_t \big]^\top  
  %\mathbf{Z}_t \big\} ,   
  \\ 
  \widehat{\mathbf{A}}_1(t) 
  = I_{N} \otimes M_{Z}^{(2)}(t) ; 
  %=   \mathbb{E} \big\{ \diag[ Z^{1\top}_t Z^1_t, \dots, Z^{N\top}_t Z^N_t ] 
  % \big\} 
  % = \mathbb{E} \mathbf{Z}_t^\top \mathbf{Z}_t, 
   \\ 
  \widehat{\mathbf{A}}_2(t) 
 = \big( \widehat{\mathbf{A}}_2^{m,n}(t) \big)_{1\leq m, n \leq N} : \quad 
 \\  
 \qquad \widehat{\mathbf{A}}_2^{m,m}(t) = M_{Z}^{(2)}(t) , \quad  1\leq m \leq N; 
 \\ 
    \qquad  \widehat{\mathbf{A}}_2^{m,n}(t) = \big( M_{Z}^{(1)}(t) \big)^\top  M_{Z}^{(1)}(t) , 
 \quad 1 \leq m \neq n \leq N ; 
% \mathbb{E} \big\{ \big[ Z^{1}_t, \dots, Z^N_t \big]^\top 
%  \mathbf{1}_y \mathbf{Z}_t
% \big\} , 
  \\  
  \mathbf{B}(t) =  \big[ P_1(t+1) \mathbf{e}^y_1  M_Z^{(1)}(t)  , \dots, P_N(t+1) \mathbf{e}^y_N  M_Z^{(1)}(t)  \big]^\top 
  (\boldsymbol{\Theta} + \overline{\boldsymbol{\Theta}}/N ) ; 
  \\ 
   \mathbf{D}(t) =  \big[  \Theta_1^\top  M_Z^{(1)}(t)  , \dots,  \Theta_N^\top  M_Z^{(1)}(t)  \big]^\top  ;  
  % =  \mathbb{E} \mathbf{Z}_t^\top \boldsymbol{\Theta} ,
   \\ 
   \widehat{\mathbf{D}}(t) 
   =  \mathbf{1}_{N\times 1}  \otimes \big[ ( M_Z^{(1)}(t) )^\top \Theta \big] ;
   %= \mathbb{E} \big[ Z^1_t, \dots, Z^N_t \big]^\top \Theta  ,
   \\ 
    \overline{\mathbf{D}}(t) 
    = \mathbf{1}_{N\times 1} \otimes  \big[ (M_Z^{(1)}(t))^\top \, \overline\Theta \, \big] ;   
    %=  \mathbb{E} \big[ Z^1_t, \dots, Z^N_t \big]^\top \overline{\Theta} ,    
   \\ 
   \mathbf{C}(t) =  \big[  S_1(t+1)^\top \mathbf{e}^y_1 M_Z^{(1)}(t) , \dots, S_N(t+1)^\top \mathbf{e}^y_N M_Z^{(1)}(t) \big]^\top ; 
   \\ 
   \mathbf{F}(t) 
   =  \mathbf{1}_{N\times 1} \otimes \big[ ( M_Z^{(1)}(t) )^\top y_{t+1} \big] . 
  % = \mathbb{E} \big[ Z^1_t, \dots, Z^N_t \big]^\top   y_{t+1}  .  
  \end{cases}
  \label{bfABCDF:assum:Zt}
\end{align}

Then the coefficients $\mathbf{Q}_n(t)$ and $\mathbf{L}_n(t)$ can be written as 
\begin{align} 
\mathbf{Q}_n(t; P_n) = & 
 e^{-\alpha(T-1-t)} \Big\{ \kappa \mathbf{e}_n^z M_Z^{(2)}(t) \mathbf{e}_n^{z \top} 
 + \overline\kappa \mathbb{E}\big[ \big( Z^n_t \mathbf{e}_n^{z\top} - \mathbf{1}_y \mathbf{Z}_t/N \big)^\top  \big( Z^n_t \mathbf{e}_n^{z\top} - \mathbf{1}_y \mathbf{Z}_t/N \big) \big] 
 + \gamma \mathbf{e}^z_n \mathbf{e}^{z\top}_n
 \Big\}
\notag \\ 
& 
+ \mathbb{E}( \mathbf{Z}_t^\top P_n(t+1)  \mathbf{Z}_t )  
 , 
  \label{bfQn:assm:Zt} \\
%\end{align} 
%\begin{align} 
\mathbf{L}_n(t; P_n) 
= &
 e^{-\alpha(T-1-t)}  \big[ \kappa \mathbf{e}^z_n ( M_Z^{(1)}(t) )^\top  (\Theta_n - \overline\Theta /N) 
  + \overline\kappa \mathbb{E}\big( Z^n_t \mathbf{e}_n^{z\top} - \mathbf{1}_y \mathbf{Z}_t / N \big)^\top 
  (\Theta_n - \Theta/N)
 \big]  
 \notag \\ 
 & + \mathbb{E}\mathbf{Z}_t^\top P_n(t+1) ( \mathbf{\Theta} + \overline{\mathbf{\Theta}}/N )
\label{bfLn:assm:Zt}
\end{align}

\noindent 
\textbf{Step 1:} 
For each $2\leq i<j \leq N$ fixed,  
we denote $P^\dagger_n := J^{y\top}_{ij} P_n J^y_{ij}$ for $1\leq n \leq N$, 
and show that 
\begin{align} 
( P^\dagger_1, \dots, P^\dagger_{i-1}, P^\dagger_j, P^\dagger_{i+1}, \dots, 
P^\dagger_{j-1}, P^\dagger_i, P^\dagger_{j+1}, \dots, P^\dagger_N ) 
\notag 
\end{align} 
solves the system~\eqref{Pn=Phin(t;P1:N)}, that is, 
\begin{align} 
\label{Pdagger=Phi(Pdagger)}
&
\begin{cases} 
 P_n^\dagger(t) = 
\Phi_n( t; P^\dagger_1, \dots, P^\dagger_{i-1}, P^\dagger_j, P^\dagger_{i+1}, \dots, 
P^\dagger_{j-1}, P^\dagger_i, P^\dagger_{j+1}, \dots, P^\dagger_N ) , 
\quad & \mbox{if} \quad n\neq i, j ;  
 \\ 
%%%%%%% 
 P_i^\dagger(t) = 
\Phi_j( t; P^\dagger_1, \dots, P^\dagger_{i-1}, P^\dagger_j, P^\dagger_{i+1}, \dots, 
P^\dagger_{j-1}, P^\dagger_i, P^\dagger_{j+1}, \dots, P^\dagger_N ) ,  
\quad & \mbox{if} \quad n=i;  \\ 
%%%%%%%  
 P_j^\dagger(t) = 
\Phi_i( t; P^\dagger_1, \dots, P^\dagger_{i-1}, P^\dagger_j, P^\dagger_{i+1}, \dots, 
P^\dagger_{j-1}, P^\dagger_i, P^\dagger_{j+1}, \dots, P^\dagger_N ) . 
\quad & \mbox{if} \quad n = j .   
\end{cases} 
\end{align} 
with the terminal conditiion $P^\dagger_n(T) = 0$, $n = 1, \dots N$. 
Comparing~\eqref{Pn=Phin(t;P1:N)} and~\eqref{Pdagger=Phi(Pdagger)}, we have by uniqueness that  
\begin{align} 
( P^\dagger_1, \dots, P^\dagger_{i-1}, P^\dagger_j, P^\dagger_{i+1}, \dots, 
P^\dagger_{j-1}, P^\dagger_i, P^\dagger_{j+1}, \dots, P^\dagger_N)
= (P_1, \dots, P_N) ,  
\notag 
\end{align} 
which implies that $P_1 = P^\dagger_1$. 
Since $P^\dagger_1=J^{y\top}_{ij} P_1 J^{y}_{ij}$, it then follows that $P_1=J^{y\top}_{ij} P_1 J^{y}_{ij}$ for all $2\leq i < j \leq N$. 
Denote $P_n=\big(P_n^{(k,l)}\big)_{1\leq k,l\leq N}$ with each $P_n^{(k,l)}\in \mathbb{R}^{d_y\times d_y}$. 
Then we have for all $t=0, 1, \dots, T$ that 
\begin{align} 
& P_1^{(1,2)}(t) = P_1^{(1,3)}(t) = \dots = P_1^{(1,N)}(t),  
 \notag \\ 
& P_1^{(2,1)}(t) = P_1^{(3,1)}(t) = \dots = P_1^{(N,1)}(t) , 
\notag \\ 
& P_1^{(2,2)}(t) = P_1^{(3,3)}(t) = \dots = P_1^{(N,N)}(t) ,  
\notag \\ 
& P_1^{(k_1,l_1)}(t) = P_1^{(k_2, l_2)}(t) , 
\quad \forall \, 2\leq k_1 \neq l_1 \leq N, 
\quad \forall \, 2\leq k_2 \neq l_2 \leq N . 
\notag 
\end{align} 
It then follows that $P_1(t)$, $t=0, 1, \dots, T$ has the form~\eqref{P1:PiN1234}.

Now we prove~\eqref{Pdagger=Phi(Pdagger)}.  
By~\eqref{bfABCDF:assum:Zt} we have the relations 
\begin{align} 
& 
 \Theta J^y_{ij}=\Theta, 
\quad 
\overline\Theta J^y_{ij}=\overline\Theta, 
\quad 
\Theta_n J^y_{ij}
=
\begin{cases} 
\Theta_n , \quad \mbox{if} \quad n\neq i,j, \\ 
\Theta_j, \quad \mbox{if} \quad n=i, \\
\Theta_i, \quad \mbox{if} \quad n=j ,  
\end{cases} 
\notag \\ 
& 
J^{y\top}_{ij} \mathbf{\Theta} J^y_{ij} =
J^{y\top}_{ij} \mathbf{\Theta}J^y_{ij}=\mathbf{\Theta}, 
\quad 
J^{y\top}_{ij} \mathbf{\overline\Theta}  = 
\mathbf{\overline\Theta}J^y_{ij}=\mathbf{\overline\Theta}, 
\notag \\ 
& J^{z\top}_{ij}\widehat{\mathbf{A}}_1(t) J^z_{ij} = \widehat{\mathbf{A}}_1(t) , 
\quad  J^{z\top}_{ij}\widehat{\mathbf{A}}_2(t) J^z_{ij} = \widehat{\mathbf{A}}_2(t) , 
\notag \\ 
&  J^{z\top}_{ij} \mathbf{D}(t) J^y_{ij} = \mathbf{D}(t) , 
\quad 
 J^{z\top}_{ij} \widehat{\mathbf{D}}(t) J^y_{ij} = \widehat{\mathbf{D}}(t) ,  
 \quad 
  J^{z\top}_{ij} \overline{\mathbf{D}}(t) J^y_{ij} = \overline{\mathbf{D}}(t) . 
\notag 
\end{align} 
Then we have 
\begin{align} 
& J^{y\top}_{ij} ( \boldsymbol{\Theta} + \overline{\boldsymbol{\Theta}}/N )^\top P_n(t+1) 
( \boldsymbol{\Theta} + \overline{\boldsymbol{\Theta}}/N ) J^y_{ij} 
\notag \\ 
 = & \big[ J^{y\top}_{ij} ( \boldsymbol{\Theta} + \overline{\boldsymbol{\Theta}}/N )^\top J^y_{ij} \big]
 \big( J^{y\top}_{ij} P_n(t+1) J^y_{ij} \big)  
 \big[ J^{y\top}_{ij} ( \boldsymbol{\Theta} + \overline{\boldsymbol{\Theta}}/N ) J^y_{ij} \big]  
\notag \\ 
= &  ( \boldsymbol{\Theta} + \overline{\boldsymbol{\Theta}}/N )^\top 
  P_n^\dagger(t+1)   
 ( \boldsymbol{\Theta} + \overline{\boldsymbol{\Theta}}/N )  ,  
\label{J(Theta+barTheta/N)Pn(Theta+barTheta/N)J}
\end{align} 
and 
\begin{align} 
& e^{-\alpha(T-1-t)} J^{y\top}_{ij} \big[ \kappa  
 \big( \Theta_n + \overline\Theta/N \big)^\top  \big( \Theta_n + \overline\Theta/N \big) 
+ \overline\kappa  
 \big( \Theta_n - \Theta/N \big)^\top \big( \Theta_n - \Theta/N \big)  \big] 
  J^y_{ij} 
 \notag \\ 
 = & 
  e^{-\alpha(T-1-t)} \big[ \kappa  
 J^{y\top}_{ij} \big( \Theta_n + \overline\Theta/N \big)^\top J^{y}_{ij}  
  J^{y\top}_{ij} \big( \Theta_n + \overline\Theta/N \big) J^{y}_{ij} 
+ \overline\kappa  
 J^{y\top}_{ij} \big( \Theta_n - \Theta/N \big)^\top J^{y}_{ij} 
  J^{y\top}_{ij} \big( \Theta_n - \Theta/N \big) J^{y}_{ij}  \big] 
 \notag \\
%%%%%%%%%%%%% 
  = & 
  \begin{cases} 
e^{-\alpha(T-1-t)}  \big[ \kappa  
 \big( \Theta_n + \overline\Theta/N \big)^\top  \big( \Theta_n + \overline\Theta/N \big) 
+ \overline\kappa  
 \big( \Theta_n - \Theta/N \big)^\top \big( \Theta_n - \Theta/N \big)  \big] , 
 \quad \text{if} \quad n \neq i, j ; 
 \\
e^{-\alpha(T-1-t)}  \big[ \kappa  
 \big( \Theta_j + \overline\Theta/N \big)^\top  \big( \Theta_j + \overline\Theta/N \big) 
+ \overline\kappa  
 \big( \Theta_j - \Theta/N \big)^\top \big( \Theta_j - \Theta/N \big)  \big] , 
 \quad \text{if} \quad n = i ; 
 \\ 
 e^{-\alpha(T-1-t)}  \big[ \kappa  
 \big( \Theta_i + \overline\Theta/N \big)^\top  \big( \Theta_i + \overline\Theta/N \big) 
+ \overline\kappa  
 \big( \Theta_i - \Theta/N \big)^\top \big( \Theta_i - \Theta/N \big)  \big] , 
 \quad \text{if} \quad n = j .  
  \end{cases} 
\label{J[Thetan]J}
\end{align} 

By~\eqref{bfG}, we have that   
\allowdisplaybreaks 
\begin{align} 
 & J^{z\top}_{ij} \mathbf{G}(t) J^y_{ij} 
 =
   J^{z\top}_{ij} \mathbf{G}(t; P_1, \dots, P_N) J^y_{ij} 
 \notag \\ 
= & 
- J^{z\top}_{ij} \Big\{ e^{-\alpha(T-1-t)}   
 \big[  \big( \kappa + \overline\kappa (1-1/N) \big) \widehat{\mathbf{A}}_1(t) 
  + \big( \overline\kappa (1-1/N)(-1/N) \big) \widehat{\mathbf{A}}_2(t) 
 + \gamma I \big]  
 + \mathbf{A}(t)  
\Big\}^{-1} \cdot \notag \\ 
& \Big\{ e^{-\alpha(T-1-t)} \big[ (\kappa + \overline{\kappa}(1-1/N) ) \mathbf{D}(t)  
  - \overline{\kappa}(1-1/N)(1/N) \widehat{\mathbf{D}}(t) 
  + (\kappa/N) \overline{\mathbf{D}} 
 \big] + \mathbf{B}(t) \Big\} J^y_{ij} 
 \notag \\
 %%%%%% 
 = & 
- J^{z\top}_{ij} \Big\{ e^{-\alpha(T-1-t)}   
 \big[  \big( \kappa + \overline\kappa (1-1/N) \big) \widehat{\mathbf{A}}_1(t) 
  + \big( \overline\kappa (1-1/N)(-1/N) \big) \widehat{\mathbf{A}}_2(t) 
 + \gamma I \big]  
 + \mathbf{A}(t)  
\Big\}^{-1} J^z_{ij} \cdot \notag \\ 
& J^{z\top}_{ij} \Big\{ e^{-\alpha(T-1-t)} \big[ (\kappa + \overline{\kappa}(1-1/N) ) \mathbf{D}(t)  
  - \overline{\kappa}(1-1/N)(1/N) \widehat{\mathbf{D}}(t) 
  + (\kappa/N) \overline{\mathbf{D}} 
 \big] + \mathbf{B}(t) \Big\} J^y_{ij} 
 \notag \\
 %%%%%%%%%%%%
 = & 
-  \Big\{ e^{-\alpha(T-1-t)}   
 \big[  \big( \kappa + \overline\kappa (1-1/N) \big) J^{z\top}_{ij}\widehat{\mathbf{A}}_1(t) 
 J^z_{ij} 
  + \big( \overline\kappa (1-1/N)(-1/N) \big) J^{z\top}_{ij}
  \widehat{\mathbf{A}}_2(t)
  J^z_{ij} 
 + J^{z\top}_{ij}\gamma I J^z_{ij}  \big]  
 \notag \\ 
&\hspace{12cm} + J^{z\top}_{ij}\mathbf{A}(t)
 J^z_{ij} 
\Big\}^{-1} \cdot \notag \\ 
&  \Big\{ e^{-\alpha(T-1-t)} \big[ (\kappa + \overline{\kappa}(1-1/N) ) J^{z\top}_{ij}
\mathbf{D}(t) 
J^y_{ij} 
  - \overline{\kappa}(1-1/N)(1/N) J^{z\top}_{ij}
  \widehat{\mathbf{D}}(t) 
  J^y_{ij} 
  + (\kappa/N) J^{z\top}_{ij}
  \overline{\mathbf{D}} 
  J^y_{ij} 
 \big] 
 \notag \\ 
& \hspace{12cm} 
+ J^{z\top}_{ij}\mathbf{B}(t) 
 J^y_{ij} \Big\} 
 \notag \\ 
 %%%%%%%%%%%%%%%%% 
= & -  \Big\{ e^{-\alpha(T-1-t)}   
 \big[  \big( \kappa + \overline\kappa (1-1/N) \big) \widehat{\mathbf{A}}_1(t) 
  + \big( \overline\kappa (1-1/N)(-1/N) \big) \widehat{\mathbf{A}}_2(t) 
 + \gamma I \big]  
 + J^{z\top}_{ij} \mathbf{A}(t) 
 J^z_{ij}
\Big\}^{-1} \cdot \notag \\ 
& \Big\{ e^{-\alpha(T-1-t)} \big[ (\kappa + \overline{\kappa}(1-1/N) ) \mathbf{D}(t)  
  - \overline{\kappa}(1-1/N)(1/N) \widehat{\mathbf{D}}(t) 
  + (\kappa/N) \overline{\mathbf{D}} 
 \big] + J^{z\top}_{ij} \mathbf{B}(t)  J^y_{ij} \Big\} , 
 \label{JTbfGJ} 
\end{align}

By~\eqref{JTbfGJ}, \eqref{JTbfA(P)J=bfA(Pdagger)}, and~\eqref{JTbfB(P)J=bfB(Pdagger)}, we then have 
\begin{align}  
 J^{z\top}_{ij} \mathbf{G}(t) J^y_{ij} 
= &  J^{z\top}_{ij} \mathbf{G}( t; P_1, \dots, P_N )  J^y_{ij} 
\notag  \\ 
 = & \mathbf{G}( t; P^\dagger_1, \dots, P^\dagger_{i-1}, P^\dagger_j, P^\dagger_{i+1}, \dots, 
P^\dagger_{j-1}, P^\dagger_i, P^\dagger_{j+1}, \dots, P^\dagger_N) . 
\label{JG(P)J=G(Pdagger)} 
\end{align} 
By~\eqref{bfQn}, we have that  
\begin{align} 
J^{z\top}_{ij} \mathbf{Q}_n(t) J^{z}_{ij}  
 = & J^{z\top}_{ij} \mathbf{Q}_n(t; P_n)  J^{z}_{ij}  
 \notag \\ 
 = & 
 e^{-\alpha(T-1-t)} \Big\{ \kappa J^{z\top}_{ij}  \mathbf{e}_n^z \mathbb{E}(Z_t^{n \top} Z^n_t) \mathbf{e}_n^{z \top} J^{z}_{ij} 
 \notag \\ 
 & + \overline\kappa J^{z\top}_{ij}  \mathbb{E}\big[ \big( Z^n_t \mathbf{e}_n^{z\top} - \mathbf{1}_y \mathbf{Z}_t/N \big)^\top  \big( Z^n_t \mathbf{e}_n^{z\top} - \mathbf{1}_y \mathbf{Z}_t/N \big) \big] 
 J^{z}_{ij} 
 + \gamma J^{z\top}_{ij}  \mathbf{e}^z_n \mathbf{e}^{z\top}_n J^{z}_{ij} 
 \Big\}
\notag \\ 
& +     
   J^{z\top}_{ij} \mathbb{E}( \mathbf{Z}_t^\top P_n(t+1)  \mathbf{Z}_t )  J^{z}_{ij} 
\notag 
\end{align}

If $n\neq i, j$, then by~\eqref{JE(Zn-Z/N)(Zn-Z/N)J} and~\eqref{E(ZPZ)=JE(ZPdaggerZ)J}, 
\begin{align} 
J^{z\top}_{ij} \mathbf{Q}_n(t; P_n) J^{z}_{ij}  
 = & 
 e^{-\alpha(T-1-t)} \Big\{ \kappa   \mathbf{e}_n^z \mathbb{E}(Z_t^{n \top} Z^n_t) \mathbf{e}_n^{z \top}  
 \notag \\ 
 & +   \mathbb{E}\big[ \big( Z^n_t \mathbf{e}_n^{z\top} - \mathbf{1}_y \mathbf{Z}_t/N \big)^\top  \big( Z^n_t \mathbf{e}_n^{z\top} - \mathbf{1}_y \mathbf{Z}_t/N \big) \big]  
 + \gamma   \mathbf{e}^z_n \mathbf{e}^{z\top}_n  
 \Big\}
\notag \\ 
& +     
    \mathbb{E}( \mathbf{Z}_t^\top P^\dagger_n(t+1)  \mathbf{Z}_t )   
 \notag \\ 
 = & \mathbf{Q}_n(t; P_n^\dagger ) ; 
 \label{JijQnJij=Qn} 
\end{align} 
if $n=i$, then 
\begin{align} 
J^{z\top}_{ij} \mathbf{Q}_i(t; P_i ) J^{z}_{ij}  
 = &
    e^{-\alpha(T-1-t)} \Big\{ \kappa   \mathbf{e}_j^z \mathbb{E}(Z_t^{j \top} Z^j_t) \mathbf{e}_j^{z \top}  
 \notag \\ 
 & + \overline\kappa   \mathbb{E}\big[ \big( Z^j_t \mathbf{e}_j^{z\top} - \mathbf{1}_y \mathbf{Z}_t/N \big)^\top 
 \big( Z^j_t \mathbf{e}_j^{z\top} - \mathbf{1}_y \mathbf{Z}_t/N \big) \big]  
 + \gamma   \mathbf{e}^z_j \mathbf{e}^{z\top}_j  
 \Big\}
\notag \\ 
& +     
    \mathbb{E}( \mathbf{Z}_t^\top P^\dagger_i(t+1)  \mathbf{Z}_t )   
 \notag \\ 
 = & \mathbf{Q}_j(t; P_i^\dagger) ; 
 \label{JijQiJij=Qj}
\end{align} 
if $n=j$, then 
\begin{align} 
J^{z\top}_{ij} \mathbf{Q}_j(t; P_j ) J^{z}_{ij}  
 = &
    e^{-\alpha(T-1-t)} \Big\{ \kappa  \mathbf{e}_i^z \mathbb{E}(Z_t^{i \top} Z^i_t) \mathbf{e}_i^{z \top} 
 \notag \\ 
 & + \overline\kappa   \mathbb{E}\big[ \big( Z^i_t \mathbf{e}_i^{z\top} - \mathbf{1}_y \mathbf{Z}_t/N \big)^\top  
 \big( Z^i_t \mathbf{e}_i^{z\top} - \mathbf{1}_y \mathbf{Z}_t/N \big) \big] 
 J^{z}_{ij} 
 + \gamma  \mathbf{e}^z_i \mathbf{e}^{z\top}_i  
 \Big\}
\notag \\ 
& +     
    \mathbb{E}( \mathbf{Z}_t^\top P_j^\dagger(t+1)  \mathbf{Z}_t )  
 \notag \\ 
 = & \mathbf{Q}_i(t; P_j^\dagger) .  
 \label{JijQjJij=Qi}
\end{align} 
By~\eqref{JG(P)J=G(Pdagger)}, \eqref{JijQnJij=Qn}, \eqref{JijQiJij=Qj}, and~\eqref{JijQjJij=Qi}, we have that  
\begin{align} 
& J_{ij}^{y\top} \mathbf{G}^\top(t) \mathbf{Q}_n(t) \mathbf{G}(t) J^y_{ij} 
\notag \\ 
= &   ( J_{ij}^{y\top} \mathbf{G}^\top(t) J^z_{ij} )  
 ( J^{z\top}_{ij} \mathbf{Q}_n(t) J^{z}_{ij} )
 ( J^{z\top}_{ij} \mathbf{G}(t) J^y_{ij}) 
 \notag \\ 
= & \big( J_{ij}^{y\top} \mathbf{G}^\top(t; P_1, \dots, P_N ) J^z_{ij} \big) 
\big( J^{z\top}_{ij} \mathbf{Q}_n(t; P_n) J^z_{ij} \big)
\big( J^{z\top}_{ij} \mathbf{G}(t; P_1, \dots, P_N ) J^y_{ij} \big)  
\notag \\ 
= &  
\begin{cases} 
\mathbf{G}^\top(  t; P^\dagger_1, \dots, P^\dagger_{i-1}, P^\dagger_j, P^\dagger_{i+1}, \dots, 
P^\dagger_{j-1}, P^\dagger_i, P^\dagger_{j+1}, \dots, P^\dagger_N) 
\mathbf{Q}_n(t; P_n^\dagger) \cdot 
\\ 
\hspace{2cm} 
\mathbf{G}(  t; P^\dagger_1, \dots, P^\dagger_{i-1}, P^\dagger_j, P^\dagger_{i+1}, \dots, 
P^\dagger_{j-1}, P^\dagger_i, P^\dagger_{j+1}, \dots, P^\dagger_N) 
, \quad \mbox{if} \quad n \neq i , j ;  \\ 
%%%%% 
\mathbf{G}^\top(  t; P^\dagger_1, \dots, P^\dagger_{i-1}, P^\dagger_j, P^\dagger_{i+1}, \dots, 
P^\dagger_{j-1}, P^\dagger_i, P^\dagger_{j+1}, \dots, P^\dagger_N) 
\mathbf{Q}_j(t; P_i^\dagger) \cdot 
\\ 
\hspace{2cm} 
\mathbf{G}(  t; P^\dagger_1, \dots, P^\dagger_{i-1}, P^\dagger_j, P^\dagger_{i+1}, \dots, 
P^\dagger_{j-1}, P^\dagger_i, P^\dagger_{j+1}, \dots, P^\dagger_N) 
, \quad \mbox{if} \quad n = i ;    \\ 
%%%%% 
\mathbf{G}^\top(  t; P^\dagger_1, \dots, P^\dagger_{i-1}, P^\dagger_j, P^\dagger_{i+1}, \dots, 
P^\dagger_{j-1}, P^\dagger_i, P^\dagger_{j+1}, \dots, P^\dagger_N) 
\mathbf{Q}_i(t; P_j^\dagger) \cdot 
\\ 
\hspace{2cm} 
\mathbf{G}(  t; P^\dagger_1, \dots, P^\dagger_{i-1}, P^\dagger_j, P^\dagger_{i+1}, \dots, 
P^\dagger_{j-1}, P^\dagger_i, P^\dagger_{j+1}, \dots, P^\dagger_N) 
, \quad \mbox{if} \quad n = j .     
\end{cases} 
\label{JGQnGJ:3cases}
\end{align}

By~\eqref{bfLn}, we have that 
\begin{align} 
  J^{z\top}_{ij} \mathbf{L}_n(t) J^y_{ij} 
  = & \mathbf{L}_n(t; P_n )
  \notag \\ 
= &
 e^{-\alpha(T-1-t)}  \big[ \kappa J^{z\top}_{ij} \mathbf{e}^z_n \mathbb{E} Z_t^{n \top} (\Theta_n - \overline\Theta /N) 
  J^z_{ij}
  + \overline\kappa J^{z\top}_{ij} \mathbb{E}\big( Z^n_t \mathbf{e}_n^{z\top} - \mathbf{1}_y \mathbf{Z}_t / N \big)^\top 
  (\Theta_n - \Theta/N) 
   J^z_{ij}
 \big]  
 \notag \\ 
 & + J^{z\top}_{ij} \mathbb{E}\mathbf{Z}_t^\top J^y_{ij} ( J^{y\top}_{ij} P_n(t+1) J^y_{ij} ) 
  J^{y\top}_{ij} ( \mathbf{\Theta} + \overline{\mathbf{\Theta}}/N ) J^z_{ij} . 
  \notag 
\end{align} 
If $n\neq i, j$, then 
\begin{align} 
  J^{z\top}_{ij} \mathbf{L}_n(t; P_n ) J^y_{ij} 
= &
 e^{-\alpha(T-1-t)}  \big[ \kappa  \mathbf{e}^z_n \mathbb{E} Z_t^{n \top} (\Theta_n - \overline\Theta /N) 
  + \overline\kappa  \mathbb{E}\big( Z^n_t \mathbf{e}_n^{z\top} - \mathbf{1}_y \mathbf{Z}_t / N \big)^\top 
  (\Theta_n - \Theta/N)
 \big]  
 \notag \\ 
 & +  \mathbb{E}\mathbf{Z}_t^\top P_n^\dagger(t+1) ( \mathbf{\Theta} + \overline{\mathbf{\Theta}}/N )  
 \notag \\ 
=& \mathbf{L}_n(t; P^\dagger_n ) ; 
  \label{JijLnJij=Ln}
\end{align} 
if $n=i$, then 
\begin{align} 
  J^{z\top}_{ij} \mathbf{L}_i(t; P_i ) J^y_{ij} 
= &
 e^{-\alpha(T-1-t)}  \big[ \kappa  \mathbf{e}^z_j \mathbb{E} Z_t^{j \top} (\Theta_j - \overline\Theta /N) 
  + \overline\kappa  \mathbb{E}\big( Z^j_t \mathbf{e}_j^{z\top} - \mathbf{1}_y \mathbf{Z}_t / N \big)^\top 
  (\Theta_j - \Theta/N)
 \big]  
 \notag \\ 
 & +  \mathbb{E}\mathbf{Z}_t^\top P_i^\dagger(t+1) ( \mathbf{\Theta} + \overline{\mathbf{\Theta}}/N ) 
\notag \\ 
=& \mathbf{L}_j(t; P^\dagger_i ) ; 
  \label{JijLiJij=Lj}
\end{align} 
if $n=j$, then 
\begin{align} 
  J^{z\top}_{ij} \mathbf{L}_j(t; P_j ) J^y_{ij} 
= &
 e^{-\alpha(T-1-t)}  \big[ \kappa  \mathbf{e}^z_i \mathbb{E} Z_t^{i \top} (\Theta_i - \overline\Theta /N) 
  + \overline\kappa  \mathbb{E}\big( Z^i_t \mathbf{e}_i^{z\top} - \mathbf{1}_y \mathbf{Z}_t / N \big)^\top 
  (\Theta_i - \Theta/N)
 \big]  
 \notag \\ 
 & + J^{z\top}_{ij} \mathbb{E}\mathbf{Z}_t^\top P_n(t+1) ( \mathbf{\Theta} + \overline{\mathbf{\Theta}}/N ) J^z_{ij} 
\notag \\ 
=& \mathbf{L}_i(t; P^\dagger_j ) . 
\label{JijLjJij=Li}
\end{align} 
By~\eqref{JG(P)J=G(Pdagger)} and~\eqref{JijLnJij=Ln}, \eqref{JijLiJij=Lj}, and~\eqref{JijLjJij=Li}, we have  
\begin{align} 
 & J^{y\top}_{ij} \mathbf{G}^\top(t; P_1, \dots, P_N ) \mathbf{L}_n(t; P_n ) J^y_{ij} \notag \\ 
  = & \big( J^{y\top}_{ij} \mathbf{G}^\top(t; P_1, \dots, P_N ) J^z_{ij} \big) \big( J^{z\top}_{ij} \mathbf{L}_n(t) J^y_{ij} \big) 
  \notag \\ 
  = & 
  \begin{cases} 
 \mathbf{G}^\top( t; P^\dagger_1, \dots, P^\dagger_{i-1}, P^\dagger_j, P^\dagger_{i+1}, \dots, 
P^\dagger_{j-1}, P^\dagger_i, P^\dagger_{j+1}, \dots, P^\dagger_N) 
\mathbf{L}_n(t) , \quad \mbox{if} \quad n\neq i, j;  \\ 
%%%%% 
\mathbf{G}^\top( t; P^\dagger_1, \dots, P^\dagger_{i-1}, P^\dagger_j, P^\dagger_{i+1}, \dots, 
P^\dagger_{j-1}, P^\dagger_i, P^\dagger_{j+1}, \dots, P^\dagger_N) 
\mathbf{L}_j(t) , \quad \mbox{if} \quad n = i;  \\ 
%%%%% 
\mathbf{G}^\top( t; P^\dagger_1, \dots, P^\dagger_{i-1}, P^\dagger_j, P^\dagger_{i+1}, \dots, 
P^\dagger_{j-1}, P^\dagger_i, P^\dagger_{j+1}, \dots, P^\dagger_N) 
\mathbf{L}_i(t) , \quad \mbox{if} \quad n = j .    
  \end{cases}
\label{JGLJ:3cases}  
\end{align}

Combining~\eqref{Phin(t;P1:N)}, \eqref{J(Theta+barTheta/N)Pn(Theta+barTheta/N)J}, 
\eqref{J[Thetan]J}
\eqref{JGQnGJ:3cases}, and~\eqref{JGLJ:3cases} yields    
\begin{align} 
 &  J^{y\top}_{ij} \Phi_n(t; P_1, \dots, P_N ) J^{y}_{ij} 
 \notag \\ 
 =& 
 \begin{cases} 
\Phi_n( P^\dagger_1, \dots, P^\dagger_{i-1}, P^\dagger_j, P^\dagger_{i+1}, \dots, 
P^\dagger_{j-1}, P^\dagger_i, P^\dagger_{j+1}, \dots, P^\dagger_N ) , \quad \mbox{if} \quad n\neq i, j ;  \\
%%% 
\Phi_j( P^\dagger_1, \dots, P^\dagger_{i-1}, P^\dagger_j, P^\dagger_{i+1}, \dots, 
P^\dagger_{j-1}, P^\dagger_i, P^\dagger_{j+1}, \dots, P^\dagger_N ) , \quad \mbox{if} \quad n = i ; \\ 
%%% 
\Phi_i( P^\dagger_1, \dots, P^\dagger_{i-1}, P^\dagger_j, P^\dagger_{i+1}, \dots, 
P^\dagger_{j-1}, P^\dagger_i, P^\dagger_{j+1}, \dots, P^\dagger_N ) , \quad \mbox{if} \quad n = j . 
 \end{cases} 
 \label{JPhin(t;P1:N)J}
\end{align}

From~\eqref{Pn=Phin(t;P1:N)}, we have 
\begin{align} 
\begin{aligned} 
1\leq n \leq N: \quad 
& P_n^\dagger(t) = J^{y\top}_{ij} P_n(t) J^y_{ij} = J^{y\top}_{ij} \Phi_n(t; P_1, \dots, P_N ) J^y_{ij} , 
\quad 0 \leq t < T , 
 \\ 
& P_n^\dagger(T) = J^{y\top}_{ij} P_n(T) J^y_{ij} = 0 . 
\end{aligned} 
\label{JPnJ=JPhin(t;P1:N)J} 
\end{align}

It then follows from~\eqref{JPhin(t;P1:N)J} and~\eqref{JPnJ=JPhin(t;P1:N)J} that 
\begin{align} 
( P^\dagger_1, \dots, P^\dagger_{i-1}, P^\dagger_j, P^\dagger_{i+1}, \dots, 
P^\dagger_{j-1}, P^\dagger_i, P^\dagger_{j+1}, \dots, P^\dagger_N)  
\notag 
\end{align} 
satisfies~\eqref{Pdagger=Phi(Pdagger)} and the terminal condition $P^\dagger_n(T)=0$, $n=1, \dots, N$.

\bigskip  

\noindent 
\textbf{Step 2:} For each $2\leq j \leq N$, denote $P_n^\ddagger=J^{y\top}_{1j} P_n J^y_{1j}$, $1\leq n \leq N$. Then we have that 
\begin{align} 
\big( P^\ddagger_j, P^\ddagger_2, \dots, P^\ddagger_{j-1}, P^\ddagger_{1}, P^\ddagger_{j+1}, \dots, P^\ddagger_{N}   \big) 
\notag 
\end{align} 
and $(P_1, \dots, P_N)$ satisfy~\eqref{ODE:Pn}. This implies that $P_n = J^{y\top}_{1j} P_1 J^y_{1j}$ for $2\leq j \leq N$, and thus~\eqref{Pn=J1nP1J1n} holds. 
\hfill$\blacksquare$ 
%\end{proof} 

\emph{Proof of Corollary~\ref{cor:inv(hatA1+hatA2+hatA)}.}
%\begin{proof} 
By Proposition~\ref{prop:PnSubmat:PiN1234} and the definition of $\mathbf{A}(t)$ in~\eqref{bfABCDF}, we obtain 
\allowdisplaybreaks
\begin{align} 
  \mathbf{A}(t) = & \mathbb{E} \big\{ \begin{bmatrix} 
  P_1(t+1) \mathbf{e}^y_1 Z^1_t, 
 \dots, 
  P_N(t+1) \mathbf{e}^y_N Z^N_t
 \end{bmatrix}^\top  
 \big[ \mathbf{e}^y_1 Z^1_t, \dots, \mathbf{e}^y_N Z^N_t \big] \big\}  
 \notag \\ 
 = & \mathbb{E}
 \begin{bmatrix} 
Z_t^{1\top} \mathbf{e}_1^{y\top} P_1(t+1) \mathbf{e}^y_1 Z^1_t , 
& Z_t^{1\top} \mathbf{e}_1^{y\top} P_1(t+1) \mathbf{e}^y_2 Z^2_t, 
& \dots & Z_t^{1\top} \mathbf{e}_1^{y\top} P_1(t+1) \mathbf{e}^y_N Z^N_t \\ 
Z_t^{2 \top} \mathbf{e}_2^{y\top} P_2(t+1) \mathbf{e}^y_1 Z^1_t , 
& Z_t^{2 \top} \mathbf{e}_2^{y\top} P_2(t+1) \mathbf{e}^y_2 Z^2_t, 
& \dots & Z_t^{2 \top} \mathbf{e}_2^{y \top} P_2(t+1) \mathbf{e}^y_N Z^N_t \\ 
\vdots & \vdots & \ddots & \vdots \\ 
Z_t^{N \top} \mathbf{e}_N^{y \top} P_N(t+1) \mathbf{e}^y_1 Z^1_t , 
& Z_t^{N \top} \mathbf{e}_N^{y \top} P_N(t+1) \mathbf{e}^y_2 Z^2_t, 
& \dots & Z_t^{N \top} \mathbf{e}_N^{y\top} P_N(t+1) \mathbf{e}^y_N Z^N_t \\ 
 \end{bmatrix} 
 \notag \\ 
 %%%%%%%%%%%%%%%%% 
  = & \mathbb{E}
 \begin{bmatrix} 
Z_t^{1 \top}  \Pi_1^N(t+1) Z^1_t , 
& Z_t^{1 \top} \Pi_2^N(t+1) Z^2_t, 
& \dots & Z_t^{1 \top}  \Pi_2^N(t+1) Z^N_t \\ 
Z_t^{2 \top} \Pi_2^N(t+1) Z^1_t , 
& Z_t^{2 \top} \Pi_1^N(t+1) Z^2_t, 
& \dots & Z_t^{2 \top} \Pi_2^N(t+1) Z^N_t \\ 
\vdots & \vdots & \ddots & \vdots \\ 
Z_t^{N \top} \Pi_2^N(t+1) Z^1_t , 
& Z_t^{N \top} \Pi_2^N(t+1) Z^2_t,  
& \dots & Z_t^{N \top} \Pi_1^N(t+1) Z^N_t \\ 
 \end{bmatrix} , 
 \label{bfA:submat}
\end{align} 
By~\eqref{bfABCDF}, ~\eqref{bfA:submat} and Assumption~\ref{assm:Zt}, we have 
\begin{align} 
&  e^{-\alpha(T-1-t)}   
 \big[  \big( \kappa + \overline\kappa (1-1/N) \big) \widehat{\mathbf{A}}_1(t) 
  + \overline\kappa (1-1/N)(-1/N)  \widehat{\mathbf{A}}_2(t) 
 + \gamma I \big]  
 + \mathbf{A}(t)  
\notag \\ 
= & 
\begin{bmatrix} 
F^N(t) & K^N(t) & \dots & K^N(t) \\ 
K^N(t) & F^N(t) & \dots & K^N(t) \\ 
\vdots & \vdots & \ddots & \vdots \\ 
K^N(t) & K^N(t) & \dots & F^N(t)  
\end{bmatrix} , 
\notag 
\end{align} 
where $F^N(t)$ and $K^N(t)$ are defined in the lemma statement. 
Then the inverse matrix takes the form~\eqref{inv(bfA+bfhatA1+bfhatA2+gammaI):submat}, 
which can be verified by checking that  
\begin{align} 
& F^N(t) M^N(t) +(N-1) K^N(t) E^N(t) = I,\notag \\ 
&  K^N(t) M^N(t) + [ F^N(t) + (N-2)K^N(t)  ] E^N(t) = 0 . 
\notag 
\end{align} 
\hfill$\blacksquare$ 
%\end{proof} 

\emph{Proof of Corollary~\ref{cor:bfGsubmat}.}
Substituting~\eqref{P1:PiN1234} and~\eqref{Pn=J1nP1J1n} into the expression of $\mathbf{G}$ in~\eqref{bfG}, and using Assumption~\eqref{assm:Zt} and the block structure of the inverse matrix given in~\eqref{inv(bfA+bfhatA1+bfhatA2+gammaI):submat}, we obtain that $\mathbf{G}(\cdot)$ admits the submatrix decomposition~\eqref{bfG:submat} with $G_1^N(\cdot)$ and $G_2^N(\cdot)$ given by~\eqref{G1N}-\eqref{G2N}.  
\hfill$\blacksquare$

\emph{Proof of Corollary~\ref{cor:bfQnsubmat}.}
Substituting~\eqref{P1:PiN1234} and~\eqref{Pn=J1nP1J1n} into~\eqref{bfQn}, and using Assumption~\ref{assm:Zt}, we obtain the claimed submatrix decomposition~\eqref{bfQn:submat} of $\mathbf{Q}_n(\cdot)$ for $n\in[N]$. 
\hfill$\blacksquare$

%With the above results, we are ready to prove Corollary~\ref{cor:ODEs:PiN1234}. 
%\begin{proof}
\emph{Proof of Corollary~\ref{cor:ODEs:PiN1234}.} 
By setting $n=1$ in~\eqref{ODE:Pn} and substituting~\eqref{P1:PiN1234}, \eqref{bfG:submat}, and~\eqref{bfQn:submat}, we obtain that the submatrices $\Pi_i^N$, $i=1,2,3,4$, satisfy the system~\eqref{ODE:PiNi}, where the mappings $\varphi_i^N(\cdot)$, $i=1,2,3,4$, are defined below.
\allowdisplaybreaks 
\begin{align} 
 & \varphi_1^N( t; \Pi^N_1, \Pi^N_2, \Pi^N_3 , \Pi^N_4  )    
 \label{varphiN1/4}   \\ 
\eqdef & G_1^{N\top}(t) Q_1^N(t) G_1^N(t) 
+ G_1^{N\top}(t)  Q_2^N(t)  G_2^N(t) (N-1) 
+ G_2^{N\top}(t)   Q_2^{N\top}(t) G_1^N(t) (N-1)
\notag \\ 
& + G_2^{N\top}(t) \big[ Q_3^N(t) + (N-2) Q_4^N(t) \big] G_2^N(t) (N-1)
\notag \\ 
%%%%%% 
& +  e^{-\alpha(T-1-t)}   \kappa G_1^{N\top}(t) ( M^{(1)}_Z(t) )^\top (\theta - \overline\theta/N)   
\notag \\ 
& +  e^{-\alpha(T-1-t)}      
  \overline\kappa \big[ G_1^{N\top}(t) ( M^{(1)}_Z(t) )^\top \theta (1-1/N)^2   
- G_2^{N\top}(t) ( M^{(1)}_Z(t) )^\top \theta(1-1/N)(N-1)/N
  \big]   
\notag \\ 
& + G_1^{N\top}(t) ( M^{(1)}_Z(t) )^\top \big[ \Pi_1^N(t+1)(\theta + \overline\theta/N) 
 + \Pi_2^N(t+1)\overline\theta (N-1)/N \big] 
 \notag \\ 
 & + G_2^{N\top}(t) ( M^{(1)}_Z(t) )^\top \big\{ \Pi_2^{N\top}(t+1)(\theta + \overline\theta/N)
  + \big[ \Pi_3^N(t+1) + \Pi_4^N(t+1) (N-2) \big]\overline\theta/N
 \big\} (N-1) 
\notag \\ 
%%%%%% 
& +  \big\{ e^{-\alpha(T-1-t)}   \kappa G_1^{N\top}(t) ( M^{(1)}_Z(t) )^\top (\theta - \overline\theta/N) 
\big\}^\top 
\notag \\ 
& + \big\{ e^{-\alpha(T-1-t)}      
  \overline\kappa \big[ G_1^{N^\top}(t) ( M^{(1)}_Z(t) )^\top \theta (1-1/N)^2   
- G_2^{N\top}(t) ( M^{(1)}_Z(t) )^\top \theta(1-1/N)(N-1)/N
  \big] 
  \big\}^\top 
\notag \\ 
& + \big\{ G_1^{N\top}(t) ( M^{(1)}_Z(t) )^\top \big[ \Pi_1^N(t+1)(\theta + \overline\theta/N) 
 + \Pi_2^N(t+1)\overline\theta (N-1)/N \big] 
 \big\}^\top 
 \notag \\ 
 & + \Big\{ G_2^{N\top}(t) ( M^{(1)}_Z(t) )^\top \big\{ \Pi_2^{N\top}(t+1)(\theta + \overline\theta/N)
  + \big[ \Pi_3^N(t+1) + \Pi_4^N(t+1) (N-2) \big]\overline\theta/N
 \big\} (N-1) 
 \Big\}^\top 
\notag \\ 
%%%%%%
& + e^{-\alpha(T-1-t)}\big[ \kappa (\theta + \overline\theta/N )^\top (\theta + \overline\theta/N ) 
 + \overline\kappa \theta^\top \theta (1-1/N)^2
\big]
\notag \\ 
%%%%%% 
& + (\theta+\overline\theta/N)^\top \Pi_1^N(t+1) (\theta+\overline\theta/N)  
+ (\theta+\overline\theta/N)^\top \big[ (N-1) \Pi_2^N(t+1) \big] \overline\theta/N 
\notag \\ 
& + (\overline\theta/N)^\top \big[ (N-1) \Pi_2^{N\top}(t+1) \big] (\theta+\overline\theta/N) 
\notag \\ 
& + (\overline\theta/N)^\top \big[ (N-1) \Pi_3^N(t+1) + (N-2)(N-1) \Pi_4^N(t+1) \big] \overline\theta/N . 
\notag 
\end{align}

\allowdisplaybreaks 
\begin{align} 
 & \varphi_2^N( t; \Pi^N_1, \Pi^N_2, \Pi^N_3 , \Pi^N_4  )  
 \label{varphiN2/4} \\ 
\eqdef & G_1^{N\top}(t) Q_2^N(t) G_1^N(t) 
+ G_1^{N\top}(t) \big[ Q_1^N(t) + (N-2)Q_2^N(t) \big] G_2^N(t) 
\notag \\ 
& + G_2^{N\top}(t) \big[ Q_3^{N}(t) + (N-2) Q_4^N(t) \big] G_1^N(t) 
\notag \\ 
& + G_2^{N\top}(t) \big[ (N-1) Q_2^{N\top}(t) + (N-2) Q_3^N(t) + (N-2)^2 Q_4^N(t)   \big] G_2^N(t)
\notag \\ 
%%%%%
& + e^{-\alpha(T-1-t)} \kappa G_1^{N\top}(t) ( M^{(1)}_Z(t) )^\top (-\overline\theta/N) 
\notag \\ 
& + e^{-\alpha(T-1-t)} \overline\kappa \big[ G_1^{N\top}(t) ( M^{(1)}_Z(t) )^\top \theta(1-1/N)(-1/N)
 + G_2^{N\top}(t) ( M^{(1)}_Z(t) )^\top \theta (N-1)/N^2 
\big]
\notag \\ 
& + G_1^{N\top}(t) ( M^{(1)}_Z(t) )^\top \big\{ \Pi_1^N(t+1) \overline\theta/N 
 + \Pi_2^N(t+1) \big[ \theta + \overline\theta(N-1)/N \big] \big\}
\notag \\ 
& + G_2^{N\top}(t) ( M^{(1)}_Z(t) )^\top 
\big[ \Pi_2^{N\top}(t+1) \overline\theta/N + \Pi_3^N(t+1) (\theta + \overline\theta/N) 
 + \Pi_4^N(t+1) \overline\theta (N-2)/N 
\big] 
\notag \\ 
& + 
G_2^{N\top}(t) ( M^{(1)}_Z(t) )^\top 
\big\{ \Pi_2^{N\top}(t+1) \overline\theta/N + \Pi_3^N(t+1) \overline\theta/N 
 + \Pi_4^N(t+1) \big[ \theta + \overline\theta (N-2)/N \big]  \big\} (N-2)
\notag \\ 
%%%%% 
& + \big\{ e^{-\alpha(T-1-t)} \kappa G_2^{N\top}(t) ( M^{(1)}_Z(t) )^\top (\theta - \overline\theta/N ) \big\}^\top 
\notag \\ 
& + \big\{ e^{-\alpha(T-1-t)} \overline\kappa \big[ G_1^{N\top}(t) (M^{(1)}_Z(t))^\top \theta  (1-1/N)(-1/N) 
 + G_2^{N\top}(t) ( M^{(1)}_Z(t) )^\top \theta (1-1/N)(1/N)
\big] 
\big\}^\top 
\notag \\ 
 & + \big\{ G_2^{N\top }(t) ( M^{(1)}_Z(t) )^\top \big[ \Pi_1^N(t+1)(\theta+\overline\theta/N) 
  + \Pi_2^N(t+1) \overline\theta (N-1)/N \big] 
  \big\}^\top 
 \notag \\ 
& + \Big\{  \big[ G_1^{N\top}(t) + (N-2) G_2^{N\top}(t) \big] ( M^{(1)}_Z(t) )^\top 
\cdot 
\notag \\ 
& \hspace{4cm}
\big\{ \Pi_2^{N\top}(t+1) (\theta+\overline\theta/N) 
 + \big[ \Pi_3^N(t+1) + \Pi_4^N(t+1) (N-2) \big] \overline\theta/N  \big\} 
 \Big\}^\top 
\notag \\ 
%%%%% 
& + e^{-\alpha(T-1-t)} \big[ \kappa (\theta + \overline\theta/N)^\top \overline\theta/N 
+ \overline\kappa \theta^\top \theta (1-1/N)(-1/N) \big]
\notag \\ 
%%%%% 
& + (\theta+\overline\theta/N)^\top  \Pi_2^N(t+1) (\theta+\overline\theta/N)
+ (\theta+\overline\theta/N)^\top \big[ \Pi_1^N(t+1) + (N-2)\Pi_2^N(t+1) \big] \overline\theta/N  
\notag \\ 
& + (\overline\theta/N)^\top \big[ \Pi_3^{N}(t+1) + (N-2) \Pi_4^{N}(t+1) \big] (\theta+\overline\theta/N) 
\notag \\ 
& + (\overline\theta/N)^\top  \big[ (N-1) \Pi_2^{N\top}(t+1) + (N-2)\Pi_3^{N}(t+1) + (N-2)^2 \Pi_4^{N}(t+1) \big] 
\overline\theta/N . 
\notag 
\end{align}

\allowdisplaybreaks
\begin{align} 
 & \varphi_3^N( t; \Pi^N_1, \Pi^N_2, \Pi^N_3 , \Pi^N_4  )   
 \label{varphiN3/4} \\ 
\eqdef & G_1^{N\top}(t) Q_3^N(t) G_1^N(t) 
+ G_1^{N\top}(t) \big[ Q_2^{N\top}(t) + (N-2) Q_4^{N}(t) \big] G_2^N(t)
\notag \\ 
& + G_2^{N\top}(t) \big[ Q_2^{N}(t) + (N-2) Q_4^{N}(t) \big] G_1^N(t) 
\notag \\ 
& + G_2^{N\top}(t) \big[ Q_1^N(t) + (Q_2^{N\top}(t) + Q_2^N(t)  )(N-2) 
+ Q_3^N(t)(N-2) + (N-3)(N-2)Q_4^N(t)
\big] G_2^N(t) 
\notag \\ 
%%%%%%%
& + e^{-\alpha(T-1-t)} \kappa G_2^{N\top}(t) ( M^{(1)}_Z(t) )^\top  \overline\theta(-1/N ) 
\notag \\ 
& + e^{-\alpha(T-1-t)} \overline\kappa 
 \big[ G_1^{N\top}(t) ( M^{(1)}_Z(t) )^\top \theta (1/N^2) 
  + G_2^{N\top}(t) ( M^{(1)}_Z(t) )^\top \theta (-1/N^2) \big] 
  \notag \\ 
& + G_2^{N\top}(t) ( M^{(1)}_Z(t) )^\top 
 \big\{ \Pi_1^N(t+1) \overline\theta/N + \Pi_2^N(t+1) \big[ \theta + \overline\theta (N-1)/N \big]  \big\} 
\notag \\ 
& 
+ G_1^{N\top}(t) ( M^{(1)}_Z(t) )^\top  
\big[ \Pi_2^{N\top}(t+1) \overline\theta/N + \Pi_3^{N}(t+1)(\theta + \overline\theta/N) 
 + \Pi_4^N(t+1)\overline\theta(N-2)/N \big]
\notag \\ 
& 
+ G_2^{N\top}(t) ( M^{(1)}_Z(t) )^\top 
\big\{ \Pi_2^{N\top}(t+1) \overline\theta/N + \Pi_3^N(t+1) \overline\theta/N 
 + \Pi_4^{N\top}(t+1) \big[ \theta +  \overline\theta (N-2)/N \big]
\big\} (N-2)
\notag \\ 
%%%%%%%
& + \big\{ e^{-\alpha(T-1-t)} \kappa G_2^{N\top}(t) ( M^{(1)}_Z(t) )^\top  \overline\theta(-1/N ) \big\}^\top
\notag \\ 
& + \big\{ e^{-\alpha(T-1-t)} \overline\kappa 
 \big[ G_1^{N\top}(t) ( M^{(1)}_Z(t) )^\top \theta (1/N^2) 
  + G_2^{N\top}(t) ( M^{(1)}_Z(t) )^\top \theta (-1/N^2) \big] 
  \big\}^\top 
  \notag \\ 
& + \Big\{ G_2^{N\top}(t) ( M^{(1)}_Z(t) )^\top 
 \big\{ \Pi_1^N(t+1) \overline\theta/N + \Pi_2^N(t+1) \big[ \theta + \overline\theta (N-1)/N \big]  \big\} 
\Big\}^\top 
\notag \\ 
& 
+ \big\{ G_1^{N\top}(t) ( M^{(1)}_Z(t) )^\top  
\big[ \Pi_2^{N\top}(t+1) \overline\theta/N + \Pi_3^{N}(t+1)(\theta + \overline\theta/N) 
 + \Pi_4^N(t+1)\overline\theta(N-2)/N \big] 
 \big\}^\top 
\notag \\ 
& 
+ \big\{ G_2^{N\top}(t) ( M^{(1)}_Z(t) )^\top 
\big\{ \Pi_2^{N\top}(t+1) \overline\theta/N + \Pi_3^N(t+1) \overline\theta/N 
 + \Pi_4^{N\top}(t+1) \big[ \theta +  \overline\theta (N-2)/N \big]
\big\} (N-2) 
\big\}^\top 
\notag \\ 
%%%%% 
& 
 + e^{-\alpha(T-1-t)} \big[ \kappa \overline\theta^\top \overline\theta (1/N^2) 
  + \overline\kappa \theta^\top \theta (1/N^2) \big]
\notag \\ 
%%%%%%
& + (\theta + \overline\theta/N)^\top \Pi_3^N(t+1) (\theta + \overline\theta/N) 
+ (\theta + \overline\theta/N)^\top \big[ \Pi_2^{N\top}(t+1) + (N-2) \Pi_4^{N}(t+1) \big] 
\overline\theta/N
\notag \\ 
& + (\overline\theta/N)^\top \big[ \Pi_2^{N}(t+1) + (N-2) \Pi_4^{N}(t+1) \big] (\theta + \overline\theta/N)  
\notag \\ 
& + (\overline\theta/N)^\top \big[ \Pi_1^N(t+1) + (\Pi_2^{N\top}(t+1) + \Pi_2^N(t+1) + \Pi_3^N(t+1) )(N-2) 
\notag \\ 
& 
\hspace{6cm}
+ \Pi_4^N(t+1)(N-2)(N-3)
\big] \overline\theta/N  . 
\notag 
\end{align}

\allowdisplaybreaks
\begin{align} 
 & \varphi_4^N( t; \Pi^N_1, \Pi^N_2, \Pi^N_3 , \Pi^N_4  )  
 \label{varphiN4/4}  \\ 
\eqdef & G_1^{N\top}(t) Q_4^N(t) G_1^N(t) 
+ G_1^{N\top}(t) \big[ Q_2^{N\top}(t) + Q_3^N(t) + (N-3) Q_4^N(t) \big] G_2^N(t) 
\notag \\ 
& + G_2^{N\top}(t) \big[ Q_2^N(t)  +  Q_3^{N}(t) + (N-3) Q_4^N(t) \big] G_1^{N}(t)
\notag \\ 
& + G_2^{N\top}(t) \big\{ Q_1^N(t) + ( Q_2^{N}(t) + Q_2^{N\top}(t)  )(N-2) 
  + Q_3^N(t) (N-3) + Q_4^N(t) [ (N-3)^2 + N-2 ]
\big\} G_2^N(t) 
\notag \\ 
%%%%%%% 
&  + e^{-\alpha(T-1-t)} \kappa G_2^{N\top}(t) ( M^{(1)}_Z(t) )^\top \overline\theta (-1/N)
\notag \\ 
& + e^{-\alpha(T-1-t)} \overline\kappa \big[ G_1^{N\top}(t) (1/N^2) +  G_2^{N\top}(t) (-1/N^2) \big] 
( M^{(1)}_Z(t) )^\top \theta 
\notag \\ 
& + G_2^{N\top}(t) ( M^{(1)}_Z(t) )^\top  \big\{ \Pi_1^{N }(t+1) \overline\theta /N + \Pi_2^N(t+1) \big[ \theta + \overline\theta (N-1)/N \big] 
 + \Pi_2^{N\top}(t+1) \theta/N 
 \notag \\ 
&
\hspace{3cm}
+ \Pi_3^{N}(t+1) (\theta + \overline\theta/N) + \Pi_4^N(t+1) \overline\theta (N-2)/N 
\big\}
\notag \\ 
& + \big[ G_1^{N\top}(t) + (N-3) G_2^{N\top}(t) \big] ( M^{(1)}_Z(t) )^\top 
\cdot 
\notag \\ 
& 
\hspace{3cm}
\big\{ \Pi_2^{N\top}(t+1) \overline\theta/N + \Pi_3^N(t+1) \overline\theta/N 
 + \Pi_4^{N}(t+1) \big[ \theta + \overline\theta (N-2)/N \big]  \big\}
\notag \\ 
%%%%%%% 
&  + \big\{ e^{-\alpha(T-1-t)} \kappa G_2^{N\top}(t) ( M^{(1)}_Z(t) )^\top \overline\theta (-1/N) 
\big\}^\top 
\notag \\ 
& + \big\{ e^{-\alpha(T-1-t)} \overline\kappa \big[ G_1^{N\top}(t) ( M^{(1)}_Z(t) )^\top \theta (1/N^2) 
 + G_2^{N\top}(t) ( M^{(1)}_Z(t) )^\top \theta (-1/N^2) \big] \big\}^\top 
\notag \\ 
& 
+ \Big\{ G_2^{N\top}(t) ( M^{(1)}_Z(t) )^\top \big\{ \Pi_1^N(t+1) \overline\theta/N + \Pi_2^N(t+1) \big[ \theta + \overline\theta(N-1)/N \big] 
+ \Pi_2^{N\top}(t+1) \overline\theta/N 
\notag \\ 
& 
\hspace{3cm}
+ \Pi_3^N(t+1) (\theta + \overline\theta/N ) 
 + \Pi_4^N(t+1) \overline\theta (N-2)/N
\big\} 
\Big\}^\top 
\notag \\ 
& 
+ \Big\{ \big[ G_1^{N\top}(t) + (N-3) G_2^{N\top}(t) \big] ( M^{(1)}_Z(t) )^\top 
\cdot
\notag \\ 
& 
\hspace{4cm}
\big\{ \Pi_2^{N\top}(t+1) \overline\theta/N + \Pi_3^N(t+1) \overline\theta/N 
 + \Pi_4^N(t+1)\big[ \theta + \overline\theta(N-2)/N \big] \big\} 
 \Big\}^\top 
\notag \\
%%%%%%% 
& 
 + e^{-\alpha(T-1-t)} 
 \big[ \kappa \overline\theta^\top \overline\theta (1/N^2) + \overline\kappa \theta^\top \theta (1/N^2) \big]
\notag \\ 
%%%%%%%%
& + (\theta + \overline\theta/N )^\top  \Pi_4^N(t+1) (\theta + \overline\theta/N ) 
+ (\theta + \overline\theta/N )^\top \big[ \Pi_2^{N\top}(t+1) + \Pi_3^{N}(t+1) + (N-3) \Pi_4^N(t+1) \big] 
 \overline\theta/N 
\notag \\ 
& + (\overline\theta/N)^\top \big[ \Pi_2^N(t+1) + \Pi_3^N(t+1) + (N-3) \Pi_4^{N}(t+1) \big] 
(\theta + \overline\theta/N )
\notag \\ 
& + (\overline\theta/N)^\top \big[ \Pi_1^N(t+1) + ( \Pi_2^{N}(t+1) + \Pi_2^{N\top}(t+1) + \Pi_4^N(t+1) )(N-2) 
\notag \\ 
& \hspace{5cm}  + \Pi_3^N(t+1)(N-3) + \Pi_4^N(t+1)(N-3)^2 
\big] \overline\theta/N . 
\notag 
\end{align} 
\hfill$\blacksquare$ 

\emph{Proof of Proposition~\ref{prop:SnSubmat}.}
%\begin{proof} 
The proof follows the analogous arguments as in Proposition~\ref{prop:PnSubmat:PiN1234}. 

\begin{comment} 
For $n=1, \dots, N$, the equation system~\eqref{ODE:Sn} can be written as 
\begin{align} 
 & S_n(t) =  \Psi_n(t; S_1, \dots, S_N; P_1, \dots, P_N )  , \quad 0 \leq t \leq T-1 , 
 \notag \\ 
 & S_n(T) =  0 ,  
 \notag 
\end{align} 
where 
\begin{align} 
 & \Psi_n(t; S_1, \dots, S_N; P_1, \dots, P_N ) \notag \\ 
 \eqdef  
 &  \mathbf{G}^\top(t; P_1, \dots, P_N ) \mathbf{Q}_n(t; P_n) \mathbf{H}(t; S_1, \dots, S_N; P_1, \dots, P_N) 
 \notag \\ 
 & + \mathbf{G}^\top(t; P_1, \dots, P_N ) \big[ - e^{-\alpha(T-1-t)} \kappa \mathbf{e}^z_n \mathbb{E}Z_t^{n\top} y_{t+1}  
  + \mathbb{E} \mathbf{Z}_t^\top S_n(t+1)
 \big] 
  \notag \\ 
 & +   \big\{ e^{-\alpha(T-1-t)} \big[ \kappa ( \Theta_n + \overline\Theta/N )^\top \mathbb{E} Z^n_t \mathbf{e}^{z\top}_n 
 + \overline\kappa (\Theta_n - \Theta/N )^\top 
  \mathbb{E}  \big( Z^n_t \mathbf{e}_n^{z\top} - \mathbf{1}_y \mathbf{Z}_t /N \big)  
 \big] 
  \notag \\ 
 & \qquad  + ( \mathbf{\Theta} + \overline{\mathbf{\Theta}}/N )^\top P_n(t+1) \mathbb{E}\mathbf{Z}_t 
 \big\} \mathbf{H}(t; S_1, \dots, S_N ; P_1, \dots, P_N ) 
  \notag \\ 
 &  - e^{-\alpha(T-1-t)} \kappa (\Theta_n + \overline\Theta/N )^\top y_{t+1} 
  +   (\mathbf{\Theta} + \overline{\Theta}/N )^\top S_n(t+1) 
  \notag 
\end{align} 
\end{comment}

\noindent 
\textbf{Step 1:} 
We denote $S_n^\dagger = J^{y\top}_{ij}S_n$ for $1\leq i < j \leq N$.  
We aim to show that for all $2\leq i < j \leq N$, 
\begin{align} 
 ( S^\dagger_1, \dots, S^\dagger_{i-1}, S^\dagger_j, S^\dagger_{i+1}, \dots, 
S^\dagger_{j-1}, S^\dagger_i, S^\dagger_{j+1}, \dots, S^\dagger_N )
\notag 
\end{align} 
solves the system~\eqref{ODE:Sn}, that is, 
\begin{align} 
\label{Sdagger=Psi(Sdagger)}
\begin{cases} 
S_n^\dagger(t) = \Psi_n(t; S^\dagger_1, \dots, S^\dagger_{i-1}, S^\dagger_j, S^\dagger_{i+1}, \dots, 
S^\dagger_{j-1}, S^\dagger_i, S^\dagger_{j+1}, \dots, S^\dagger_N ; 
 \\ 
\hspace{3cm}
P^\dagger_1, \dots, P^\dagger_{i-1}, P^\dagger_j, P^\dagger_{i+1}, \dots, 
P^\dagger_{j-1}, P^\dagger_i, P^\dagger_{j+1}, \dots, P^\dagger_N
), \quad \mbox{if} \quad n\neq i, j; \\ 
%%%%% 
S_i^\dagger(t) = \Psi_j(t; S^\dagger_1, \dots, S^\dagger_{i-1}, S^\dagger_j, S^\dagger_{i+1}, \dots, 
S^\dagger_{j-1}, S^\dagger_i, S^\dagger_{j+1}, \dots, S^\dagger_N ; 
 \\ 
\hspace{3cm}
P^\dagger_1, \dots, P^\dagger_{i-1}, P^\dagger_j, P^\dagger_{i+1}, \dots, 
P^\dagger_{j-1}, P^\dagger_i, P^\dagger_{j+1}, \dots, P^\dagger_N
), \quad \mbox{if} \quad n = i; \\ 
%%%%% 
S_j^\dagger(t) = \Psi_i(t; S^\dagger_1, \dots, S^\dagger_{i-1}, S^\dagger_j, S^\dagger_{i+1}, \dots, 
S^\dagger_{j-1}, S^\dagger_i, S^\dagger_{j+1}, \dots, S^\dagger_N ; 
 \\ 
\hspace{3cm}
P^\dagger_1, \dots, P^\dagger_{i-1}, P^\dagger_j, P^\dagger_{i+1}, \dots, 
P^\dagger_{j-1}, P^\dagger_i, P^\dagger_{j+1}, \dots, P^\dagger_N
), \quad \mbox{if} \quad n=j. \\ 
\end{cases} 
\end{align} 
Comparing~\eqref{ODE:Sn} and~\eqref{Sdagger=Psi(Sdagger)}, we have by uniqueness that 
\begin{align} 
( S^\dagger_1, \dots, S^\dagger_{i-1}, S^\dagger_j, S^\dagger_{i+1}, \dots, 
S^\dagger_{j-1}, S^\dagger_i, S^\dagger_{j+1}, \dots, S^\dagger_N)
= (S_1, \dots, S_N) .   
\notag 
\end{align} 
It then follows that $S_1 = S_1^\dagger = J^{y\top}_{ij} S_1$ for all $2\leq i < j \leq N$, which proves 
\eqref{S1:submat}.

Now we prove~\eqref{Sdagger=Psi(Sdagger)}. 

\begin{align} 
 & J^{y\top}_{ij} S_n(t) =  J^{y\top}_{ij} \Psi_n(t; S_1, \dots, S_N; P_1, \dots, P_N )  , \quad 0 \leq t \leq T-1 , 
 \notag \\ 
 & J^{y\top}_{ij} S_n(T) =  0 ,  
 \notag 
\end{align} 
where 
\begin{align} 
 & J^{y\top}_{ij} \Psi_n(t; S_1, \dots, S_N; P_1, \dots, P_N ) \notag \\ 
 =  & J^{y\top}_{ij} \mathbf{G}^\top(t; P_1, \dots, P_N ) \mathbf{Q}_n(t; P_n) \mathbf{H}(t; S_1, \dots, S_N; P_1, \dots, P_N) 
 \notag \\ 
 & + J^{y\top}_{ij} \mathbf{G}^\top(t; P_1, \dots, P_N ) \big[ - e^{-\alpha(T-1-t)} \kappa \mathbf{e}^z_n \mathbb{E}Z_t^{n\top} y_{t+1}  
  + \mathbb{E} \mathbf{Z}_t^\top S_n(t+1)
 \big] 
  \notag \\ 
 & +  J^{y\top}_{ij} \big\{ e^{-\alpha(T-1-t)} \big[ \kappa ( \Theta_n + \overline\Theta/N )^\top \mathbb{E} Z^n_t \mathbf{e}^{z\top}_n 
 + \overline\kappa (\Theta_n - \Theta/N )^\top 
  \mathbb{E}  \big( Z^n_t \mathbf{e}_n^{z\top} - \mathbf{1}_y \mathbf{Z}_t /N \big)  
 \big] 
  \notag \\ 
 & \qquad \qquad   + ( \mathbf{\Theta} + \overline{\mathbf{\Theta}}/N )^\top P_n(t+1) \mathbb{E}\mathbf{Z}_t 
 \big\} \mathbf{H}(t; S_1, \dots, S_N; P_1, \dots, P_N ) 
  \notag \\ 
 &  - e^{-\alpha(T-1-t)} \kappa (\Theta_n + \overline\Theta/N )^\top y_{t+1} 
  +   (\mathbf{\Theta} + \overline{\Theta}/N )^\top S_n(t+1)  . 
  \notag 
\end{align}

By~\eqref{bfH}, we have 
\allowdisplaybreaks
\begin{align} 
 & J^{z\top}_{ij} \mathbf{H}(t; P_1, \dots, P_N; S_1, \dots, S_N ) 
 \notag \\ 
= & - J^{z\top}_{ij} \Big\{ e^{-\alpha(T-1-t)}   
 \big[  \big( \kappa + \overline\kappa (1-1/N) \big) \widehat{\mathbf{A}}_1(t) 
  + \big( \overline\kappa (1-1/N)(-1/N) \big) \widehat{\mathbf{A}}_2(t) 
 +  \gamma I \big]  
 + \mathbf{A}(t)  
\Big\}^{-1} \cdot  \notag \\ 
 &\Big[  e^{-\alpha(T-1-t)} (-\kappa) \mathbf{F}(t) 
 + \mathbf{C}(t)  \Big] 
 \notag \\ 
 = & 
 - J^{z\top}_{ij} \Big\{ e^{-\alpha(T-1-t)}   
 \big[  \big( \kappa + \overline\kappa (1-1/N) \big) \widehat{\mathbf{A}}_1(t) 
  + \big( \overline\kappa (1-1/N)(-1/N) \big) \widehat{\mathbf{A}}_2(t) 
 +  \gamma I \big]  
 + \mathbf{A}(t)  
\Big\}^{-1} J^z_{ij} \cdot  \notag \\ 
 & J^{z\top}_{ij} \Big[  e^{-\alpha(T-1-t)} (-\kappa) \mathbf{F}(t) 
 + \mathbf{C}(t)  \Big] 
 \notag \\ 
 %%%%% 
  = & 
 -  \Big\{ e^{-\alpha(T-1-t)}   
 \big[  \big( \kappa + \overline\kappa (1-1/N) \big) J^{z\top}_{ij} \widehat{\mathbf{A}}_1(t) J^z_{ij} 
  + \big( \overline\kappa (1-1/N)(-1/N) \big) J^{z\top}_{ij} \widehat{\mathbf{A}}_2(t) J^z_{ij} 
 + J^{z\top}_{ij} \gamma I J^z_{ij} \big] 
 \notag \\ 
& \hspace{13cm} + J^{z\top}_{ij} \mathbf{A}(t) J^z_{ij} 
\Big\}^{-1}  \cdot  \notag \\ 
 &  \Big[  e^{-\alpha(T-1-t)} (-\kappa) J^{z\top}_{ij} \mathbf{F}(t) 
 + J^{z\top}_{ij} \mathbf{C}(t)  \Big] 
 \notag \\ 
 %%%%% 
  = & 
 -  \Big\{ e^{-\alpha(T-1-t)}   
 \big[  \big( \kappa + \overline\kappa (1-1/N) \big)  \widehat{\mathbf{A}}_1(t) 
  + \big( \overline\kappa (1-1/N)(-1/N) \big)  \widehat{\mathbf{A}}_2(t)  
 +  \gamma I  \big] 
 + J^{z\top}_{ij} \mathbf{A}(t) J^z_{ij} 
\Big\}^{-1}  \cdot  \notag \\ 
 &  \Big[  e^{-\alpha(T-1-t)} (-\kappa)  \mathbf{F}(t) 
 + J^{z\top}_{ij} \mathbf{C}(t)  \Big] ,  
 \notag 
\end{align} 
where 
\begin{align} 
 J^{z\top}_{ij} \mathbf{C}(t) 
 =&J^{z\top}_{ij} \mathbf{C}(t; S_1, \dots, S_N) 
 \notag \\ 
 =& J^{z\top}_{ij} \mathbb{E} \big[  S_1(t+1)^\top \mathbf{e}^y_1 Z^1_t , 
  \dots, S_N(t+1)^\top \mathbf{e}^y_N Z^N_t \big]^\top 
 \notag \\ 
 =& J^{z\top}_{ij} \mathbb{E} \big[  S_1(t+1)^\top J^y_{ij} J^{y\top}_{ij} \mathbf{e}^y_1 Z^1_t , \dots, S_N(t+1)^\top J^y_{ij} J^{y\top}_{ij} \mathbf{e}^y_N Z^N_t \big]^\top  
 \notag \\ 
 %%%%% 
 =& J^{z\top}_{ij} \mathbb{E} \big[  S_1^\dagger(t+1)^\top J^{y\top}_{ij} \mathbf{e}^y_1 Z^1_t , \dots, S_N^\dagger(t+1)^\top  J^{y\top}_{ij} \mathbf{e}^y_N Z^N_t \big]^\top  
 \notag \\ 
 %%%%% 
 =& J^{z\top}_{ij} \mathbb{E} \big[  S_1^\dagger(t+1)^\top  \mathbf{e}^y_1 Z^1_t , \dots, 
  S_{i-1}^\dagger(t+1)^\top  \mathbf{e}^y_{i-1} Z^{i-1}_t,   S_i^\dagger(t+1)^\top  \mathbf{e}^y_j Z^i_t,  
   S_{i+1}^\dagger(t+1)^\top  \mathbf{e}^y_{i+1} Z^{i+1}_t , \dots 
   \notag \\ 
   & \quad \dots, 
   S_{j-1}^\dagger(t+1)^\top  \mathbf{e}^y_{j-1} Z^{j-1}_t,   S_j^\dagger(t+1)^\top  \mathbf{e}^y_i Z^j_t,  
   S_{j+1}^\dagger(t+1)^\top  \mathbf{e}^y_{j+1} Z^{j+1}_t , \dots , 
  S_N^\dagger(t+1)^\top  J^{y\top}_{ij} \mathbf{e}^y_N Z^N_t \big]^\top  
 \notag \\ 
 %%%%%
 =&  \mathbb{E} \big[  S_1^\dagger(t+1)^\top \mathbf{e}^y_1 Z^1_t , \dots, S_{i-1}^\dagger(t+1)^\top \mathbf{e}^y_{i-1} Z^{i-1}_t , 
 S_j^\dagger(t+1)^\top \mathbf{e}^y_i Z^j_t , 
 S_{i+1}^\dagger(t+1)^\top \mathbf{e}^y_{i+1} Z^{i+1}_t, 
 \dots, 
 \notag \\ 
& \hspace{1cm}  \dots, S_{j-1}^\dagger(t+1)^\top \mathbf{e}^y_{j-1} Z^{j-1}_t, 
 S_i^\dagger(t+1)^\top \mathbf{e}^y_j Z^i_t, 
 S_{j+1}^\dagger(t+1)^\top \mathbf{e}^y_{j+1} Z^{j+1}_t , 
 \dots , 
 S_N^\dagger(t+1)^\top \mathbf{e}^y_N Z^N_t \big]^\top  
 \notag \\ 
 %%%%% 
 = & \mathbf{C}(t; S^\dagger_1, \dots, S^\dagger_{i-1}, S^\dagger_j, S^\dagger_{i+1}, \dots, S^\dagger_{j-1}, S^\dagger_i, 
 S^\dagger_{j+1}, \dots, S^\dagger_N ) . 
 \notag 
\end{align} 
Then we have that 
\begin{align} 
  J^{z\top}_{ij} \mathbf{H}(t) 
 = & J^{z\top}_{ij} \mathbf{H}(t; S_1, \dots, S_N; P_1, \dots, P_N ) 
 \notag \\ 
 = & \mathbf{H} (t; S_1^\dagger, \dots, S_{i-1}^\dagger, S_j^\dagger, S_{i+1}^\dagger, \dots, S_{j-1}^\dagger, S_i^\dagger, S_{j+1}^\dagger, \dots, S_N^\dagger; 
 \notag \\ 
& \hspace{1cm} P_1^\dagger, \dots, P_{i-1}^\dagger, P_j^\dagger, P_{i+1}^\dagger, \dots, P_{j-1}^\dagger, P_i^\dagger, P_{j+1}^\dagger, \dots, P_N^\dagger ) . 
 \label{JH(SP)=H(SPdagger)}
\end{align} 

By~\eqref{JG(P)J=G(Pdagger)}, \eqref{JijQnJij=Qn}, \eqref{JijQiJij=Qj}, \eqref{JijQjJij=Qi}, and~\eqref{JH(SP)=H(SPdagger)}, we have that 
\begin{align} 
& J_{ij}^{y\top } \mathbf{G}^\top(t; P_1, \dots, P_N ) \mathbf{Q}_n(t; P_n) 
  \mathbf{H}(t; P_1, \dots, P_N; S_1, \dots, S_N ) 
  \notag \\ 
= & \big[ J_{ij}^{y\top } \mathbf{G}^\top(t; P_1, \dots, P_N ) J^{z}_{ij} \big] 
\big[ J^{z\top}_{ij} \mathbf{Q}_n(t; P_n) J^z_{ij} \big]  
 \big[ J^{z\top}_{ij} \mathbf{H}(t; P_1, \dots, P_N; S_1, \dots, S_N ) \big] 
 \notag \\ 
 = & 
 \begin{cases} 
\mathbf{G}^\top(  t; P^\dagger_1, \dots, P^\dagger_{i-1}, P^\dagger_j, P^\dagger_{i+1}, \dots, 
P^\dagger_{j-1}, P^\dagger_i, P^\dagger_{j+1}, \dots, P^\dagger_N) 
\mathbf{Q}_n(t; P_n^\dagger) \cdot 
\\ 
\hspace{2cm} 
\mathbf{H} (t; S_1^\dagger, \dots, S_{i-1}^\dagger, S_j^\dagger, S_{i+1}^\dagger, \dots, S_{j-1}^\dagger, S_i^\dagger, S_{j+1}^\dagger, \dots, S_N^\dagger; 
  \\ 
 \hspace{3cm} P_1^\dagger, \dots, P_{i-1}^\dagger, P_j^\dagger, P_{i+1}^\dagger, \dots, P_{j-1}^\dagger, P_i^\dagger, P_{j+1}^\dagger, \dots, P_N^\dagger ) 
, \quad \mbox{if} \quad n \neq i , j ;  \\ 
%%%%% 
\mathbf{G}^\top(  t; P^\dagger_1, \dots, P^\dagger_{i-1}, P^\dagger_j, P^\dagger_{i+1}, \dots, 
P^\dagger_{j-1}, P^\dagger_i, P^\dagger_{j+1}, \dots, P^\dagger_N) 
\mathbf{Q}_j(t; P_i^\dagger) \cdot 
\\ 
\hspace{2cm} 
\mathbf{H} (t; S_1^\dagger, \dots, S_{i-1}^\dagger, S_j^\dagger, S_{i+1}^\dagger, \dots, S_{j-1}^\dagger, S_i^\dagger, S_{j+1}^\dagger, \dots, S_N^\dagger; 
  \\ 
 \hspace{3cm} P_1^\dagger, \dots, P_{i-1}^\dagger, P_j^\dagger, P_{i+1}^\dagger, \dots, P_{j-1}^\dagger, P_i^\dagger, P_{j+1}^\dagger, \dots, P_N^\dagger ) 
, \quad \mbox{if} \quad n = i ;    \\ 
%%%%% 
\mathbf{G}^\top(  t; P^\dagger_1, \dots, P^\dagger_{i-1}, P^\dagger_j, P^\dagger_{i+1}, \dots, 
P^\dagger_{j-1}, P^\dagger_i, P^\dagger_{j+1}, \dots, P^\dagger_N) 
\mathbf{Q}_i(t; P_j^\dagger) \cdot 
\\ 
\hspace{2cm} 
\mathbf{H} (t; S_1^\dagger, \dots, S_{i-1}^\dagger, S_j^\dagger, S_{i+1}^\dagger, \dots, S_{j-1}^\dagger, S_i^\dagger, S_{j+1}^\dagger, \dots, S_N^\dagger; 
  \\ 
 \hspace{3cm} P_1^\dagger, \dots, P_{i-1}^\dagger, P_j^\dagger, P_{i+1}^\dagger, \dots, P_{j-1}^\dagger, P_i^\dagger, P_{j+1}^\dagger, \dots, P_N^\dagger ) 
, \quad \mbox{if} \quad n = j .     
\end{cases} 
 \label{JGQnH:3cases}  
\end{align}

\begin{align} 
 & J^{y\top}_{ij} \mathbf{G}^\top(t; P_1, \dots, P_N )   \big[ - e^{-\alpha(T-1-t)} \kappa \mathbf{e}^z_n \mathbb{E}Z_t^{n \top} y_{t+1}  
  + \mathbb{E} \mathbf{Z}_t^\top S_n(t+1)
 \big] \notag \\ 
 = & \big[ J^{y\top}_{ij} \mathbf{G}^\top(t; P_1, \dots, P_N ) J^{z}_{ij} \big] J^{z\top}_{ij}  \big[ - e^{-\alpha(T-1-t)} \kappa \mathbf{e}^z_n \mathbb{E}Z_t^{n\top} y_{t+1}  
  + \mathbb{E} \mathbf{Z}_t^\top S_n(t+1)
 \big] 
 \notag \\ 
 %%%%%%%%% 
 = & 
 \begin{cases} 
\mathbf{G}^\top(  t; P^\dagger_1, \dots, P^\dagger_{i-1}, P^\dagger_j, P^\dagger_{i+1}, \dots, 
P^\dagger_{j-1}, P^\dagger_i, P^\dagger_{j+1}, \dots, P^\dagger_N) 
\cdot \\ 
\hspace{2cm}  
\big[ - e^{-\alpha(T-1-t)} \kappa \mathbf{e}^z_n \mathbb{E}Z_t^{n \top } y_{t+1}  
  + \mathbb{E} \mathbf{Z}_t^\top S_n^\dagger(t+1)
 \big] 
, \quad \mbox{if} \quad n \neq i , j ;  \\ 
%%%%% 
\mathbf{G}^\top(  t; P^\dagger_1, \dots, P^\dagger_{i-1}, P^\dagger_j, P^\dagger_{i+1}, \dots, 
P^\dagger_{j-1}, P^\dagger_i, P^\dagger_{j+1}, \dots, P^\dagger_N)  
\cdot 
\\ 
\hspace{2cm} 
\big[ - e^{-\alpha(T-1-t)} \kappa \mathbf{e}^z_j \mathbb{E}Z_t^{j \top} y_{t+1}  
  + \mathbb{E} \mathbf{Z}_t^\top S_i^\dagger(t+1)
 \big] 
, \quad \mbox{if} \quad n = i ;    \\ 
%%%%% 
\mathbf{G}^\top(  t; P^\dagger_1, \dots, P^\dagger_{i-1}, P^\dagger_j, P^\dagger_{i+1}, \dots, 
P^\dagger_{j-1}, P^\dagger_i, P^\dagger_{j+1}, \dots, P^\dagger_N)  
\cdot 
\\ 
\hspace{2cm} 
\big[ - e^{-\alpha(T-1-t)} \kappa \mathbf{e}^z_i \mathbb{E}Z_t^{i \top} y_{t+1}  
  + \mathbb{E} \mathbf{Z}_t^\top S_j^\dagger(t+1)
 \big] 
, \quad \mbox{if} \quad n = j .     
\end{cases} 
\label{JG[Sn]:3cases} 
\end{align}

\begin{align} 
& J^{y\top}_{ij} \Big\{ e^{-\alpha(T-1-t)} \big[ \kappa ( \Theta_n + \overline\Theta/N )^\top \mathbb{E} Z^n_t e^{z\top}_n 
 + \overline\kappa (\Theta_n - \Theta/N )^\top \mathbb{E}  ( Z^n_t \mathbf{e}_n^{z\top} - \mathbf{1}_y \mathbf{Z}_t /N )  
 \big] 
  \notag \\ 
 & \qquad \qquad  
 + ( \mathbf{\Theta} + \overline{\mathbf{\Theta}}/N )^\top P_n(t+1) \mathbb{E}\mathbf{Z}_t 
 \Big\} J^{z}_{ij}  J^{z} _{ij} \mathbf{H}(t; P_1, \dots, P_N; S_1, \dots, S_N ) 
 \notag \\ 
 %%%%%%%%%%%% 
 = & 
 \begin{cases}  
 \Big\{ e^{-\alpha(T-1-t)} \big[ \kappa ( \Theta_n + \overline\Theta/N )^\top \mathbb{E} Z^n_t e^{z\top}_n 
 + \overline\kappa (\Theta_n - \Theta/N )^\top \mathbb{E}( Z^n_t \mathbf{e}_n^{z\top} - \mathbf{1}_y \mathbf{Z}_t /N )  
 \big] 
   \\ 
 \hspace{2cm}  
 + ( \mathbf{\Theta} + \overline{\mathbf{\Theta}}/N )^\top P_n^\dagger(t+1) \mathbb{E}\mathbf{Z}_t 
 \Big\} \cdot 
 \\ 
 \hspace{0cm} \mathbf{H} (t; S_1^\dagger, \dots, S_{i-1}^\dagger, S_j^\dagger, S_{i+1}^\dagger, \dots, S_{j-1}^\dagger, S_i^\dagger, S_{j+1}^\dagger, \dots, S_N^\dagger; 
  \\ 
 \hspace{1cm} P_1^\dagger, \dots, P_{i-1}^\dagger, P_j^\dagger, P_{i+1}^\dagger, \dots, P_{j-1}^\dagger, P_i^\dagger, P_{j+1}^\dagger, \dots, P_N^\dagger )  , 
 \qquad \text{if} \quad n \neq i, j , 
 \\ 
 %%%%%%%%%%%%% 
  \Big\{ e^{-\alpha(T-1-t)} \big[ \kappa ( \Theta_j + \overline\Theta/N )^\top \mathbb{E} Z^j_t e^{z\top}_j 
 + \overline\kappa (\Theta_j - \Theta/N )^\top \mathbb{E}  ( Z^j_t \mathbf{e}_j^{z\top} - \mathbf{1}_y \mathbf{Z}_t /N )  
 \big] 
   \\ 
 \hspace{2cm}  
 + ( \mathbf{\Theta} + \overline{\mathbf{\Theta}}/N )^\top P_i^\dagger(t+1) \mathbb{E}\mathbf{Z}_t 
 \Big\} \cdot 
 \\ 
 \hspace{0cm} \mathbf{H} (t; S_1^\dagger, \dots, S_{i-1}^\dagger, S_j^\dagger, S_{i+1}^\dagger, \dots, S_{j-1}^\dagger, S_i^\dagger, S_{j+1}^\dagger, \dots, S_N^\dagger; 
  \\ 
 \hspace{1cm} P_1^\dagger, \dots, P_{i-1}^\dagger, P_j^\dagger, P_{i+1}^\dagger, \dots, P_{j-1}^\dagger, P_i^\dagger, P_{j+1}^\dagger, \dots, P_N^\dagger )  , 
 \qquad \text{if} \quad  n = i , 
  \\ 
 %%%%%%%%%%%%%% 
  \Big\{ e^{-\alpha(T-1-t)} \big[ \kappa ( \Theta_i + \overline\Theta/N )^\top \mathbb{E} Z^i_t e^{z\top}_i 
 + \overline\kappa (\Theta_i - \Theta/N )^\top \mathbb{E} ( Z^i_t \mathbf{e}_i^{z\top} - \mathbf{1}_y \mathbf{Z}_t /N )  
 \big] 
   \\ 
 \hspace{2cm}  
 + ( \mathbf{\Theta} + \overline{\mathbf{\Theta}}/N )^\top P_j^\dagger(t+1) \mathbb{E}\mathbf{Z}_t 
 \Big\} \cdot 
 \\ 
 \hspace{0cm} \mathbf{H} (t; S_1^\dagger, \dots, S_{i-1}^\dagger, S_j^\dagger, S_{i+1}^\dagger, \dots, S_{j-1}^\dagger, S_i^\dagger, S_{j+1}^\dagger, \dots, S_N^\dagger; 
  \\ 
 \hspace{1cm} P_1^\dagger, \dots, P_{i-1}^\dagger, P_j^\dagger, P_{i+1}^\dagger, \dots, P_{j-1}^\dagger, P_i^\dagger, P_{j+1}^\dagger, \dots, P_N^\dagger )  , 
 \qquad \text{if} \quad n =  j .  
 \end{cases}
\label{J[Pn]H:3cases}
\end{align}

\begin{align} 
 & J^{y\top}_{ij} \big[  - e^{-\alpha(T-1-t)} \kappa (\Theta_n + \overline\Theta/N )^\top y_{t+1} 
  +   (\mathbf{\Theta} + \overline{\Theta}/N )^\top S_n(t+1) \big]  
  \notag \\ 
 = &
\begin{cases} 
- e^{-\alpha(T-1-t)} \kappa (\Theta_n + \overline\Theta/N )^\top y_{t+1} 
  +   (\mathbf{\Theta} + \overline{\Theta}/N )^\top S_n^\dagger(t+1) , 
  \quad \text{if} \quad n \neq i, j, 
  \\
%%%%% 
- e^{-\alpha(T-1-t)} \kappa (\Theta_j + \overline\Theta/N )^\top y_{t+1} 
  +   (\mathbf{\Theta} + \overline{\Theta}/N )^\top S_i^\dagger(t+1) , 
  \quad \text{if} \quad n = i , 
  \\ 
%%%%% 
- e^{-\alpha(T-1-t)} \kappa (\Theta_i + \overline\Theta/N )^\top y_{t+1} 
  +   (\mathbf{\Theta} + \overline{\Theta}/N )^\top S_j^\dagger(t+1) , 
  \quad \text{if} \quad n = j, 
\end{cases} 
\label{J[Sn]:3cases} 
\end{align} 
Then~\eqref{Sdagger=Psi(Sdagger)} follows from~\eqref{JGQnH:3cases}, \eqref{JG[Sn]:3cases}, \eqref{J[Pn]H:3cases}, and~\eqref{J[Sn]:3cases}.

\bigskip 
\noindent 
\textbf{Step 2:} For each $2\leq j \leq N$, denote $S_n^\ddagger=J^{y\top}_{1j} S_n $, $1\leq n \leq N$. Then we have that 
\begin{align} 
\big( S^\ddagger_j, S^\ddagger_2, \dots, S^\ddagger_{j-1}, S^\ddagger_{1}, S^\ddagger_{j+1}, \dots, S^\ddagger_{N}   \big) 
\notag 
\end{align} 
and $(S_1, \dots, S_N)$ satisfy~\eqref{ODE:Sn}. This implies that $S_n = J^{y\top}_{1j} S_1$ for $2\leq j \leq N$, and thus~\eqref{Sn=J1nS1} holds. 

\hfill$\blacksquare$
%\end{proof} 

%\begin{proof} 
\emph{Proof of Corollary~\ref{cor:bfHsubmat}.}
Substituting~\eqref{P1:PiN1234}-\eqref{Pn=J1nP1J1n} and~\eqref{S1:submat}-\eqref{Sn=J1nS1} into the expression of $\mathbf{H}(\cdot)$ in~\eqref{bfH}, and using Assumption~\eqref{assm:Zt} and the block structure of the inverse matrix given in~\eqref{inv(bfA+bfhatA1+bfhatA2+gammaI):submat}, we obtain that $\mathbf{H}(\cdot)$ admits the submatrix decomposition~\eqref{bfH:submat} with submatrices given by~\eqref{HN}. 
\hfill$\blacksquare$ 
%\end{proof}

%\begin{proof}
\emph{Proof of Corollary~\ref{cor:ODEs:XiN12}.} 
By taking $n=1$ and substituting~\eqref{EZit=MZ1t}, \eqref{P1:PiN1234}, \eqref{bfG:submat}, \eqref{bfQn:submat}, and~\eqref{bfH:submat} into~\eqref{ODE:Sn}, we obtain that  
the submatrices $\Xi_i^N$, $i=1,2$, satisfy the system~\eqref{ODE:XiNi}, where the mappings 
$\psi_i^N(\cdot)$, $i=1, 2$, are defined below.  
\allowdisplaybreaks 
\begin{align}  
& \psi^N_1\big( t; \Pi^N_1, \Pi^N_2, \Pi^N_3, \Pi^N_4, \Xi_1^N, \Xi_2^N \big) 
 \label{XiN1} \\ 
  \eqdef 
 &  G_1^{N\top}(t) \big[ Q_1^N(t) + (N-1) Q_2^N(t) \big] H^N(t) 
 \notag \\ 
 &  + (N-1) G_2^{N\top}(t) \big[ Q_2^{N\top}(t) + Q_3^N(t) + (N-2) Q_4^N(t) \big] H^N(t) 
  \notag \\ 
  & - e^{-\alpha(T-1-t)} \kappa G_1^{N\top}(t) ( M^{(1)}_Z(t) )^\top y_{t+1} 
  + G_1^{N\top}(t) ( M^{(1)}_Z(t) )^\top \Xi_1^N(t+1) 
  \notag \\ 
  & + (N-1) G_2^{N\top}(t) ( M^{(1)}_Z(t) )^\top \Xi_2^N(t+1) 
   \notag \\ 
   & + e^{-\alpha(T-1-t)} \kappa ( \theta + \overline\theta/N )^\top M^{(1)}_Z(t) H^N(t) 
   \notag \\ 
   & + (\theta + \overline\theta/N)^\top \big[ \Pi_1^N(t+1) + (N-1) \Pi_2^N(t+1) \big] M^{(1)}_Z(t) H^N(t) 
   \notag \\ 
   & + \overline\theta^\top (1/N) \big[ \Pi_2^{N\top}(t+1) + \Pi_3^N(t+1) + (N-2) \Pi_4^N(t+1) \big] M^{(1)}_Z(t) H^N(t) 
   \notag \\ 
   & - e^{-\alpha(T-1-t)} \kappa (\theta + \overline\theta/N)^\top y_{t+1} 
   + ( \theta + \overline\theta/N )^\top \Xi_1^N(t+1) + \overline\theta^\top(1/N) \Xi_2^N(t+1)(N-1)
\notag 
\end{align} 
and 
\allowdisplaybreaks
\begin{align} 
&  \psi^N_2 \big( t; \Pi^N_1, \Pi^N_2, \Pi^N_3, \Pi^N_4 , \Xi_1^N, \Xi_2^N \big) 
  \label{XiN2} \\ 
 \eqdef & \big[ G_1^{N\top}(t) + (N-2) G_2^{N\top}(t) \big] \big[ Q_2^{N\top}(t) + Q_3^N(t) + (N-2) Q_4^N(t) \big] H^N(t) 
 \notag \\ 
 & + G_2^{N\top}(t) \big[ Q_1^N(t) + (N-1) Q_2^N(t) \big] H^N(t) 
 \notag \\ 
 & - e^{-\alpha(T-1-t)} \kappa G_2^{N\top}(t) (M^{(1)}_Z(t))^\top y_{t+1} 
 \notag \\ 
 &  + G_1^{N\top}(t) ( M^{(1)}_Z(t) )^\top \Xi_2^N(t+1) 
   + G_2^{N\top}(t) ( M^{(1)}_Z(t) )^\top \big[ \Xi_1^N(t+1) + (N-2) \Xi_2^N(t+1) \big]
\notag \\ 
& + e^{-\alpha(T-1-t)} \kappa \overline\theta^\top (1/N) 
 M^{(1)}_Z(t) H^N(t) 
\notag \\ 
& + \overline\theta^\top (1/N) \big[ \Pi_1^N(t+1) + (N-1) \Pi_2^N(t+1) \big] M^{(1)}_Z(t) H^N(t) 
\notag \\ 
& + \big[ \theta^\top + \overline\theta^\top  (N-1)/N \big] 
 \big[ \Pi_2^{N\top}(t+1) + \Pi_3^N(t+1) + (N-2) \Pi_4^N(t+1) \big] M^{(1)}_Z(t) H^N(t) 
 \notag \\ 
 &  - e^{-\alpha(T-1-t)} \kappa \overline\theta^\top (1/N) y_{t+1} 
 + \overline\theta^\top (1/N) \Xi_1^N(t+1) 
  + \big[ \theta^\top + \overline\theta^\top (N-1)/N \big] \Xi_2^N(t+1) . 
\notag 
\end{align} 
\hfill$\blacksquare$ 

\section{Proof of the Decentralized Policy} 
\label{appendix:decentralizedStrat}

\begin{comment} 
As shown in the proof of Theorem~\ref{thm:NagentNash}, the value function of each agent admits a quadratic affine form~\eqref{Vn(t,y):QuadraticAffine}. Furthermore, leveraging the matrix decompositions in  Propositions~\ref{prop:PnSubmat:PiN1234} and~\ref{prop:SnSubmat}, we obtain, for $\hat{\mathbf{y}} = \big[ \hat{y}_1^\top, \dots, \hat{y}_N^\top \big]^\top$, that 
\begin{align} 
 V_n( t, \hat{\mathbf{y}} ) 
= & \hat{\mathbf{y}}^\top P_n(t) \hat{\mathbf{y}} + 2 S_n^\top(t) \hat{\mathbf{y}}  + r_n(t)
\notag \\ 
= & \hat{y}_n^\top  \Pi_1^N(t) \hat{y}_n + \hat{y}_n^\top \Pi_2^N(t) \sum_{i=1\neq n}^N \hat{y}_i 
 +  \big( \sum_{i=1\neq n}^N \hat{y}_i \big)^\top  \Pi_2^{N\top}(t) \hat{y}_n 
 + \sum_{i=1\neq n}^N \hat{y}_i^\top \Pi_3^N(t) \hat{y}_i 
 \notag \\ 
& 
+  \sum_{i\neq j = 1}^N \hat{y}_i^\top \Pi_4^N(t) \hat{y}_j  
   + 2 \, \Xi_1^{N\top}(t) \hat{y}_n + 2 \, \Xi_2^{N\top} \sum_{i=1\neq n}^N \hat{y}_i 
  + r_n(t) . 
\notag 
\end{align} 
\end{comment}

In this section, we prove Theorem~\ref{thm:Y(N)->barY}, which characterizes the asymptotic behavior of the predictions under the decentralized learning policy~\eqref{betan_decentralized}. 

Following the rescaling method in~\cite[Eq.~(4.9)]{HY21mfg}, the sub-matrices in~\eqref{P1:PiN1234} and~\eqref{S1:submat} exhibit the asymptotic orders
\begin{align}
 & \Pi_1^N(t) = \mathcal{O}(1), \quad 
 \Pi_2^N(t) = \mathcal{O}(1/N), \quad 
 \Pi_3^N(t) = \mathcal{O}(1/N^2), \quad 
 \Pi_4^N(t) = \mathcal{O}(1/N^2), \notag \\
 & \Xi_1^N(t) = \mathcal{O}(1), \quad 
 \Xi_2^N(t) = \mathcal{O}(1/N). \notag
\end{align}
Motivated by these scalings, we introduce the normalized variables
\begin{align}
& \Pi_1^N(t) = \Lambda_1^N(t), \quad 
  \Pi_2^N(t) = \frac{\Lambda_2^N(t)}{N}, \quad 
  \Pi_3^N(t) = \frac{\Lambda_3^N(t)}{N^2}, \quad 
  \Pi_4^N(t) = \frac{\Lambda_4^N(t)}{N^2},
\label{LambdaN1234=PiN1234} \\
& \Xi_1^N(t) = \chi_1^N(t), \quad 
  \Xi_2^N(t) = \frac{\chi_2^N(t)}{N},
\label{chiN12=XiN12}
\end{align}
where the rescaled quantities $\{ \Lambda_i^N(\cdot) \}_{i=1}^4$ and $\{ \chi_i^N(\cdot) \}_{i=1}^2$ are all of order $\mathcal{O}(1)$. This normalization removes the explicit dependence on $N$ in the leading-order terms and facilitates the analysis of the large-population limit. In particular, it enables us to derive a limiting system of equations for $\{\Lambda_i\}_{i=1}^4$ and $\{\chi_i\}_{i=1,2}$ as $N \to \infty$, which characterizes the mean-field behavior of the model and serves as the basis for our subsequent analysis.

Upon introducing~\eqref{LambdaN1234=PiN1234} and~\eqref{chiN12=XiN12}, we obtain a system of equations for the $\mathcal{O}(1)$ quantities $\{\Lambda_i^N\}_{i=1}^4$ and $\{\chi_i^N\}_{i=1,2}$, as stated in the following lemma.

\begin{lemma}  
\label{lem:ODEsLambdaN123}
The matrix-valued functions $\{\Lambda_i^N\}_{i=1}^4$ and $\{\chi_i^N \}_{i=1,2}$ defined in~\eqref{LambdaN1234=PiN1234} and~\eqref{chiN12=XiN12} satisfy
\begin{align} 
\begin{cases} 
 \Lambda_1^N(t) = \varphi_1\!\left( t; \Lambda_1^N \right) + g_1^N\!\left( t; \Lambda_1^N , \Lambda_2^N, \Lambda_3^N, \Lambda_4^N \right),  \\[4pt]
 \Lambda_2^N(t) = \varphi_2\!\left( t; \Lambda_1^N, \Lambda_2^N \right) 
 + g_2^N\!\left( t; \Lambda_1^N, \Lambda_2^N , \Lambda_3^N , \Lambda_4^N \right),  \\[4pt]
 \Lambda_3^N(t) = \varphi_3\!\left( t ; \Lambda_1^N, \Lambda_2^N, \Lambda_3^N, \Lambda_4^N \right) 
 + g_3^N\!\left( t, \Lambda_1^N , \Lambda_2^N , \Lambda_3^N, \Lambda_4^N \right),  \\[4pt]
 \Lambda_4^N(t) = \varphi_4\!\left( t ; \Lambda_1^N, \Lambda_2^N, \Lambda_3^N, \Lambda_4^N \right) 
 + g_4^N\!\left( t; \Lambda_1^N , \Lambda_2^N , \Lambda_3^N , \Lambda_4^N \right),  \\[6pt]
 \chi_1^N(t) = \psi_1\!\left( t; \Lambda_1^N, \Lambda_2^N, \chi_1^N \right) 
 + g_5^N\!\left( t; \Lambda_1^N, \Lambda_2^N, \Lambda_3^N, \Lambda_4^N, \chi_1^N, \chi_2^N \right), \\[4pt]
 \chi_2^N(t) = \psi_2\!\left( t; \Lambda_1^N, \Lambda_2^N, \Lambda_4^N, \chi_1^N, \chi_2^N \right) 
 + g_6^N\!\left( t; \Lambda_1^N, \Lambda_2^N, \Lambda_3^N, \Lambda_4^N, \chi_1^N, \chi_2^N \right).
\end{cases} 
\label{ODEs:LambdaN1234ChiN12}
\end{align} 
Here, $\{\varphi_i(\cdot)\}_{i=1}^4$ and $\{\psi_i(\cdot)\}_{i=1,2}$ are defined in~\eqref{varphi1}--\eqref{varphi4} and~\eqref{psi1}--\eqref{psi2}, respectively. The residual terms $\{g_i^N(\cdot)\}_{i=1}^6$ satisfy
\begin{align} 
\begin{aligned} 
&\lim_{N\to\infty} \big\| g_i^N\!\left( t; \Lambda_1^N, \Lambda_2^N, \Lambda_3^N, \Lambda_4^N \right) \big\| = 0, \quad i = 1,2,3,4, \\  
&\lim_{N\to\infty} \big\|g_i^N\!\left( t; \Lambda_1^N, \Lambda_2^N, \Lambda_3^N, \Lambda_4^N, \chi_1^N, \chi_2^N \right) \big\| = 0, \quad i = 5,6.
\end{aligned} 
\label{gN123456->0}
\end{align} 
\end{lemma}  

We next introduce the limiting system corresponding to~\eqref{ODEs:LambdaN1234ChiN12} as the population size $N \to \infty$, justified by Lemma~\ref{lem:|LambdaN-Lambda|->0}:
\begin{align} 
\begin{cases} 
 \Lambda_1(t) = \varphi_1\!\left( t; \Lambda_1 \right),  \\[4pt]
 \Lambda_2(t) = \varphi_2\!\left( t; \Lambda_1, \Lambda_2 \right),  \\[4pt]
 \Lambda_3(t) = \varphi_3\!\left( t; \Lambda_1, \Lambda_2, \Lambda_3, \Lambda_4 \right),  \\[4pt]
 \Lambda_4(t) = \varphi_4\!\left( t; \Lambda_1, \Lambda_2, \Lambda_3, \Lambda_4 \right),  \\[6pt]
 \chi_1(t) = \psi_1\!\left( t; \Lambda_1, \Lambda_2, \chi_1 \right), \\[4pt]
 \chi_2(t) = \psi_2\!\left( t; \Lambda_1, \Lambda_2, \Lambda_4, \chi_1, \chi_2 \right).
\end{cases} 
\label{ODEs:Lambda1234Chi12}
\end{align}

\begin{lemma}  
\label{lem:|LambdaN-Lambda|->0}
Let $\{\Lambda_i\}_{i=1}^4$ and $\{\chi_i \}_{i=1,2}$ be the solution of the system~\eqref{ODEs:Lambda1234Chi12}.  
Then it satisfies that for each $t$, 
\begin{align}
\begin{aligned} 
 & \lim_{N\to \infty } \big\| \Lambda_i^N(t) - \Lambda_i(t) \big\| = 0 , 
 \quad i = 1, 2, 3, 4, 
  \\ 
  & \lim_{N\to \infty } \big\| \chi_i^N(t) - \chi_i(t) \big\| = 0 , 
  \quad i = 1, 2.  
 \end{aligned} 
 \label{|LambdaN-Lambda|->0}
\end{align} 
\end{lemma}  

%\begin{proof} 
\emph{Proof of Lemma~\ref{lem:|LambdaN-Lambda|->0}.}
We prove the result for $\Lambda_1^N$; the arguments for $\Lambda_i^N$, $i=2,3,4$, and $\chi_i^N$, $i=1,2$, are analogous.

By~\eqref{ODEs:LambdaN1234ChiN12} and~\eqref{ODEs:Lambda1234Chi12}, we have
\begin{align}
\Lambda_1^N(t) - \Lambda_1(t)
= \varphi_1\big(t; \Lambda_1^N\big) - \varphi_1\big(t; \Lambda_1\big)
+ g_1^N\big(t; \Lambda_1^N,\Lambda_2^N,\Lambda_3^N,\Lambda_4^N\big).
\end{align}

By~\eqref{G12H} and~\eqref{Q1234}, the mappings $G_1(\cdot)$ and $Q_1(\cdot)$ are Lipschitz in their arguments. It follows that $\varphi_1$ is Lipschitz, i.e., there exists a constant $C>0$ such that
\begin{align}
\big\| \varphi_1(t;\Lambda_1^N) - \varphi_1(t;\Lambda_1) \big\|
\le C \big\| \Lambda_1^N(t+1) - \Lambda_1(t+1) \big\|.
\end{align}

Therefore,
\begin{align}
\big\| \Lambda_1^N(t) - \Lambda_1(t) \big\|
\le C \big\| \Lambda_1^N(t+1) - \Lambda_1(t+1) \big\|
+ \big\| g_1^N(t;\Lambda_1^N,\Lambda_2^N,\Lambda_3^N,\Lambda_4^N) \big\|.
\label{lambda_recursion}
\end{align}

Since $\Lambda_1^N(T)=\Lambda_1(T)=0$, and $g_1^N \to 0$ uniformly by~\eqref{gN123456->0}, we apply backward induction on $t$ to conclude that
\[
\lim_{N\to\infty} \big\| \Lambda_1^N(t) - \Lambda_1(t) \big\| = 0,
\quad \forall\, t=0,1,\dots,T.
\]

The same argument applies to $\Lambda_i^N$, $i=2,3,4$, and $\chi_i^N$, $i=1,2$, completing the proof.
\hfill$\blacksquare$

\begin{lemma} 
The sub-matrices $E^N(t)$ and $M^N(t)$  in~\eqref{inv(bfA+bfhatA1+bfhatA2+gammaI):submat}
and the associated $F^N(t)$ and $K^N(t)$ admits the decompositions  
\begin{align} 
 F^N(t) = & F(t; \Lambda_1^N ) + \widetilde{F}^N(t) , \notag \\ 
  N K^N(t) = & K(t; \Lambda_2^N ) + \widetilde{K}^N(t) ,  \notag \\ 
  M^N(t) = & M(t; \Lambda_1^N ) + \widetilde{M}^N(t) ,  \notag \\ 
  N E^N(t) = & E(t; \Lambda_1^N, \Lambda_2^N ) + \widetilde{E}^N(t) , \notag 
\end{align} 
where 
\begin{align} 
\begin{cases} 
  F( t; \Lambda_1^N ) =  e^{-\alpha(T-1-t)}   
 \big[  \big( \kappa + \overline\kappa  \big) 
  M^{(2)}_Z(t)  
 + \gamma I \big] 
 + M^{(2)}_{Z, \, \Lambda_1^N(t+1)}(t) , 
  \\ 
   K( t; \Lambda_2^N )  =   - e^{-\alpha(T-1-t)} \overline\kappa ( M^{(1)}_Z(t) )^\top M^{(1)}_Z(t) 
  + ( M^{(1)}_Z(t) )^\top \Lambda_2^N(t+1) M^{(1)}_Z(t) 
  ,   \\ 
  %%%
  M(t; \Lambda_1^N ) =  (F(t; \Lambda_1^N ) )^{-1} ,  
  \\ 
 %%%%%%%
  E(t; \Lambda_1^N , \Lambda_2^N ) =  - \big[ F(t; \Lambda_1^N ) + K(t; \Lambda_2^N ) \big]^{-1} K(t; \Lambda_2^N ) 
 (F(t; \Lambda_1^N ))^{-1}.   
 \end{cases} 
 \label{FKME}
\end{align} 
and for each $t$, $\widetilde{F}^N(t)$, $\widetilde{K}^N(t)$, $\widetilde{E}^N(t)$, and $\widetilde{M}^N(t)$ are in $\mathcal{O}(1/N)$. 
\end{lemma}

\begin{lemma} 
\label{lem:GN12HN=GH+tildeGH}
The sub-matrices $G_i^N(t)$, $i=1,2$ in~\eqref{bfG:submat} and $H^N(t)$ in~\eqref{bfH:submat} can be decomposed as 
\begin{align} 
 G_1^N(t) = & G_1(t; \Lambda_1^N ) + \widetilde{G}_1^N(t) , 
 \notag \\ 
 N G_2^N(t) = & G_2(t; \Lambda_1^N, \Lambda_2^N ) + \widetilde{G}_2^N(t) , 
 \notag \\ 
 H^N(t) = & H(t; \Lambda_1^N, \chi_1^N ) + \widetilde{H}^N(t) , 
 \notag 
\end{align} 
where 
\begin{align} 
\begin{cases} 
G_1(t; \Lambda_1^N ) = - e^{-\alpha(T-1-t)} (\kappa + \overline{\kappa}) M(t; \Lambda_1^N ) 
 ( M^{(1)}_Z(t) )^\top \theta  
  - M(t; \Lambda_1^N ) ( M^{(1)}_Z(t) )^\top \Lambda^N_1(t+1) \theta  
   \\ 
%%%%%%%
 G_2(t; \Lambda_1^N, \Lambda_2^N )=  - e^{-\alpha(T-1-t)} ( \kappa + \overline{\kappa} ) 
  E(t; \Lambda_1^N , \Lambda_2^N ) 
  ( M^{(1)}_Z(t) )^\top \theta  
  \\ 
  + e^{-\alpha(T-1-t)} \big[ M(t; \Lambda_1^N ) + E(t; \Lambda_1^N , \Lambda_2^N ) \big] 
 ( M^{(1)}_Z(t) )^\top 
 \big( \overline{\kappa} \theta 
  - \kappa \overline{\theta} \,   \big) 
   \\ 
   - \big[ E(t; \Lambda_1^N, \Lambda_2^N ) (M^{(1)}_Z(t))^\top \Lambda_1^N(t+1) \theta + ( M(t; \Lambda_1^N ) + E(t; \Lambda_1^N , \Lambda_2^N ) ) 
   ( M^{(1)}_Z(t))^\top \Lambda_2^N(t+1) \theta  \big] 
   \\ 
   - \big[ M(t; \Lambda_1^N ) + E(t; \Lambda_1^N , \Lambda_2^N ) \big] ( M^{(1)}_Z(t) )^\top 
  \big( \Lambda_1^N(t+1) + \Lambda_2^N(t+1)  \big) \overline\theta ,  
 \\ 
 %%%%%% 
  H(t; \Lambda_1^N, \chi_1^N ) =  - [ M(t; \Lambda_1^N ) + E(t; \Lambda_1^N , \Lambda_2^N ) ] 
  ( M^{(1)}_Z(t) )^\top 
  \big\{ e^{-\alpha(T-1-t)} (-\kappa)  
    y_{t+1}
  +   \chi_1^N(t+1) \big\} 
 \end{cases} 
 \label{G12H}
\end{align} 
and for each $t$, $\widetilde{G}_1^N(t)$, $\widetilde{G}_2^N(t)$, and $\widetilde{H}^N(t)$ are in $\mathcal{O}(1/N)$. 
\end{lemma}

\begin{lemma} 
The sub-matrices $Q_i^N(t)$, $i=1, 2, 3$ in~\eqref{bfQn:submat} can be decomposed as  
\begin{align} 
 Q_1^N(t) = & Q_1(t; \Lambda_1^N ) + \widetilde{Q}_1^N(t) , \notag \\ 
 N Q_2^N(t) = & Q_2(t; \Lambda_2^N ) + \widetilde{Q}_2^N(t) , \notag \\ 
 N^2 Q_3^N(t) = & Q_3(t; \Lambda_3^N) + \widetilde{Q}_3^N(t) , \notag \\ 
 N^2 Q_4^N(t) = & Q_4(t; \Lambda_4^N) + \widetilde{Q}_4^N(t)  \notag 
\end{align} 
where 
\begin{align} 
\begin{cases} 
Q_1(t, \Lambda_1^N ) =  (\kappa + \overline{\kappa}) M^{(2)}_Z(t) + \gamma I  
 + M^{(2)}_{Z, \, \Lambda_1^N(t+1)}(t)  , 
 \\  
 Q_2(t, \Lambda_2^N ) =  - e^{-\alpha(T-1-t)} ( M^{(1)}_Z(t) )^\top M^{(1)}_Z(t)   
 +  ( M^{(1)}_Z(t) )^\top \Lambda_2^N(t+1) M^{(1)}_Z(t) , 
  \\ 
 Q_3(t, \Lambda_3^N ) =  M^{(2)}_{Z, \, \Lambda_3^N(t+1) }(t) 
+  e^{-\alpha(T-1-t)} \overline\kappa M^{(2)}_Z(t)  , 
  \\ 
 Q_4(t, \Lambda_4^N) =  ( M^{(1)}_Z(t) )^\top \Lambda_4^N(t+1) M^{(1)}_Z(t) 
 + e^{-\alpha(T-1-t)} \overline\kappa ( M^{(1)}_Z(t) )^\top M^{(1)}_Z(t) , 
 \end{cases} 
  \label{Q1234}
\end{align} 
and for each $i=1, 2, 3, 4$ and for each $t$, 
$\widetilde{Q}_i^N(t) \in \mathcal{O}(1/N)$. 
\end{lemma} 

\begin{lemma} 
\label{lem:|hatY|bdd|checkY|bdd} 
%Suppose $\{Z^n_t\}_{n\in [N], t=0, \dots, T}$ are uniformly bounded, i.e., there exists a constant $K_Z<\infty$ such that for all $n$ and for all $t$, $\big\| Z^n_t \big\| \leq K_Z$ almost surely on $\Omega$. 
Under the hypothesis about $\{Z^n_t\}_{n\in[N], t=0, \dots, T}$ in Theorem~\ref{thm:Y(N)->barY},  
there exists a constant $0 < K_Y<\infty$ such that 
$\big\| \widehat{Y}^n_t \big\| \leq  K_{Y}$ and $\big\| \check{Y}^n_t \big\| \leq K_{Y}$ for all $n\in [N]$ and for all $t\in\{0, 1, \dots, T-1\}$. 
\end{lemma} 
\begin{proof}
Substituting~\eqref{betan_centralized} into~\eqref{hatYit+1finite} yields
\begin{align}
\widehat{Y}^n_{t+1} 
=  & \theta \widehat{Y}^n_t + \overline{\theta} \widehat{Y}^{(N)}_t 
+ Z^n_t \big[ G_1^{N}(t) \widehat{Y}^n_t 
+ G_2^{N}(t) \big( N \widehat{Y}^{(N)}_t - \widehat{Y}^n_t \big) + H^N(t) \big].
\notag
\end{align}
Using the triangle inequality and sub-multiplicativity of the Frobenius norm, we obtain
\begin{align}
\|\widehat{Y}^n_{t+1}\|
\le  & \|\theta\|\, \|\widehat{Y}^n_t\| 
+ \|\overline{\theta}\|\, \|\widehat{Y}^{(N)}_t\| \notag \\
& + \|Z^n_t\| \big( \|G_1^{N}(t)\|\, \|\widehat{Y}^n_t\| 
+ \|G_2^{N}(t)\| \big( N \|\widehat{Y}^{(N)}_t\| + \|\widehat{Y}^n_t\| \big)
+ \|H^N(t)\| \big).
\notag
\end{align}
By the uniform boundedness of $Z^n_t$, we have $\|Z^n_t\| \le K_Z$ almost surely. Moreover, since
\[
\|\widehat{Y}^{(N)}_t\| = \Big\| \frac{1}{N} \sum_{n=1}^N \widehat{Y}^n_t \Big\| 
\le \frac{1}{N} \sum_{n=1}^N \|\widehat{Y}^n_t\|,
\]
it follows that $\|\widehat{Y}^{(N)}_t\| \le \max_{n\in[N]} \|\widehat{Y}^n_t\|$. Therefore,
\begin{align}
\|\widehat{Y}^n_{t+1}\|
\le {} & \Big( \|\theta\| + \|\overline{\theta}\| \Big)\|\widehat{Y}^n_t\|
+ K_Z \Big( \|G_1^{N}(t)\| + (N+1)\|G_2^{N}(t)\| \Big)\|\widehat{Y}^n_t\|
+ K_Z \|H^N(t)\|.
\notag
\end{align}
Since $\mathbb{E}\|\widehat{Y}^n_0\| = \|y_0\|$ is uniformly bounded over $n\in[N]$, we can apply the above inequality together with Lemma~\ref{lem:GN12HN=GH+tildeGH} and proceed by induction on $t$ to conclude that $\|\widehat{Y}^n_t\|$ is uniformly bounded for all $n\in[N]$ and $t=0,1,\dots,T$. The boundedness of $\check{Y}^n_t$ follows by analogous arguments.
\end{proof}

\emph{Proof of Theorem~\ref{thm:Y(N)->barY}.}
Since $\{Z^n_t\}_{n=1}^N$ are i.i.d. with finite second moments, we have by the law of large numbers~\cite[Theorem 7.4.3]{GrimmettStirzaker01Probability} that 
\begin{align} 
 \lim_{N\to\infty} \mathbb{E} \big\|  Z^{(N)}_t - M^{(1)}_Z(t)  \big\| 
  = 0 . 
  \label{E|Z(N)-M2|->0}
\end{align} 

By Theorem~\ref{thm:betan_centralized}, the Nash equilibrium strategy can be written as 
\begin{align} 
 \widehat\beta^n_t 
 = & 
 G_1^{N}(t) \widehat{Y}^n_t 
 + G_2^{N}(t) \sum_{j\neq n = 1}^N \widehat{Y}^j_t + H^N(t) , 
 \quad n = 1, \dots, N, 
 \notag \\ 
 =& G_1^{N}(t) \widehat{Y}^n_t 
 + G_2^{N}(t) \big( N \widehat{Y}^{(N)}_t - \widehat{Y}^n_t \big) + H^N(t) . 
 \notag 
 \end{align} 
 When all the $N$ agents take the same strategy, we have 
\begin{align} 
  \widehat{Y}^{(N)}_{t+1} 
 = & (\theta + \bar\theta ) \widehat{Y}^{(N)}_t + \frac{1}{N} \sum_{n=1}^N Z^n_t \widehat\beta^n_t 
 \notag \\ 
  = & (\theta + \bar\theta ) \widehat{Y}^{(N)}_t 
  + \frac{1}{N} \sum_{n=1}^N  Z^n_t \big[ G_1^{N}(t) \widehat{Y}^n_t 
 + G_2^{N}(t) \big( N \widehat{Y}^{(N)}_t - \widehat{Y}^n_t \big) + H^N(t)  \big]  . 
 \label{hatYN-ClosedLoop}
\end{align} 
Note that $\widehat{Y}^{(N)}_0=\overline{Y}_0=y_0$ by~\eqref{barY}. 
Suppose by induction that $\lim_{N\to\infty} \mathbb{E}\big\| \widehat{Y}^{(N)}_t - \overline{Y}_t \big\| = 0$. 
Taking difference between~\eqref{hatYN-ClosedLoop} and~\eqref{barY} and 
applying the triangle inequality of the Frobenius norm yield   
\begin{align} 
 \mathbb{E} \big\| \widehat{Y}^{(N)}_{t+1} - \overline{Y}_{t+1} \big\| 
  \leq & \big\| \theta + \bar\theta \big\| \mathbb{E} \big\| \widehat{Y}^{(N)}_t - \overline{Y}_t \big\| 
  + \mathbb{E} \Big\| (1/N)\sum_{n=1}^N Z^n_t G_1^N(t) \widehat{Y}^n_t - M^{(1)}_Z G_1(t; \Lambda_1) \overline{Y}_t \Big\| 
  \notag \\ 
  &  + \mathbb{E} \Big\| (1/N) \sum_{n=1}^N Z^n_t G_2^N(t) ( N \widehat{Y}^{(N)}_t - \widehat{Y}^n_t ) - M^{(1)}_Z(t) G_2(t; \Lambda_1, \Lambda_2 ) \overline{Y}_t \Big\| 
  \notag \\ 
  & + \mathbb{E} \Big\|(1/N)\sum_{n=1}^N Z^n_t H^N(t) - M^{(1)}_Z(t) H(t; \Lambda_1, \chi_1 ) \Big\| . 
  \notag 
\end{align} 
By the triangle inequality and Lemma~\ref{lem:|hatY|bdd|checkY|bdd}, we have 
\begin{align} 
 & \mathbb{E} \Big\| (1/N)\sum_{n=1}^N Z^n_t G_1^N(t) \hat{Y}^n_t - M^{(1)}_Z G_1(t; \Lambda_1) \overline{Y}_t \Big\|  
 \notag \\ 
 \leq & 
 \mathbb{E} \Big\| (1/N)\sum_{n=1}^N ( Z^n_t - M^{(1)}_Z(t) ) G_1^N(t) \widehat{Y}^n_t  \Big\| 
  + 
   \big\|  M^{(1)}_Z(t) \big\| \big\|  G_1^N(t) - G_1(t; \Lambda_1) \big\| \mathbb{E} \big\| \widehat{Y}^{(N)}_t  \big\|  
  \notag \\ 
  &  +  \big\|  M^{(1)}_Z(t) \big\| \big\| G_1(t; \Lambda_1 ) \big\|
  \mathbb{E} \big\|  \widehat{Y}^{(N)}_t - \overline{Y}_t  \big\|  
 \notag \\ 
 %%%%%% 
 \leq & 
  \mathbb{E} \big\|  Z^{(N)}_t - M^{(1)}_Z(t)  \big\|    
  \big\| G_1^N(t) \big\| K_Y 
  + 
   \big\|  M^{(1)}_Z(t) \big\| \big\|  G_1^N(t) - G_1(t; \Lambda_1) \big\|  K_Y  
  \notag \\ 
  &  +  \big\|  M^{(1)}_Z(t) \big\| \big\| G_1(t, \Lambda_1 ) \big\|
  \mathbb{E} \big\|  \widehat{Y}^{(N)}_t - \overline{Y}_t  \big\| 
  \to 0 \quad \text{ as $N\to\infty$} , 
 \notag 
\end{align} 
where the convergence to $0$ is justified by~\eqref{E|Z(N)-M2|->0}, Lemma~\ref{lem:GN12HN=GH+tildeGH}, and the induction hypothesis $\lim_{N\to\infty} \mathbb{E} \big\|  \widehat{Y}^{(N)}_t - \overline{Y}_t  \big\|=0$. 

Similarly, we have 
\begin{align} 
 & \mathbb{E} \Big\| (1/N) \sum_{n=1}^N Z^n_t G_2^N(t) ( N \widehat{Y}^{(N)}_t - \widehat{Y}^n_t ) - M^{(1)}_Z(t) G_2(t; \Lambda_1, \Lambda_2 ) \overline{Y}_t \Big\| 
 \notag \\ 
 \leq & 
 \mathbb{E} \Big\| (1/N) \sum_{n=1}^N Z^n_t G_2^N(t)  N \widehat{Y}^{(N)}_t  - M^{(1)}_Z(t) G_2(t; \Lambda_1, \Lambda_2 ) \overline{Y}_t \Big\| 
 + 
 \mathbb{E} \Big\| (1/N) \sum_{n=1}^N Z^n_t G_2^N(t)  \widehat{Y}^n_t   \Big\| 
 \notag \\ 
 \leq & 
 \mathbb{E} \big( \big\|  Z^{(N)}_t - M^{(1)}_Z(t) \big\|   
  \big\| N G_2^N(t) \big\| \big\| \widehat{Y}^{(N)}_t \big\| 
 \big)
  + 
   \big\| M^{(1)}_Z(t) \big\| \big\| N G_2^N(t) - G_2(t; \Lambda_1, \Lambda_2 ) \big\| 
    \mathbb{E} \big\| \widehat{Y}^{(N)}_t \big\| 
   \notag \\ 
   & +  \big\| M^{(1)}_Z(t) \big\| \big\| G_2(t; \Lambda_1, \Lambda_2 ) \big\| 
    \mathbb{E} \big\| \widehat{Y}^{(N)}_t - \overline{Y}_t \big\|
    + ( M^{(2)}_Z(t) )^{1/2} \big\| G_2^N(t) \big\| \mathbb{E} \big\| \widehat{Y} \big\|^2 K_Y 
    \\ 
    \leq & 
 \mathbb{E} \big( \big\|  Z^{(N)}_t - M^{(1)}_Z(t) \big\|  \big) 
 \big\| N G_2^N(t) \big\| K_Y
  + 
   \big\| M^{(1)}_Z(t) \big\| \big\| N G_2^N(t) - G_2(t; \Lambda_1, \Lambda_2 ) \big\| 
    K_Y 
   \notag \\ 
   & +  \big\| M^{(1)}_Z(t) \big\| \big\| G_2(t; \Lambda_1, \Lambda_2 ) \big\| 
    \mathbb{E} \big\| \widehat{Y}^{(N)}_t - \overline{Y}_t \big\| 
    \notag \\ 
 &   + ( M^{(2)}_Z(t) )^{1/2} \big\| G_2^N(t) \big\| \mathbb{E} \big\| \widehat{Y} \big\|^2 K_Y
 \to 0 , \text{ as $N\to\infty$} , 
\notag 
\end{align} 
and
\begin{align} 
& \mathbb{E} \Big\|(1/N)\sum_{n=1}^N Z^n_t H^N(t) - M^{(1)}_Z(t) H(t; \Lambda_1, \chi_1 ) \Big\|
\notag \\ 
\leq & 
\mathbb{E} \big\| Z^{(N)}_t - M^{(1)}_Z(t) \big\| \big\| H^N(t) \big\|   
 + \big\| M^{(1)}_Z(t) \big\| \big\| H^N(t) - H(t; \Lambda_1, \chi_1 ) \big\| 
 \to 0, \text{ as $N\to\infty$}. 
\notag 
\end{align} 
Summing up the above estimates, we have 
$\lim_{N\to\infty}\mathbb{E} \big\| \widehat{Y}^{(N)}_{t+1} - \overline{Y}_{t+1} \big\|=0$. 
Then we conclude by induction that $\lim_{N\to\infty} \mathbb{E}\big\| \widehat{Y}^{(N)}_{t} - \overline{Y}_{t} \big\| = 0$ for all $t$. 

The proof of $\lim_{N\to\infty} \mathbb{E}\big\| \check{Y}^{(N)}_{t} - \overline{Y}_{t} \big\| = 0$ is carried out in a similar way. And $\lim_{N\to\infty} \mathbb{E}\big\| \widehat{Y}^{(N)}_{t} - \check{Y}_{t} \big\| = 0$ follows by the triangle inequality of the Frobenius norm. 
\hfill$\blacksquare$

\bigskip 
The functions $\{\varphi_i(\cdot)\}_{i=1}^4$ and $\{\psi_i(\cdot)\}_{i=1}^2$ in Algorithm~\ref{alg:sync_subroutine} are defined as follows. 
\allowdisplaybreaks
\begin{align} 
 \varphi_1(t; \Lambda_1^N) 
 \eqdef & G_1^\top(t, \Lambda_1^N ) Q_1(t, \Lambda_1^N ) G_1(t, \Lambda_1^N ) 
 \label{varphi1} \\ 
 & + e^{-\alpha(T-1-t)} (\kappa + \overline\kappa) G_1^\top (t, \Lambda_1^N ) ( M^{(1)}_Z(t) )^\top  \theta 
  + G_1^\top(t, \Lambda_1^N ) ( M^{(1)}_Z(t) )^\top \Lambda_1^N(t+1) 
  \theta 
  \notag \\ 
  & + \big[ e^{-\alpha(T-1-t)} (\kappa + \overline\kappa) G_1^\top (t, \Lambda_1^N ) (M^{(1)}_Z(t))^\top  \theta 
  + G_1^\top(t, \Lambda_1^N ) ( M^{(1)}_Z(t) )^\top \Lambda_1^N(t+1) 
  \theta 
  \big]^\top  
  \notag \\ 
  & + e^{-\alpha(T-1-t)} ( \kappa  + \overline\kappa ) \theta^\top \theta 
  + \theta^\top \Lambda_1^N(t+1) \theta , 
\notag \\ 
%\end{align} 
%%%%%%%%%%%%%%%% 
%\begin{align} 
 \varphi_2(t; \Lambda_1^N, \Lambda_2^N  ) 
 \eqdef & G_1^\top(t, \Lambda_1^N ) Q_2(t, \Lambda_2^N ) G_1(t, \Lambda_1^N ) 
 \label{varphi2} \\ 
 & + G_1^\top(t, \Lambda_1^N ) \big[ Q_1(t, \Lambda_1^N ) + Q_2(t, \Lambda_2^N ) \big] G_2(t, \Lambda_1^N, \Lambda_2^N ) 
 \notag  \\ 
 & + e^{-\alpha(T-1-t)}  G_1^\top(t, \Lambda_1^N ) ( M^{(1)}_Z(t) )^\top ( - \kappa \overline\theta - \overline\kappa \theta ) 
 \notag \\ 
 & + G_1^\top(t, \Lambda_1^N ) ( M^{(1)}_Z(t))^\top \big\{ \Lambda_1^N(t+1) \overline\theta + \Lambda_2^N(t+1) (\theta + \overline\theta ) \big\}
 \notag \\ 
 & + \big\{ e^{-\alpha(T-1-t)} \big[ \kappa G_2^\top(t, \Lambda_1^N, \Lambda_2^N ) 
  - \overline\kappa G_1^\top(t, \Lambda_1^N ) \big] ( M^{(1)}_Z(t) )^\top \theta  \big\}^\top  
 \notag \\ 
 & + \big\{ G_2^\top(t, \Lambda_1^N, \Lambda_2^N ) ( M^{(1)}_Z(t) )^\top \Lambda_1^N(t+1) \theta \big\}^\top 
 \notag \\ 
 & + \big\{ \big[ G_1^\top(t, \Lambda_1^N ) + G_2^\top(t, \Lambda_1^N, \Lambda_2^N ) \big] ( M^{(1)}_Z(t) )^\top  \Lambda_2^{N\top}(t+1) \theta   \big\}^\top 
 \notag \\ 
 & + e^{-\alpha(T-1-t)} \big[ \kappa \theta^\top \overline\theta - \overline\kappa \theta^\top \theta  \big] 
 \notag \\ 
 & + \theta^\top \Lambda_2^N(t+1) \theta + \theta^\top \big[ \Lambda_1^N(t+1) + \Lambda_2^N(t+1) \big] \overline\theta , 
  \notag 
 \end{align} 
 %%%%%%%%%%%%%
\vspace{-0.5cm}
 \begin{align} 
  \varphi_3( t; \Lambda_1^N, & \Lambda_2^N, \Lambda_3^N , 
  \Lambda_4^N )  
 \label{varphi3} 
 \eqdef  G_1^\top(t, \Lambda_1^N ) 
  Q_3( t, \Lambda_3^N ) G_1(t, \Lambda_1^N ) 
  \\
& + G_1^\top(t, \Lambda_1^N ) \big[ Q_2^\top(t, \Lambda_2^N ) + Q_4(t, \Lambda_4^N ) \big] 
 G_2(t, \Lambda_1^N, \Lambda_2^N ) 
 \notag  \\ 
 & + G_2^\top(t, \Lambda_1^N, \Lambda_2^N ) \big[ Q_2(t, \Lambda_2^N ) + Q_4(t, \Lambda_4^N )  \big] 
  G_1(t, \Lambda_1^N ) 
 \notag \\ 
 & + G_2^\top(t, \Lambda_1^N, \Lambda_2^N ) \big[ Q_1(t, \Lambda_1^N ) + Q_2(t, \Lambda_2^N ) 
  + Q_2^\top(t, \Lambda_2^N ) + Q_4(t, \Lambda_4^N ) \big] G_2(t, \Lambda_1^N, \Lambda_2^N ) 
 \notag \\ 
 & + e^{-\alpha(T-1-t)} \kappa G_2^\top(t, \Lambda_1^N, \Lambda_2^N ) ( M^{(1)}_Z(t) )^\top \overline\theta (-1) 
 \notag \\ 
 & + e^{-\alpha(T-1-t)} \overline\kappa  G_1^\top(t, \Lambda_1^N ) ( M^{(1)}_Z(t) )^\top \theta  
 \notag \\ 
 & + G_2^\top(t, \Lambda_1^N, \Lambda_2^N ) ( M^{(1)}_Z(t) )^\top \big[ \Lambda_1^N(t+1) \overline\theta + \Lambda_2^N(t+1) (\theta + \overline\theta ) 
 +  \Lambda_2^{N\top}(t+1)\overline\theta + \Lambda_4^N(t+1) (\theta + \overline\theta ) 
 \big] 
 \notag \\ 
 & + G_1^\top(t, \Lambda_1^N ) ( M^{(1)}_Z(t) )^\top \big[ \Lambda_2^{N\top}(t+1) \overline\theta + \Lambda_3^N(t+1)\theta 
  + \Lambda_4^N(t+1) \overline\theta \,  \big] 
 \notag \\ 
 & + \big\{ e^{-\alpha(T-1-t)} \kappa G_2^\top(t, \Lambda_1^N, \Lambda_2^N ) ( M^{(1)}_Z(t) )^\top \overline\theta  (-1) \big\}^\top  
  + \big\{ e^{-\alpha(T-1-t)} \overline\kappa G_1^\top(t, \Lambda_1^N ) (M^{(1)}_Z(t))^\top \theta  \big\}^\top 
 \notag \\ 
 & + \Big\{  G_2^\top(t, \Lambda_1^N, \Lambda_2^N ) ( M^{(1)}_Z(t) )^\top \big[ \Lambda_1^N(t+1) \overline\theta + \Lambda_2^N(t+1) (\theta + \overline\theta ) 
 +  \Lambda_2^{N\top}(t+1)\overline\theta + \Lambda_4^N(t+1) (\theta + \overline\theta ) 
 \big]  \Big\}^\top 
 \notag \\ 
 & + \big\{ G_1^\top(t, \Lambda_1^N ) ( M^{(1)}_Z(t) )^\top \big[ \Lambda_2^{N\top}(t+1)\overline\theta 
  + \Lambda_3^N(t+1) \theta + \Lambda_4^N(t+1) \overline\theta \, \big] \big\}^\top 
  \notag \\ 
 & + e^{-\alpha(T-1-t)} \big[ \kappa \overline\theta^\top \overline\theta  
  + \overline\kappa \theta^\top \theta  \big] 
  \notag \\ 
& + \theta^\top \Lambda_3^N(t+1) \theta + \theta^\top \big[ \Lambda_2^{N\top}(t+1) + \Lambda_4^N(t+1) \big] \overline\theta  
\notag \\ 
& + \overline\theta^\top \big[ \Lambda_2^N(t+1) + \Lambda_4^N(t+1) \big] \theta 
\notag \\ 
& +  \overline\theta^\top \big[ \Lambda_1^N(t+1) + \Lambda_2^{N\top}(t+1) + \Lambda_2^N(t+1)   
+ \Lambda_4^N(t+1) 
\big] \overline{\theta} , 
\notag 
\end{align} 
\vspace{-0.5cm}
\allowdisplaybreaks
\begin{align} 
 \varphi_4( t; \Lambda_1^N, \Lambda_2^N, & \Lambda_3^N, \Lambda_4^N )     
 \eqdef   G_1^{\top}(t, \Lambda_1^N) Q_4(t, \Lambda_4^N) G_1(t, \Lambda_1^N) 
\label{varphi4}  \\ 
 & +  G_1^{\top}(t, \Lambda_1^N ) \big[ Q_2^{\top}(t, \Lambda_2^N )  +  Q_4(t, \Lambda_4^N ) \big] 
  G_2(t, \Lambda_1^N, \Lambda_2^N ) 
 \notag \\ 
& +  G_2^{\top}(t, \Lambda_1^N, \Lambda_2^N  ) \big[ Q_2(t, \Lambda_2^N )   
 + Q_4(t, \Lambda_4^N ) \big] G_1(t, \Lambda_1^N )
\notag \\ 
& +  G_2^{\top}(t,\Lambda_1^N, \Lambda_2^N ) \big\{ Q_1(t, \Lambda_1^N ) +  Q_2(t, \Lambda_2^N ) 
+ Q_2^{\top}(t, \Lambda_2^N )   
   + Q_4(t, \Lambda_4^N ) 
\big\} G_2(t, \Lambda_1^N, \Lambda_2^N ) 
\notag \\ 
%%%%%%% 
&  +  e^{-\alpha(T-1-t)} \kappa G_2^{\top}(t, \Lambda_1^N, \Lambda_2^N ) ( M^{(1)}_Z(t) )^\top \overline\theta  (-1)
\notag \\ 
& +  e^{-\alpha(T-1-t)} \overline\kappa  G_1^{\top}(t, \Lambda_1^N )   
( M^{(1)}_Z(t) )^\top \theta 
\notag \\ 
%%%
& +  G_2^{\top}(t,\Lambda_1^N, \Lambda_2^N ) ( M^{(1)}_Z(t) )^\top  \big[ \Lambda_1^{N}(t+1) \overline\theta  
 + \Lambda_2^N(t+1) \big( \theta + \overline\theta  \, \big) 
\big]
\notag \\ 
%%% 
& + \big[ G_1^{\top}(t, \Lambda_1^N ) +  G_2^{\top}(t, \Lambda_1^N, \Lambda_2^N ) \big] ( M^{(1)}_Z(t) )^\top 
\big\{ \Lambda_2^{N\top}(t+1) \overline\theta  
 + \Lambda_4^{N}(t+1) \big( \theta + \overline\theta \, \big)  \big\}
\notag \\ 
%%%%%%% 
&  +  \big\{ e^{-\alpha(T-1-t)} \kappa G_2^{\top}(t, \Lambda_1^N, \Lambda_2^N  ) ( M^{(1)}_Z(t) )^\top \overline\theta  (-1) 
\big\}^\top 
\notag \\ 
& +  \big\{ e^{-\alpha(T-1-t)} \overline\kappa  G_1^{\top}(t, \Lambda_1^N ) ( M^{(1)}_Z(t) )^\top \theta  
   \big\}^\top 
\notag \\ 
& 
+  \Big\{ G_2^{\top}(t, \Lambda_1^N, \Lambda_2^N ) ( M^{(1)}_Z(t) )^\top \big[ \Lambda_1^N(t+1) \overline\theta 
+ \Lambda_2^N(t+1) \big( \theta + \overline\theta  \big) 
\big]  
\Big\}^\top 
\notag \\ 
& 
+  \Big\{ \big[ G_1^{\top}(t, \Lambda_1^N) +  G_2^{\top}(t, \Lambda_1^N, \Lambda_2^N ) \big] ( M^{(1)}_Z(t) )^\top 
\big[ \Lambda_2^{N\top}(t+1) \overline\theta  
 + \Lambda_4^N(t+1)\big( \theta + \overline\theta \big) \big] 
 \Big\}^\top 
\notag \\
%%%%%%% 
& 
 + e^{-\alpha(T-1-t)} 
 \big[ \kappa \overline\theta^\top \overline\theta  + \overline\kappa \theta^\top \theta  \big]
\notag \\ 
%%%%%%%%
& +  \theta^\top  \Lambda_4^N(t+1) \theta  
+  \theta^\top \big[ \Lambda_2^{N\top}(t+1)  +  \Lambda_4^N(t+1) \big] 
 \overline\theta  
\notag \\ 
& +  \overline\theta^\top \big[ \Lambda_2^N(t+1)  +  \Lambda_4^{N}(t+1) \big] 
\theta  
\notag \\ 
& + \overline\theta^\top \big[ \Lambda_1^N(t+1) +  \Lambda_2^{N}(t+1) + \Lambda_2^{N\top}(t+1)   
   + \Lambda_4^N(t+1) 
\big] \overline\theta , 
\notag  
\end{align} 
\allowdisplaybreaks
\begin{align} 
 \psi_1( t, \Lambda_1^N, \Lambda_2^N, \chi_1^N  ) 
 \eqdef &  G_1^\top(t, \Lambda_1^N ) \big[ Q_1(t, \Lambda_1^N ) + Q_2(t, \Lambda_2^N ) \big] 
  H(t, \Lambda_1^N, \chi_1^N) 
  \label{psi1} \\ 
 & - e^{-\alpha(T-1-t)} \kappa G_1^\top(t, \Lambda_1^N ) ( M^{(1)}_Z(t) )^\top y_{t+1} 
  + G_1^\top(t, \Lambda_1^N ) ( M^{(1)}_Z(t) )^\top \chi_1^N(t+1) 
  \notag \\ 
  & + e^{-\alpha(T-1-t)} \kappa \theta^\top M^{(1)}_Z(t) H(t, \Lambda_1^N, \chi_1^N ) 
  \notag \\ 
  & + \theta^\top \big[ \Lambda_1^N(t+1) + \Lambda_2^N(t+1) \big] M^{(1)}_Z(t) H(t, \Lambda_1^N, \chi_1^N) 
  \notag \\ 
  &  - e^{-\alpha(T-1-t)} \kappa \theta^\top y_{t+1} 
  + \theta^\top \chi_1^N(t+1) , 
 \notag 
\end{align} 
%%%%%%%%%%%%%%%%%% 
\allowdisplaybreaks 
\begin{align} 
\label{psi2}  
  \psi_2( t, \Lambda_1^N, \Lambda_2^N, \Lambda_4^N, \chi_1^N, \chi_2^N )  
 \eqdef &  \big[ G_1^\top(t, \Lambda_1^N ) + G_2^\top(t, \Lambda_1^N, \Lambda_2^N ) \big] 
  \big[ Q_2^\top(t, \Lambda_2^N ) + Q_4^\top(t, \Lambda_4^N ) \big] H(t, \Lambda_1^N, \chi_1^N ) 
   \\ 
 & + G_2^\top(t, \Lambda_1^N, \Lambda_2^N ) \big[ Q_1(t, \Lambda_1^N ) + Q_2(t, \Lambda_2^N ) \big] 
  H(t, \Lambda_1^N, \chi_1^N ) 
 \notag  \\ 
 & - e^{-\alpha(T-1-t)} \kappa G_2^\top(t, \Lambda_1^N, \Lambda_2^N) ( M^{(1)}_Z(t) )^\top y_{t+1} 
 \notag \\ 
 & + G_1^\top(t, \Lambda_1^N ) ( M^{(1)}_Z(t) )^\top \chi_2^N(t+1) 
 \notag \\ 
 & + G_2^\top(t, \Lambda_1^N, \Lambda_2^N ) ( M^{(1)}_Z(t) )^\top \big[ \chi_1^N(t+1) + \chi_2^N(t+1) \big] 
  \notag \\ 
  & + e^{-\alpha(T-1-t)} \kappa \overline\theta^\top M^{(1)}_Z(t) H(t, \Lambda_1^N, \chi_1^N ) 
  \notag \\ 
  & + \overline\theta^\top \big[ \Lambda_1^N(t+1) + \Lambda_2^N(t+1) \big] 
  M^{(1)}_Z(t) H(t, \Lambda_1^N, \chi_1^N) 
  \notag \\ 
  & + (\theta + \overline\theta )^\top \big[ \Lambda_2^{N\top}(t+1) + \Lambda_4^N(t+1) \big] 
  M^{(1)}_Z(t) H(t, \Lambda_1^N, \chi_1^N) 
  \notag \\ 
  & - e^{-\alpha(T-1-t)} \kappa \overline\theta^\top y_{t+1} 
   + \overline\theta^\top \chi_1^N(t+1) 
    + (\theta + \overline\theta)^\top \chi_2^N(t+1) . 
 \notag 
\end{align}

\section{Ridge-Regression Objective and Solution} \label{ridge_solution}

The objective function in Equation \eqref{ridge_objective} can be re-written as
\begin{align}
& \sum_{s=(t-T)\vee 0}^{t-1} e^{ -\alpha(t-1-s)}\left\Vert \left(y_{s+1}-\theta Y_s^n - \bar{\theta}Y_s^{(N)} \right) - Z_s^n \beta_t^n \right \Vert^2+\gamma\big \Vert\beta_t^n \big \Vert ^2 \nonumber \\
&\qquad = \sum_{s=(t-T)\vee 0}^{t-1} \left \Vert e^{-\alpha(t-1-s)/2} \left(y_{s+1} - \theta Y_s^n - \bar{\theta} Y_s^{(N)} \right)- \left(e^{-\alpha(t-1-s)/2} Z_s^n \right) \beta^n_t \right \Vert^2 + \gamma \Vert\beta_t^n \Vert^2 \nonumber  \\ 
&\qquad = \big(\bar{y}^n_t - X^n_{t-1} \beta^n_t \big)^\top \big(\bar{y}^n_t - X^n_{t-1} \beta^n_t \big) 
+ \gamma \beta^{n\top}_t \beta^n_t, \label{ridge_formulation}
\end{align}
where
\begin{align*} 
& X^n_{t-1}  = \big[ e^{ -\alpha (t-1-0)/2} Z^{n\top}_{(t-T)\vee 0} , \ldots, e^{-\alpha (t-1-s)/2} Z^{n \top }_s, \ldots,  Z^{n \top }_{t-1} \big]^\top   \\ 
& \bar{y}^n_t  = \left[ e^{-\alpha (t-1-0)/2} (y_{((t-T)\vee 0)+ 1} - \theta Y^n_{(t-T)\vee 0} 
- \bar{\theta} Y^{(N)}_{(t-T)\vee 0} )^\top, \ldots, (y_{t} - \theta Y^n_{t-1} - \bar{\theta} Y^{(N)}_{t-1} )^\top \right]^\top.  
\end{align*} 
Equation \eqref{ridge_formulation} is a standard ridge-regression objective with closed-form solution \cite{HastieTibshiraniFriedman_2001ESLBook} given by 
\begin{align*}
    \beta_t^n = \left (X^{n\top}_{t-1}X_{t-1}^n + \gamma I_{d_z}\right)^{-1} X_{t-1}^{n\top} \bar{y}_t^n.
\end{align*}

\section{Metrics and Dataset Details} \label{datasets_metrics}

The primary performance metric used in this work is root mean squared error (RMSE). This is computed using the \texttt{sklearn.metrics} library. When computing the worst performing individual agents, the highest RMSE produced by the individual predictions of an agent over the entire timeseries is computed for all three runs separately, using the best hyperparameters for each dataset. These values are then averaged to produce the statistics of Table \ref{worst_agents_metric}.

\subsection{Periodic Signal}

The Periodic Signal series is, as the name suggests a highly oscillatory timeseries. It has $200$ data points rapidly oscillating between values of -0.9 and 0.9. The input, $\mathbf{x}_{t-1}$, is a single target lag such that $\mathbf{x}_{t-1} = \mathbf{y}_{t-1}$, which is used to predict $\mathbf{y}_{t}$.

\subsection{Logistic Map}

The Logistic Map dataset is a synthetic timeseries provided by the EchoTorch library \citep{echotorch}. It is based on the dynamics of the logistic map. The default parameters for the EchoTorch dataset are applied, namely $\alpha=5$, $\beta=11$, $\gamma=13$, $c=3.6$, and $b=0.13$, to generate a series with $200$ points. The target variable has a single dimension. As with the Periodic dataset, for time step $t$ and target $\mathbf{y}_t$, the input is $\mathbf{x}_{t-1} = \mathbf{y}_{t-1}$.

\subsection{Concept Drift}

The Concept Drift dataset also consists of $200$ points, drawn from the following analytical relationship. For $k \in [0, 2\pi]$, let $\mathbf{x} = [k, k, \sqrt{k}]^\top$. When $k \in \left[0, \frac{7}{8}\pi \right]$, 
\begin{align*}
    y_{1, 1}(\mathbf{x}) &= \mathbf{x}_1^2 + \sin \mathbf{x}_2 + \mathbf{x}_1  \mathbf{x}_3 + 0.5 \cos(10 \mathbf{x}_1), \\
    y_{2, 1}(\mathbf{x}) &= \mathbf{x}_1  \cos \mathbf{x}_2 + \mathbf{x}_3 - e^{-\mathbf{x}_2}.
\end{align*}
For $k \in \left[\frac{9}{8} \pi, 2\pi\right]$,
\begin{align*}
    y_{1, 2}(\mathbf{x}) &= \mathbf{x}_1 + \mathbf{x}_2 - \sin \mathbf{x}_3, \\
    y_{2, 2}(\mathbf{x}) &= \cos \mathbf{x}_1  \sin \mathbf{x}_2 + \mathbf{x}_3^2 + 0.25 \cos(10 \mathbf{x}_1).
\end{align*}
When $k \in \left[\frac{7}{8} \pi, \frac{9}{8} \pi \right]$, there is a rapid, but smooth transition between the mappings of $\mathbf{x}$ to $\mathbf{y}$. Let
\begin{align*}
    \alpha(x) = \begin{cases}
      1.0 & x < \frac{7}{8} \pi, \\
      \cos\left(2\left(x - \frac{7}{8}\pi\right)\right) & \frac{7}{8} \pi \geq x \leq \frac{9}{8} \pi, \\
      0.0 & x > \frac{9}{8} \pi. \\
    \end{cases}
\end{align*}
Then components of $\mathbf{y}$ are written $y_i(\mathbf{x}) = \alpha(\mathbf{x}_1) \cdot y_{i, 1}(\mathbf{x}) + (1-\alpha(\mathbf{x}_1)) \cdot y_{i, 2}(\mathbf{x})$. The points comprising the series are sampled uniformly from the interval $[0, 2\pi]$.

\subsection{Bank of Canada Exchange Rate}

The Bank of Canada (BoC) Exchange Rate dataset \citep{boc_exchange_rates} incorporates daily exchange rates from May 1, 2007 to April 28, 2017 relative to the Canadian dollar. In total, there are 
$3,651$ observations per currency. In the numerical experiments, USD exchange rates are used as the target variable. Inputs used to predict $\mathbf{y}_t$, $\mathbf{x}_{t-1}$, are $\{\text{USD}_{t-1}$, $\text{USD}_{t-2}$, $\text{AUD}_{t-1}$, $\text{AUD}_{t-2}$, $\text{EUR}_{t-1}$, $\text{EUR}_{t-2}$, $\text{GBP}_{t-1}$, $\text{GBP}_{t-2}$, $\text{JPY}_{t-1}$, $\text{JPY}_{t-2}\}$. Here, for example, $\text{EUR}_{t-1}$ is the exchange rate for the Euro from the previous timestep. The BoC dataset considers the first $2000$ daily rates and the validation set considers the final $2000$.

\subsection{Electricity Transformer Temperature Dataset} 

First introduced in \citep{haoyietal-informer-2021}, the Electricity Transformer Temperature (ETT) aims to predict the oil temperature (OT) at a given time within one of two electricity transformers. In the experiments, the ETT-small-h1 series is used, which records temperature measurements every hour from July 1, 2016 to June 26, 2018. In the numerical experiments, all six input features (HUFL, HULL, MUFL, MULL, LUFL, LULL) are used to make predictions. When predicting $\text{OT}_t$, inputs include $\text{OT}_{t-1}$, $\text{OT}_{t-2}$, along with $\text{X}_{t-1}$, $\text{X}_{t-2}$, $\text{X}_{t-3}$, where $\text{X}$ represents an aforementioned input feature. All values in the dataset have been normalized by dividing the input and target sequences by the maximum value seen across the time series of interest. ETT Data results correspond to the first $2000$ data points of the series, while ETT Data Validation covers the next $2000$ points.

\section{Additional Experiment Details and Hyperparameter Sweeps} \label{details_hyperparameters}

As noted in Section \ref{sec:numerics}, each experiment is run three times, and the average RMSE over the three runs is reported in the results. To ensure reproducibility, random seeds of $\{2024, 2025, 2026\}$ are used. For the Logistic Map dataset and ESN models, better performance is noted for the decentralized strategy when using $\tanh$ activation functions instead of Hard Sigmoids. As such, the results reported in the tables use a $\tanh$ function. On the other hand, this activation did not improve results for the greedy algorithm. So the Hard Sigmoid function is applied in those settings, as for other datasets and models.

\begin{table}[ht!]
\caption{Hyperparameters for the decentralized and greedy strategy experiments.}
\centering
\begin{tabular}{lcc}
\toprule
Parameter & Decentralized & Greedy \\
\midrule
$\sigma$      & \{0.1, 1.0\} & \{0.1, 1.0\} \\
$\theta$      & \{0.9, 0.7\} & \{0.9, 0.7\} \\
$d_z$         & \{5, 10\} & \{5, 10\} \\
$T_a$         & \{1, 4\} & \{1, 4\} \\
$\alpha_a$    & \{0.2\} & \{0.2\} \\
$T$         & \{1, 4, 8\} & \{1, 3, 5\} \\
$\alpha$    & \{0.01, 0.1\} & \{0.01, 0.1, 0.2\} \\
$\gamma$      & \{0.01, 0.1, 1.0\} & \{0.01, 0.1\} \\
$\kappa$      & \{10.0, 1.0\} & -- \\
$\bar{\kappa}$ & \{0, 0.1, 1.0, 10.0\} & -- \\
\bottomrule
\end{tabular}
\label{hyperparameter_sweeps}
\end{table}

The computations in Algorithm \ref{alg:sync_subroutine} require approximation of expectations associated with the latent space of the models in the ensemble or transformations thereof. For example, estimates of $\mathbb{E} (Z_t^i)$ or $\mathbb{E}(Z_t^{i\top} Z_t^i)$ are needed. In the numerical experiments, each agent samples $100$ random agents from its own distribution and, using its shared timeseries input, produces latent representations for each timestep. These latent space samples are used to compute approximations of the required expectations. This procedure adds computation time to each agents predictions, but is fully decentralized. An alternative would be for each agent to send its latent encodings to the aggregation server. The server could then produce the appropriate estimates from all participating agents' representations and send them back to all participants. This approach effectively offloads the overhead associated with agent expectation simulations, but requires agents to share potentially proprietary encodings.

There are a number of hyperparameters in both the decentralized strategy and greedy agent-side optimization algorithms. For simplicity, the matrices $\theta$ and $\bar{\theta}$ are assumed to be diagonal matrices and $\bar{\theta} = I - \theta$. Hence, in the parameters, $\theta$ is represented by a single scalar and $\bar{\theta}$ is implicitly determined. Similarly, $\sigma$ is a uniform column vector and also specified by a single scalar. In the tables, the lookback length and discount factors for agent prediction aggregation are denoted as $T_a$ and $\alpha_a$, respectively. The parameter values explore for both algorithms are listed in Table \ref{hyperparameter_sweeps}.

\begin{table}[ht!]
\caption{Best hyperparameters for each dataset with the decentralized strategy for RFN models.}
\centering
\begin{tabular}{llcccccccccc}
\toprule
Dataset & Clients  & $\sigma$ & $\theta$ & $d_z$ & $T_a$ & $\alpha_a$ & $T$ & $\alpha$ & $\gamma$ & $\kappa$ & $\bar{\kappa}$ \\
\midrule
\multirow{3}{*}{Periodic} & 25 & $0.1$ & $0.7$ & $5$ & $1$ & $0.2$ & $8$ & $0.01$ & $1.0$ & $1.0$ & $10.0$ \\
& 100 & $0.1$ & $0.7$ & $5$ & $1$ & $0.2$ & $8$ & $0.01$ & $1.0$ & $1.0$ & $10.0$ \\
& 500 & $0.1$ & $0.7$ & $5$ & $1$ & $0.2$ & $8$ & $0.01$ & $1.0$ & $1.0$ & $10.0$ \\
\midrule
\multirow{3}{*}{Logistic} & 25 & $0.1$ & $0.7$ & $5$ & $1$ & $0.2$ & $1$ & $0.01$ & $1.0$ & $1.0$ & $10.0$ \\
& 100 & $0.1$ & $0.7$ & $5$ & $1$ & $0.2$ & $1$ & $0.01$ & $1.0$ & $1.0$ & $10.0$ \\
& 500 & $0.1$ & $0.9$ & $5$ & $1$ & $0.2$ & $1$ & $0.1$ & $1.0$ & $1.0$ & $10.0$ \\
\midrule
\multirow{3}{*}{Concept} & 25 & $1.0$ & $0.7$ & $5$ & $4$ & $0.2$ & $1$ & $0.01$ & $0.1$ & $1.0$ & $0.0$ \\
& 100 & $1.0$ & $0.9$ & $5$ & $4$ & $0.2$ & $1$ & $0.01$ & $0.01$ & $10.0$ & $0.0$ \\
& 500 & $1.0$ & $0.7$ & $10$ & $4$ & $0.2$ & $1$ & $0.01$ & $0.01$ & $10.0$ & $0.0$ \\
\midrule
\multirow{3}{*}{BoC} & 25 & $1.0$ & $0.7$ & $5$ & $4$ & $0.2$ & $1$ & $0.01$ & $0.1$ & $1.0$ & $0.0$ \\
& 100 & $1.0$ & $0.9$ & $5$ & $4$ & $0.2$ & $1$ & $0.01$ & $0.01$ & $10.0$ & $0.0$ \\
& 500 & $1.0$ & $0.7$ & $10$ & $4$ & $0.2$ & $1$ & $0.01$ & $0.01$ & $10.0$ & $0.0$ \\
\midrule
\multirow{3}{*}{ETT} & 25 & $1.0$ & $0.7$ & $10$ & $1$ & $0.2$ & $1$ & $0.1$ & $1.0$ & $1.0$ & $0.1$ \\
& 100 & $1.0$ & $0.7$ & $10$ & $1$ & $0.2$ & $1$ & $0.1$ & $1.0$ & $1.0$ & $0.1$ \\
& 500 & $1.0$ & $0.7$ & $10$ & $1$ & $0.2$ & $1$ & $0.1$ & $1.0$ & $1.0$ & $0.1$ \\
\bottomrule
\end{tabular}
\label{best_hyperparameters_rfns_decent}
\end{table}

The best hyperparameters for each dataset, model, and agent-side strategy are reported in Tables \ref{best_hyperparameters_rfns_decent}--\ref{best_hyperparameters_esns_greedy}. Note that optimal hyperparameters for the BoC Exchange and ETT timeseries are also applied when performing predictions for the validation series drawn from these datasets. For the uniform-mean aggregation experiments presented in Table \ref{uniform_averages}, the same HPs as presented in Table \ref{hyperparameter_sweeps} are searched and the best are used. They may, however, differ from those presented in Tables \ref{best_hyperparameters_rfns_decent}--\ref{best_hyperparameters_esns_greedy}.

\begin{table}[ht!]
\caption{Best hyperparameters for each dataset with the greedy strategy for RFN models.}
\centering
\begin{tabular}{llccccccccc}
\toprule
Dataset & Clients  & $\sigma$ & $\theta$ & $d_z$ & $T_a$ & $\alpha_a$ & $T$ & $\alpha$ & $\gamma$ \\
\midrule
\multirow{3}{*}{Periodic} & 25 & $1.0$ & $0.7$ & $5$ & $1$ & $0.2$ & $1$ & $0.01$ & $0.1$ \\
& 100 & $1.0$ & $0.7$ & $5$ & $1$ & $0.2$ & $1$ & $0.2$ & $0.1$ \\
& 500 & $1.0$ & $0.7$ & $5$ & $1$ & $0.2$ & $1$ & $0.01$ & $0.1$ \\
\midrule
\multirow{3}{*}{Logistic} & 25 & $0.1$ & $0.7$ & $10$ & $1$ & $0.2$ & $3$ & $0.1$ & $0.1$ \\
& 100 & $0.1$ & $0.7$ & $5$ & $1$ & $0.2$ & $3$ & $0.2$ & $0.1$ \\
& 500 & $0.1$ & $0.7$ & $5$ & $1$ & $0.2$ & $5$ & $0.1$ & $0.1$ \\
\midrule
\multirow{3}{*}{Concept} & 25 & $0.1$ & $0.9$ & $10$ & $1$ & $0.2$ & $1$ & $0.01$ & $0.01$ \\
& 100 & $0.1$ & $0.9$ & $10$ & $1$ & $0.2$ & $1$ & $0.01$ & $0.01$ \\
& 500 & $0.1$ & $0.9$ & $10$ & $1$ & $0.2$ & $1$ & $0.01$ & $0.01$ \\
\midrule
\multirow{3}{*}{BoC} & 25 & $0.1$ & $0.9$ & $10$ & $1$ & $0.2$ & $1$ & $0.01$ & $0.01$ \\
& 100 & $1.0$ & $0.9$ & $10$ & $1$ & $0.2$ & $1$ & $0.2$ & $0.01$ \\
& 500 & $0.1$ & $0.7$ & $10$ & $1$ & $0.2$ & $1$ & $0.1$ & $0.01$ \\
\midrule
\multirow{3}{*}{ETT} & 25 & $0.1$ & $0.9$ & $10$ & $1$ & $0.2$ & $1$ & $0.01$ & $0.01$ \\
& 100 & $0.1$ & $0.9$ & $10$ & $1$ & $0.2$ & $1$ & $0.01$ & $0.01$ \\
& 500 & $0.1$ & $0.9$ & $10$ & $1$ & $0.2$ & $1$ & $0.01$ & $0.01$ \\
\bottomrule
\end{tabular}
\label{best_hyperparameters_rfns_greedy}
\end{table}

\begin{table}[ht!]
\caption{Best hyperparameters for each dataset with the decentralized strategies for ESNs models.}
\centering
\begin{tabular}{llcccccccccc}
\toprule
Dataset & Clients  & $\sigma$ & $\theta$ & $d_z$ & $T_a$ & $\alpha_a$ & $T$ & $\alpha$ & $\gamma$ & $\kappa$ & $\bar{\kappa}$ \\
\midrule
\multirow{3}{*}{Periodic} & 25 & $1.0$ & $0.7$ & $5$ & $1$ & $0.2$ & $8$ & $0.01$ & $1.0$ & $1.0$ & $10.0$ \\
& 100 & $1.0$ & $0.7$ & $5$ & $1$ & $0.2$ & $8$ & $0.01$ & $1.0$ & $1.0$ & $10.0$ \\
& 500 & $1.0$ & $0.7$ & $5$ & $1$ & $0.2$ & $8$ & $0.01$ & $1.0$ & $1.0$ & $10.0$ \\
\midrule
\multirow{3}{*}{Logistic}& 25 & $1.0$ & $0.9$ & $10$ & $4$ & $0.2$ & $1$ & $0.1$ & $0.01$ & $10.0$ & $0.0$ \\
& 100 & $0.1$ & $0.9$ & $10$ & $1$ & $0.2$ & $1$ & $0.01$ & $0.01$ & $1.0$ & $0.0$ \\
& 500 & $0.1$ & $0.9$ & $10$ & $1$ & $0.2$ & $1$ & $0.1$ & $0.01$ & $10.0$ & $0.0$ \\
\midrule
\multirow{3}{*}{Concept} & 25 & $0.1$ & $0.9$ & $10$ & $4$ & $0.2$ & $1$ & $0.1$ & $0.01$ & $10.0$ & $0.0$ \\
& 100 & $0.1$ & $0.7$ & $10$ & $4$ & $0.2$ & $1$ & $0.1$ & $0.01$ & $10.0$ & $0.0$ \\
& 500 & $0.1$ & $0.7$ & $10$ & $4$ & $0.2$ & $1$ & $0.1$ & $0.01$ & $10.0$ & $0.0$ \\
\midrule
\multirow{3}{*}{BoC} & 25 & $0.1$ & $0.7$ & $10$ & $1$ & $0.2$ & $1$ & $0.1$ & $1.0$ & $1.0$ & $0.1$ \\
& 100 & $1.0$ & $0.7$ & $5$ & $4$ & $0.2$ & $1$ & $0.1$ & $0.01$ & $1.0$ & $1.0$ \\
& 500 & $0.1$ & $0.9$ & $10$ & $1$ & $0.2$ & $1$ & $0.01$ & $1.0$ & $1.0$ & $0.1$ \\
\midrule
\multirow{3}{*}{ETT} & 25 & $1.0$ & $0.7$ & $5$ & $1$ & $0.2$ & $1$ & $0.1$ & $1.0$ & $1.0$ & $0.1$ \\
& 100 & $1.0$ & $0.7$ & $5$ & $1$ & $0.2$ & $1$ & $0.1$ & $1.0$ & $1.0$ & $0.1$ \\
& 500 & $1.0$ & $0.9$ & $5$ & $1$ & $0.2$ & $8$ & $0.1$ & $1.0$ & $1.0$ & $0.1$ \\
\bottomrule
\end{tabular}
\label{best_hyperparameters_esns_decent}
\end{table}

\begin{table}[ht!]
\caption{Best hyperparameters for each dataset with the greedy strategies for ESNs models.}
\centering
\begin{tabular}{llccccccccc}
\toprule
Dataset & Clients  & $\sigma$ & $\theta$ & $d_z$ & $T_a$ & $\alpha_a$ & $T$ & $\alpha$ & $\gamma$ \\
\midrule
\multirow{3}{*}{Periodic} & 25 & $0.1$ & $0.7$ & $5$ & $1$ & $0.2$ & $5$ & $0.1$ & $0.1$ \\
& 100 & $0.1$ & $0.9$ & $5$ & $1$ & $0.2$ & $1$ & $0.2$ & $0.1$ \\
& 500 & $0.1$ & $0.7$ & $5$ & $1$ & $0.2$ & $5$ & $0.1$ & $0.1$ \\
\midrule
\multirow{3}{*}{Logistic} & 25 & $1.0$ & $0.7$ & $10$ & $1$ & $0.2$ & $3$ & $0.1$ & $0.1$ \\
& 100 & $0.1$ & $0.7$ & $10$ & $1$ & $0.2$ & $1$ & $0.2$ & $0.01$ \\
& 500 & $0.1$ & $0.7$ & $5$ & $1$ & $0.2$ & $5$ & $0.1$ & $0.01$ \\
\midrule
\multirow{3}{*}{Concept} & 25 & $0.1$ & $0.9$ & $10$ & $1$ & $0.2$ & $1$ & $0.01$ & $0.01$ \\
& 100 & $0.1$ & $0.9$ & $10$ & $1$ & $0.2$ & $1$ & $0.01$ & $0.01$ \\
& 500 & $0.1$ & $0.9$ & $10$ & $1$ & $0.2$ & $1$ & $0.01$ & $0.01$ \\
\midrule
\multirow{3}{*}{BoC} & 25 & $0.1$ & $0.9$ & $5$ & $1$ & $0.2$ & $1$ & $0.01$ & $0.01$ \\
& 100 & $0.1$ & $0.7$ & $10$ & $1$ & $0.2$ & $1$ & $0.1$ & $0.01$ \\
& 500 & $0.1$ & $0.7$ & $10$ & $1$ & $0.2$ & $1$ & $0.1$ & $0.01$ \\
\midrule
\multirow{3}{*}{ETT} & 25 & $0.1$ & $0.9$ & $5$ & $1$ & $0.2$ & $1$ & $0.01$ & $0.01$ \\
& 100 & $0.1$ & $0.9$ & $5$ & $1$ & $0.2$ & $1$ & $0.01$ & $0.01$ \\
& 500 & $0.1$ & $0.9$ & $5$ & $1$ & $0.2$ & $1$ & $0.01$ & $0.01$ \\
\bottomrule
\end{tabular}
\label{best_hyperparameters_esns_greedy}
\end{table}

\section{Ablation Results: Mean Previous Predictions} \label{ablation_results}

In this appendix, the importance of incorporating the mean agent predictions, $Y_t^{(N)}$, into the prediction process and objective function defined in Equation \eqref{ridge_objective} is explored. Noting the simplification such that $\bar{\theta} = I - \theta$ from Appendix \ref{details_hyperparameters}, this is simply measured by setting $\theta = I$, or equivalently $1$ in the hyperparameters. The results are displayed in Table \ref{mean_prediction_ablation}. In almost all cases, inclusion of the mean prediction information improves performance. Furthermore, for ESN models, performance for ETT and ETT Validation are nearly identical. Notably, however, performance is markedly improved for the Concept Drift timeseries when this information is excluded. As noted in Section \ref{main_results}, this dataset is somewhat distinctive, as it is the only timeseries where the decentralized strategy under-performs in all settings.

\begin{table}[ht!]
\caption{Average RMSE for aggregated RFN (top) and ESN (bottom) model predictions using the greedy strategy with 100 agents. For $\theta \in \{0.9, 0.7\}$, the best HP is used of the two. When $\theta = 1.0$, this implicitly means $\bar{\theta} = 0$.}
\small
\centering
\begin{tabular}{llllllll}
\toprule
$\theta$ & Periodic & Logistic & Concept & BoC & BoC Val. & ETT & ETT Val. \\
\midrule
$\{0.9, 0.7\}$ & \textbf{9.8567e-1}  & \textbf{2.5047e-1}  & 2.4333e-1  & \textbf{1.0033e-2}  & \textbf{7.0897e-3}  & \textbf{2.6253e-2}  & \textbf{5.1658e-2} \\
$1.0$ & 1.0377e-0 & 2.5606e-1 & \textbf{1.6285e-1} & 1.4610e-2 & 3.6250e-2 & 2.7994e-2 & 5.3749e-2 \\
\midrule
$\{0.9, 0.7\}$ & \textbf{8.4452e-1} & \textbf{2.4312e-1} & 2.0690e-1 & \textbf{7.0577e-3} & \textbf{5.0340e-3} & 2.6357e-2 & 5.1246e-2 \\
$1.0$ & 8.4754e-1 & 2.5152e-1 & \textbf{1.6683e-1} & 9.0426e-2 & 1.5624e-2 & \textbf{2.6347e-2} & \textbf{5.1244e-2} \\
\bottomrule
\end{tabular}
\label{mean_prediction_ablation}
\end{table}

\section{Analysis of Worst-Performing Agents}
\label{worst_performing_agents_study}
To further analyze the effect of the decentralized Nash policy on the worst-performing agents, Table \ref{bottom_twenty_percent_agents} reports the mean RMSE of the bottom $20\%$ of agents in the pool. The results presented in Tables \ref{bottom_twenty_percent_agents} and \ref{worst_agents_metric} demonstrate that the decentralized strategy reduces not only the error of the single worst-performing agent, but also the aggregate error across the broader population of underperforming agents.

\begin{table}[ht!]
\caption{Average RMSE for the bottom $20\%$ of agents using the greedy and decentralized (Nash) strategies for RFN (top) and ESN (bottom) models. Numbers in parentheses denote the total agent pool size, where results are averaged over the 20 worst-performing agents for a pool of 100, and over the 100 worst-performing agents for a pool of 500.}
\small
\centering
\begin{tabular}{llllllll}
\toprule
Strategy & Periodic & Logistic & Concept & BoC & BoC Val. & ETT & ETT Val. \\
\midrule

Greedy (100) & 1.8862e-0  & 4.9248e-1  & 2.5376e+8  & 2.5915e-1  & 6.9215e-2  & 6.3253e-2  & 1.7287e-1 \\
Nash (100) & \textbf{1.0501e-0} & \textbf{2.5360e-1} & \textbf{3.1510e-1} & \textbf{1.3088e-2} & \textbf{9.5603e-3} & \textbf{3.9988e-2} & \textbf{8.1782e-2} \\
\midrule
Greedy (500) & 1.9243e-0  & 5.0704e-1  & 1.2635e-0  & 2.4934e-1  & 1.0243e-1  & 6.6883e-2  &  9.458e-2 \\
Nash (500) & \textbf{1.1581e-0}  & \textbf{4.4232e-1} & \textbf{3.1634e-1} & \textbf{1.1596e-2} & \textbf{8.3653e-3} & \textbf{4.2180-2} & \textbf{8.3913e-2} \\
\midrule
\midrule

Greedy (100) & \textbf{9.0377e-1}  & \textbf{ 3.5093e-1}  & 9.5067e-1  & 4.3773-2  & 2.9776e-2  & 3.0381e-2  & \textbf{5.3778e-2} \\
Nash (100) & 1.1038e-0 & 3.8113e-1 & \textbf{3.0407e-1} & \textbf{8.6683e-3} & \textbf{6.2561e-3} & \textbf{3.0276e-2} & 5.9781e-2 \\
\midrule
Greedy (500) & \textbf{9.1663e-1}  & 3.1862e-1  &  9.1228e-1  & 4.4002e21  & 2.8764e-2  & \textbf{3.0499e-2}  & \textbf{5.4466e-2} \\

Nash (500) & 1.1031e-0  & \textbf{3.1210e-1} & \textbf{3.0755e-1} & \textbf{9.1993e-3} & \textbf{6.6560e-3} & 4.8117e-2 & 6.7406e-2 \\

\bottomrule
\end{tabular}
\label{bottom_twenty_percent_agents}
\end{table}

%\newpage
%\input{checklist.tex} 
% \input{0_Delete_at_End/Old_Results/Interpretations_of_Initialization}

\end{document}